\documentclass{article}
\usepackage{microtype}
\usepackage{graphicx}
\usepackage{subfigure}
\usepackage{enumerate}
\usepackage{booktabs}
\usepackage{hyperref}
\usepackage{soul}
\usepackage{amsmath, enumerate}

\usepackage[accepted]{icml2024}
\usepackage{amsmath}
\usepackage{amssymb}
\usepackage{mathtools}
\usepackage{amsthm}

\DeclareMathOperator*{\argmax}{arg\,max}
\DeclareMathOperator*{\localargmax}{local\,argmax}
\providecommand{\nb}{\nabla}
\usepackage[capitalize,noabbrev]{cleveref}

\theoremstyle{plain}
\newtheorem{theorem}{Theorem}[section]
\newtheorem{proposition}[theorem]{Proposition}
\newtheorem{conjecture}[theorem]{Conjecture}
\newtheorem{lemma}[theorem]{Lemma}
\newtheorem{corollary}[theorem]{Corollary}

\theoremstyle{definition}
\newtheorem{definition}[theorem]{Definition}
\newtheorem{intuition}[theorem]{Intuition}
\newtheorem{example}[theorem]{Example}
\newtheorem{assumption}[theorem]{Assumption}
\theoremstyle{remark}
\newtheorem{remark}[theorem]{Remark}

\definecolor{green_better}{rgb}{0.0, 0.5, 0.0}
\definecolor{blue_better}{rgb}{0.0, 0.2901960784313726, 0.6784313725490196}

\newcommand{\chen}[1]{\textcolor{cyan}{[Chen: #1]}}

\usepackage[textsize=tiny]{todonotes}

\icmltitlerunning{Self-Correcting Self-Consuming Loops for Generative Model Training}

\begin{document}

\twocolumn[
\icmltitle{Self-Correcting Self-Consuming Loops for Generative Model Training}
\icmlsetsymbol{equal}{*}

\begin{icmlauthorlist}
\icmlauthor{Nate Gillman}{brown}
\icmlauthor{Michael Freeman}{brown}
\icmlauthor{Daksh Aggarwal}{brown}
\icmlauthor{Chia-Hong Hsu}{brown}\\
\icmlauthor{Calvin Luo}{brown}
\icmlauthor{Yonglong Tian}{google}
\icmlauthor{Chen Sun}{brown,google}

\end{icmlauthorlist}

\icmlaffiliation{brown}{Brown University}
\icmlaffiliation{google}{Google DeepMind}

\icmlcorrespondingauthor{Nate Gillman}{\url{nate_gillman@brown.edu}}
\icmlcorrespondingauthor{Chen Sun}{\url{chensun@brown.edu}}

\icmlkeywords{Machine Learning, Generative Modeling, Self-Consuming Loops, Data Contamination, Deep Learning, Artificial Intelligence, Human Motion Synthesis}

\vskip 0.3in
]

\printAffiliationsAndNotice{}

\begin{abstract}
As synthetic data becomes higher quality and proliferates on the internet, machine learning models are increasingly trained on a mix of human- and machine-generated data. 
Despite the successful stories of using synthetic data for representation learning, using synthetic data for generative model training creates ``self-consuming loops'' which may lead to training instability or even collapse, unless certain conditions are met. 
Our paper aims to stabilize self-consuming generative model training. 
Our theoretical results demonstrate that by introducing an idealized correction function, which maps a data point to be more likely under the true data distribution, self-consuming loops can be made \textit{exponentially} more stable. 
We then propose self-correction functions, which rely on expert knowledge (e.g. the laws of physics programmed in a simulator), and aim to approximate the idealized corrector automatically and at scale.
We empirically validate the effectiveness of self-correcting self-consuming loops on the challenging human motion synthesis task, and observe that it successfully avoids model collapse, even when the ratio of synthetic data to real data is as high as 100\%.
\end{abstract}
\vspace{-1.5em}

\section{Introduction}

\begin{figure}[t]
\vspace{-1em}
\begin{center}
\includegraphics[width=0.95\columnwidth]{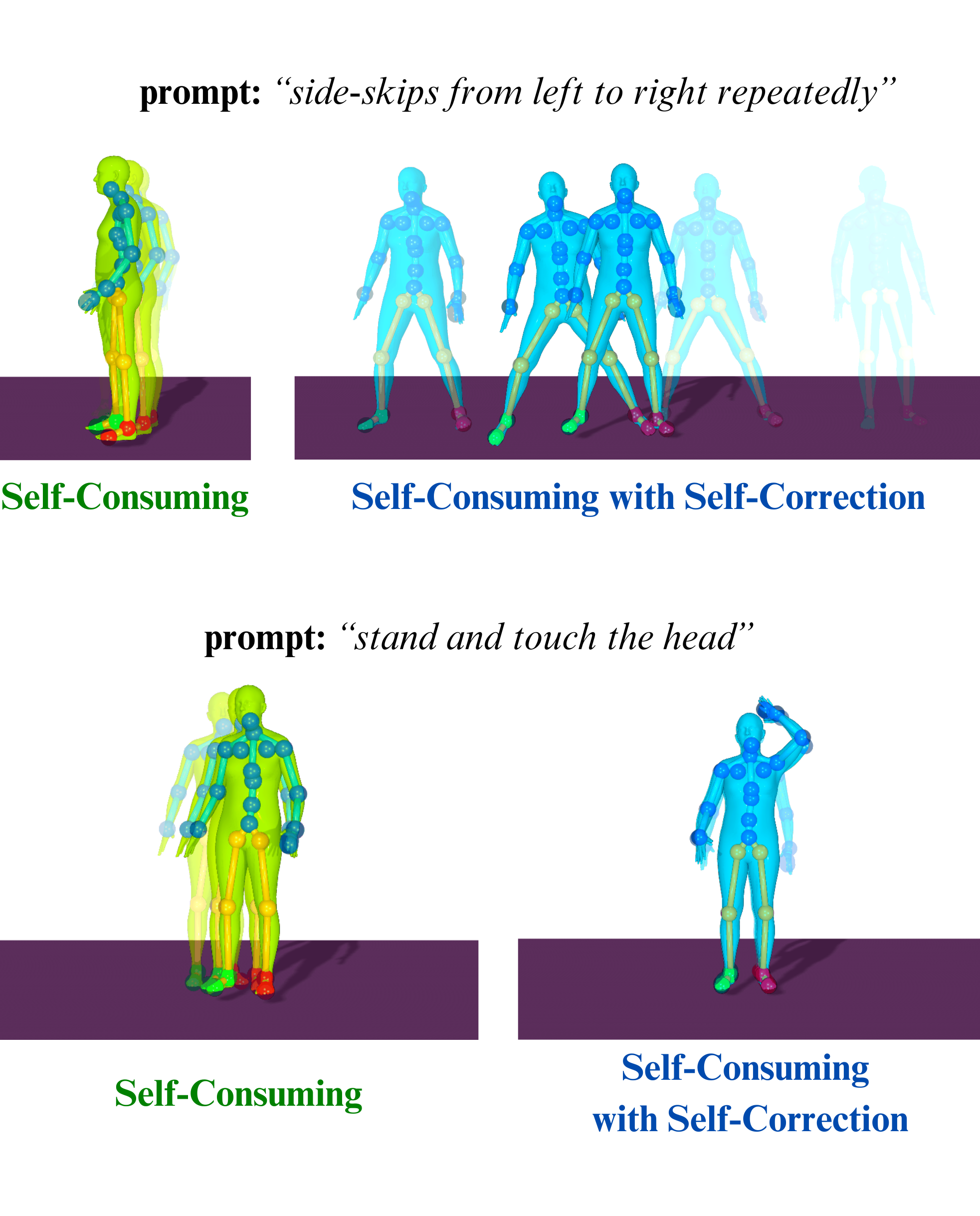}
\vspace{-1em}
\caption{What happens after iteratively training a text-conditioned generative model for human motion synthesis for 50 generations? We simulate a self-consuming loop by creating synthetic data with the latest generative model, and mixing them with the original data to continue training the next generative model. We observe that by self-correcting the synthetic data with a physics simulator, the model can successfully avoid \textbf{\color{green_better}collapse} and generate \textbf{\color{blue_better}high-quality human motion}. 
Faded poses represent poses from further back in time. Our paper provides theoretical and empirical justification for the self-correcting self-consuming loop.}
\label{fig:human_teaser}
\vspace{-2em}
\end{center}
\end{figure}

Generative models have been used to synthesize training data for various learning tasks, to varying degrees of success. For example, for the tasks of image classification and contrastive representation learning, recent work~\cite{azizi2023synthetic,tian2023learning} finds that using data synthesized from generative models rivals using real data.
Unfortunately, there is a gloomier outlook when attempting to generalize this framework to generative model training. 

On one hand, there is evidence to suggest that training a generative model with its own outputs in a self-consuming manner will lead to collapse~\cite{alemohammad2023self}. 
For example, after 50 iterations of self-consuming training, a human motion diffusion model~\cite{tevet2023human} collapses and fails to follow the text prompts or the laws of physics (see the two examples on the left of Figure~\ref{fig:human_teaser}).

On the other hand, evidence suggests that such a framework could avoid collapse, but only when a ``moderate'' amount of synthetic data is used~\cite{bertrand2023stability}. 
Worse still, this self-consuming scenario might happen without us knowing, and without us being able to quantify how much synthetic data is being used during training, due to the wide spread of AI generated content on the internet.

Intuitively, model collapse might be delayed or avoided by incorporating higher quality human generated data~\cite{alemohammad2023self}, or by manually fixing the ``mistakes'' in machine created data. 
Considering the size of datasets used in practice~\cite{schuhmann2022laion}, neither of these options is a scalable solution.

In this paper, we aim to provide a theoretical analysis of how certain operations would avoid collapse in self-consuming loops, without any assumptions on the ``moderateness'' of synthetic data corruption. 
We introduce the mathematical abstraction of a \emph{self-correction operation}.
This operation maps synthesized data that are sampled from the generative model to data that are better representatives from the target probability distribution that the model is attempting to approximate.
Instead of training on a combination of real data and synthesized data, we propose training on a combination of real data and synthesized \emph{and then self-corrected} data.
Note that injecting fresh human generated data can be viewed as a special case of this operation.

Our \textbf{main theoretical findings} (Theorem~\ref{thm:2_shortened}):
\vspace{-3mm}
\begin{enumerate}[(1)]
    \setlength\itemsep{0.0em} 
    \item The self-consuming model with self-correction is \textit{exponentially} more stable than the self-consuming model without any self-correction.
    \item The self-correction procedure guarantees less unwanted variance during self-consuming model training. 
\end{enumerate}
\vspace{-3mm}

In our theoretical study, we assume that correction is \emph{ideal} in order to obtain rigorous performance guarantees.
In our empirical study, we evaluate whether the same conclusions hold for \emph{noisy} self-correction functions. We propose to automate this ``self-correction'' process by relying on programmed expert knowledge rather than a human-in-the-loop, such that the function can be applied at scale. We focus on the human motion synthesis task~\cite{Guo_2022_CVPR}, and implement the self-correction function with a physics simulator-based imitation model \cite{Luo2021DynamicsRegulatedKP}.
Our \textbf{empirical results} confirm that our theoretical findings hold in practice:
\vspace{-3mm}
\begin{enumerate}[(1)]
    \setlength\itemsep{0.1em} 
    \item As illustrated in Figure~\ref{fig:human_teaser},
    the self-correcting self-consuming model generates higher-quality human motion than the one without any self-correction.
    \item The self-correction function allows self-consuming loops to avoid collapse even at a high synthetic data to real data ratio (e.g. 100\%).
\end{enumerate}
\vspace{-3mm}

Our theory and experiments suggest that self-correction should stabilize self-consuming model training for any generative modeling task for which there exists a high quality ``self-correction'' function. We have released all the code associated with this paper.\footnote{Project page: \href{https://nategillman.com/sc-sc.html}{https://nategillman.com/sc-sc.html}}

\section{Related Work}

\subsection{Learning Representations with Synthetic Data}

Real curated datasets are costly to obtain, so there has been much interest in generating synthetic data as training data for various vision tasks. \citet{azizi2023synthetic} demonstrates that text-to-image diffusion models such as Imagen~\citep{saharia2022photorealistic}  can generate synthetic examples that augment the ImageNet dataset for better image classification. 
\citet{he2022synthetic} studies how synthetic data from text-to-image models, when used exclusively, can be used as training data for image recognition tasks. 
Similarly, \citet{tian2023learning} finds that using synthetic outputs from a text-to-image model results in contrastive models whose downstream performance rivals that of CLIP \cite{radford2021learning} on visual recognition tasks, including dense prediction.
And the work in \citet{jahanian2021generative} explored methods for multi-view representation learning by using the latent space of the generative models to generate multiple ``views'' of the synthetic data.
The above works collectively provide evidence that \emph{some} representation learning tasks, when trained on synthetic data from some \emph{given} generative models, yield excellent results.

\subsection{Training Generative Models on Synthetic Data}

Another line of reseach investigates the use of synthetic data for training \emph{generative} models. 
\citet{shumailov2023curse} and \citet{Martinez2023towards} show that the use of model generated content in generative model training results in model degradation, likely because self-consuming loops remove low-density areas from the estimated probability manifold.
\citet{alemohammad2023self} formalize three different kinds of self-consuming generative models: the fully synthetic loop, the synthetic augmentation loop, and the fresh data loop.
In all of these loops, they iteratively re-train the model from scratch for every new generation. They empirically find that only the fresh data loop avoids model degradation.

Another recent work \cite{bertrand2023stability} considers the problem of iteratively fine-tuning in the context of synthetic augmentation loops.
They find that self-consuming augmentation loops do not necessarily collapse, so long as the synthetic augmentation percentage is sufficiently low. 
The authors use techniques from the field of performative stability~\citep{perdomo2020performative} to prove the existence of a convergence phenomenon in the space of model parameters. Our paper differs from prior work as we conduct analysis on self-consuming generative model training when the synthetic data can be optionally \textit{corrected}. The correction can be performed with a human-in-the-loop, or by incorporating learned or programmed expert knowledge, as explored for natural language~\cite{saunders2022self,welleck2022generating,wu2023self} and human motion~\citep{yuan2023physdiff,xu2023interdiff}. We validate our theory with a practical self-correcting operations designed for image generation and human motion synthesis tasks.

\begin{algorithm*}[t]
   \caption{Iterative Fine-tuning of a Generative Model \textcolor{blue_better}{With Correction}}
   \label{alg:physics_augmentation_loop}
\begin{algorithmic}
   \STATE {\bfseries Input:} $\mathcal D_{\text{real}} := \{x_i\}_{i=1}^n$, $\mathcal A$, $\mathcal A_{\text{ft}}$, \textcolor{blue_better}{$\pi_\gamma$} \textcolor{gray}{// ground truth data, learning procedure, fine-tuning procedure,} \textcolor{blue_better}{correction function} 
   \STATE {\bfseries Parameters:} $T$, $\lambda$, $\textcolor{blue_better}{\gamma}$  \textcolor{gray}{// number of retraining iterations, proportion of generated data,} \textcolor{blue_better}{correction strength} 
   
   \STATE $p_{\theta_0} \leftarrow \mathcal A(\mathcal D_{\text{real}})$ \textcolor{gray}{// learn generative model from scratch on true data}
   
   \FOR{$t = 1$ {\bfseries to} $T$}
   
   \STATE $\mathcal D_{\text{synth}} \leftarrow \{\textcolor{blue_better}{\pi_\gamma}(\tilde x_i)\}_{i=1}^{\lfloor \lambda \cdot n\rfloor}$, with  $\tilde{x}_i\sim p_{\theta_{t-1}}$ \textcolor{gray}{// sample $\lfloor \lambda\cdot n\rfloor$ synthetic data points, \textcolor{blue_better}{pass through correction function}}
   
   \STATE $p_{\theta_{t}} \leftarrow \mathcal A_{\text{ft}}(\mathcal D_{\text{real}}\cup \mathcal D_{\text{synth}}; p_{\theta_{t-1}}$) \textcolor{gray}{// fine-tune previous generation using augmented dataset}
   
   \ENDFOR
   
   \STATE {\bfseries Return} $[p_{\theta_0}, p_{\theta_1}, p_{\theta_2},\dots, p_{\theta_T}]$
\end{algorithmic}
\end{algorithm*}

\section{Overall Training Procedure}

We describe our proposed procedure in concise language in Algorithm~\ref{alg:physics_augmentation_loop}, and we explain it in more detail here.
We train the zero'th generation from scratch on the ground truth dataset $\mathcal D_{\mathrm{real}}:=\{x_i\}_{i=1}^n$, and we stop training when the model is close to convergence.
For all the following generations, we fine-tune the previous generation's latest checkpoint on a combination of the ground truth dataset $\mathcal D_{\mathrm{real}}$, as well as $\lfloor\lambda\cdot n\rfloor$ synthetic data points which are generated from the previous generation's latest checkpoint, and then passed through the \emph{correction function} $\pi_\gamma$.


The correction function $\pi_\gamma$ is parameterized by the \emph{correction strength} $\gamma\in\mathbb R_{\ge 0}$, which controls how much influence the correction function has on the input data points towards increasing a given point's likelihood with respect to the target distribution.
The other main hyperparameter $\lambda\in\mathbb R_{\ge 0}$ is the \emph{synthetic augmentation percent}, and it controls the ratio of synthetic data to real data in each iteration of fine-tuning.
When $\gamma=0$, we recover iterative re-training with synthetic augmentation considered in \cite{bertrand2023stability}.
And if we choose the synthetic augmentation percent to be $\lambda=0$, then each generation simply corresponds to fine-tuning the model on the same dataset that it was trained on initially.

We now use \emph{iterative fine-tuning} interchangeably with the more general term self-consuming loop. We also consider the idealized correction function for our theoretical analysis, and a broader family of practical \emph{correction functions} for different data types.

\section{Theoretical Analysis}\label{sec:math}

\subsection{Preliminaries}

We mostly follow the notation from \cite{bertrand2023stability}, except for introducing the correction function $\pi_\gamma$.
Let us denote by $p_{\mathrm{data}}$ the ground truth probability distribution that we want to train a generative model to estimate.
Suppose we have some dataset $\mathcal D_{\mathrm{real}}=\{x_i\}_{i=1}^n$ sampled from $p_{\mathrm{data}}$.
We write $\hat p_{\mathrm{data}}=(1/n)\sum_{i=1}^n\delta_{x_i}$. More generally, we use a hat to denote the empirical distribution over finitely many samples from the corresponding distribution.

Suppose that we have a class of generative models parameterized by $\Theta\subset\mathbb R^d$.
We denote by $p_\theta$ a probability distribution in this class with model parameters $\theta\in \Theta$.
We define the optimal model parameters within this class to be

\vspace{-.3em}
\begin{equation}\label{eq:theta_star_defn_body}
\theta^\star = \argmax_{\theta'\in\Theta}\mathbb E_{x\sim p_{\mathrm{data}}}[\log p_{\theta'}(x)],
\end{equation}
where we break ties by minimizing $\|\theta^\star\|$.
Typically, such optimal parameters yield a model $p_{\theta^\star}$ which closely approximates the oracle ground truth distribution $p_{\mathrm{data}}$, but doesn't equal it exactly; accordingly, we define the Wasserstein-2 distance between the distributions to be
\begin{equation}\label{eq:epsilon_defn}
\varepsilon:=d_W(p_{\theta^\star}, p_{\mathrm{data}}).
\end{equation}
The model weights for the first generation are naturally defined according to the optimization
\begin{equation}\label{eq:theta_0_body}
    \theta_0^n:=
    \argmax_{\theta'\in\Theta}[
\mathbb E_{x\sim \hat p_{\mathrm{data}}}[\log p_{\theta'}(x)]].
\end{equation}
This corresponds to training on the finite subset $\mathcal D_{\mathrm{real}}$.
Next, let us suppose that the model weights from generation $t$ are denoted $\theta_t^n$.
We will formalize a procedure for updating these weights for the next generation to obtain $\theta_{t+1}^n$.
For this, we need to define our correction function, and then we will use it to define the weight update.

\begin{definition}
For any probability distribution, and for any $\gamma\in\mathbb R_{\ge 0}$, we define the \emph{correction of strength $\gamma$} of distribution $p_\theta$ to be the distribution
\begin{equation}\label{eq:correction_mapping}
\pi_\gamma p_\theta(x):=\frac{{p_\theta}(x)+\gamma p_{\theta^\star}(x)}{1+\gamma},
\end{equation}
where $p_{\theta^\star}$ is defined in \eqref{eq:theta_star_defn_body}.
For any augmentation percentage $\lambda\ge 0$, we define the \emph{weight update} mapping to be
\begin{align}\label{eq:hat_h_weight_update}
\pi_\gamma\mathcal G_\lambda^n(\theta)
&:=\localargmax_{\theta'\in\Theta}\hat{\mathcal  H}(\theta,\theta')\\
&:=\localargmax_{\theta'\in\Theta}\Big[
\mathbb E_{x\sim \hat p_{\mathrm{data}}}[\log p_{\theta'}(x)]]\nonumber\\
&\qquad\qquad\qquad\,\,\,\,\,\,\,\,\,\,\,+\lambda \mathbb E_{x\sim {\widehat{\pi_\gamma p_{\theta}}}}[\log p_{\theta'}(x)]\Big],\nonumber
\end{align}
where $\hat p_{\mathrm{data}}$ and $\widehat{\pi_\gamma p_{\theta}}$ are empirical distributions of size $n$ and $\lfloor\lambda\cdot n\rfloor$ respectively.
\end{definition}

To continue our discussion from before, our iterative weight update is defined as $\theta_{t+1}^n:=\pi_\gamma\mathcal G_\lambda^n(\theta_t^n)$.

Note that we use an global maximization in \eqref{eq:theta_0_body} when defining the initial parameters $\theta_0^n$, but we use a local maximization when computing our parameter update in \eqref{eq:hat_h_weight_update}.
This difference is analogous to the differences between how model weights update during initial training, where parameter updates are more global, and during fine-tuning, where parameter updates are more local.

\subsubsection{Understanding the correction $\pi_\gamma p_\theta(x)$}

For $\gamma=0$, the correction mapping in \eqref{eq:correction_mapping} simplifies to $\pi_0 p_\theta=p_\theta$, which is just the original distribution; this corresponds to no correction at all.
For $\gamma = 1$, it is $\pi_1 p_\theta=(p_\theta + p_{\theta^\star})/2$.
And for $\gamma=\infty$, it is $\pi_\infty p_\theta=p_{\theta^\star}$, which corresponds to the optimal distribution.
So as $\gamma$ increases from $0$ to $\infty$, the distribution $\pi_\gamma p_\theta$ has a likelihood profile that matches $p_\theta$ less, and $p_{\theta^\star}$ more.
As $p_{\theta^\star}$ is the optimal model in our generative model class, this means that as $\gamma$ increases from $0$ to $\infty$, we have that $\pi_\gamma p_\theta(x)$ is a PDF which better represents the target likelihood that we want to estimate through training the generative model.

In our theoretical formulation, we consider correction functions that correct the probability distribution $p_\theta$, rather than the more intuitive (and practical) case of a correction function that corrects individual points that the distribution is defined over.
In Appendix~\ref{appendix:self_correction_pointwise}, we specify sufficient conditions under which a pointwise correction function is guaranteed to correspond to a distribution-wise correction function of the same form as those which we consider in our theoretical study and therefore can enjoy the theoretical stability guarantees we prove.
We also provide a concrete example of a projection function, in the Gaussian case, which provably satisfies those conditions.
We conduct a series of experiments on this toy example in Section~\ref{sec:experiments_gaussian}.

\subsubsection{Understanding the weight update $\pi_\gamma \mathcal G_\lambda^n(\theta)$}

The weight update $\pi_\gamma G_\lambda^n(\theta)$ in \eqref{eq:hat_h_weight_update} is a formalization of the intended output of fine-tuning $p_{\theta}$ on $\mathcal D_{\mathrm{real}}\cup\mathcal D_{\mathrm{synth}}$, where $\mathcal D_{\mathrm{real}}=\{x_i\}_{i=1}^n$ is the ground truth dataset of size $n$, and $\mathcal D_{\mathrm{synth}}=\{\tilde x_i:\tilde x_i\sim\widehat{\pi_\gamma p_\theta}\}_{i=1}^{\lfloor \lambda\cdot n\rfloor}$ is the synthesized-and-corrected dataset of size $\lfloor \lambda\cdot n\rfloor$.
In other words, in an ideal run of stochastic gradient descent fine-tuning, the model weights $\theta$ should update to $\pi_\gamma \mathcal G_\lambda^n(\theta)$, as defined in \eqref{eq:hat_h_weight_update}, when trained on $\mathcal D_{\mathrm{real}}\cup \mathcal D_{\mathrm{synth}}$.

Intuitively, the weight update $\theta\mapsto\pi_\gamma \mathcal G_\lambda^n(\theta)$ avoids the loss of variance in the generated data by ensuring that at each step, the model is trained on synthetic data which is likelier to have been sampled from the diverse target distribution. 
This positive phenomenon is more pronounced when the correction strength $\gamma$ is larger.

\subsection{Assumptions}

In order to prove our main result, we need some regularity assumptions about the learning procedure.
Informally speaking, we will assume that 
    the class of generative models that we consider is smoothly parameterized by its model weights; 
    the loss landscape is concave near the ideal model weights;
    and the class of generative models does an increasingly good job approximating the target data distribution as the dataset size increases.
We formally quantify and state these hypotheses in Assumption~\ref{assumption:body_of_paper}.

\begin{assumption}\label{assumption:body_of_paper}
The following are true.
\begin{enumerate}
\setlength\itemsep{0.0em} 
    \item There exists some $L>0$ such that, for all $\theta$ sufficiently close to $\theta^\star$, the mapping $x\mapsto\nb_\theta^2 \log p_\theta(x)$ is $L$-Lipschitz.

    \item The mapping $\theta\mapsto\mathbb E_{x\sim p_{\mathrm{data}}}[\log p_\theta(x)]$ is continuously twice differentiable locally around $\theta^\star$, and there exists some $\alpha>0$ such that $\mathbb{E}_{x\sim p_{\text{data}}}\left[\nabla_\theta^2 \log p_{\theta}(x)\right]|_{\theta^\star} \preceq -\alpha I_d \prec 0.$

    \item There exist $a,b,\varepsilon_{\text{OPT}}\ge 0$ and a neighborhood $U$ of $\theta^\star$ such that, for any $\delta\in(0,1)$, with probability $1-\delta$ over the samplings, we have\footnote{The map $\pi_\gamma G_\lambda^\infty$ is defined similarly to $\pi_\gamma G_\lambda^n$ in \eqref{eq:hat_h_weight_update}, but with $\hat p_{\mathrm{data}}$ replaced with $p_{\mathrm{data}}$, and with $\widehat{\pi_\gamma p_{\theta}}$ replaced with $\pi_\gamma p_{\theta}$. See Appendix~\ref{appendix:our_math} for more details. This estimate is identical to the analogous Assumption 3 used in \cite{bertrand2023stability}, with the only difference being it is applied to our iterative fine-tuning update function. See  Appendix~\ref{appendix:assumption_3} for further discussion.}
    \begin{equation}\label{eq:tau_defn}
        \|\pi_\gamma \mathcal G_\lambda ^n(\theta)-\pi_\gamma \mathcal G_\lambda ^\infty(\theta)\|
        \le \varepsilon_{\text{OPT}}+\frac{a}{\sqrt n}\sqrt{\log\frac{b}{\delta}}.
    \end{equation}
    for all $\theta \in U$ and $n\in\mathbb N$. Denote this bound by $\tau_n(\delta)$.

\end{enumerate}
\end{assumption}

In Assumption~\ref{assumption:body_of_paper} (2), the notation ``$\preceq$'' corresponds to the Loewner order on symmetric matrices: we write that $A\preceq B$ if $B-A$ is positive semi-definite, and $A\prec B$ if $B-A$ is positive definite.
In particular, Assumption~\ref{assumption:body_of_paper} (2) implies that the matrix $\mathbb{E}_{x\sim p_{\text{data}}}\left[\nabla_\theta^2 \log p_{\theta}(x)\right]|_{\theta^\star}$ is negative definite, and its largest eigenvalue is at most $-\alpha$.
And Assumption~\ref{assumption:body_of_paper} (3) mirrors the main assumption in \cite{bertrand2023stability}; it is motivated by generalization bounds in deep learning, see e.g. \cite{jakubovitz2019generalization,ji2021understanding}.
The interested reader can consult Appendix \ref{appendix:assumption_3} for more details on this assumption.

\subsection{Iterative Fine-Tuning with Correction}

We now have the language to state our main result, which essentially says that if the initial parameters $\theta_0$ are sufficiently close to the optimal model parameters $\theta^{\star}$, and if the augmentation percentage $\lambda$ is sufficiently small, then under iterative fine-tuning with correction, we can expect our subsequent model parameters to stay close to $\theta^{\star}$.

\begin{theorem}[Stability of Iterative Fine-Tuning with Correction]\label{thm:2_shortened}
Fix an augmentation percentage $\lambda\in\mathbb R_{>0}$ and a correction strength $\gamma\in\mathbb R_{\ge 0}$.
Suppose we have an iterative fine-tuning procedure defined by the rule  $\theta_{t+1}^n=\pi_\gamma\mathcal G_\lambda^n(\theta_t^n)$, and suppose that Assumption~\ref{assumption:body_of_paper} holds.
Define the constant
\begin{align*}
    \rho(\lambda)&:=\rho(\lambda;\alpha,\varepsilon,L)
    :=\frac{\lambda(\alpha+\varepsilon L)}{\alpha-\lambda (\alpha+\varepsilon L)}
\end{align*}
and fix any $\delta\in(0,1)$.
If $\theta_0$ is sufficiently close to $\theta^\star$, and if $\lambda\left(1+\frac{\varepsilon L}{\alpha}\right)<\frac{1+\gamma}{2+\gamma}$, then $\rho(\lambda)/(1+\gamma)<1$, and it follows that the stability estimate holds with probability $1-\delta$:
\begin{align}\label{eq:thm_2_bound}
    \|\theta_{t}^n&-\theta^\star\|\\
    &\le 
    \tau_n(\delta/t)
    \sum_{i=0}^t
    \left(\frac{\rho(\lambda)}{1+\gamma}\right)^i
    + \left(\frac{\rho(\lambda)}{1+\gamma}\right)^t
    \|\theta_0^n-\theta^\star\|\nonumber
\end{align}
for all $t > 0$.
\end{theorem}

We prove Theorem~\ref{thm:2_shortened} in Appendix~\ref{appendix:our_math}.

\begin{remark}
If we apply Theorem~\ref{thm:2_shortened} with correction strength $\gamma=0$, then the iterative fine-tuning procedure trains successively on a combination of raw synthetic data that has not been corrected using a correction function and ground truth data.
This is exactly the case considered in \cite{bertrand2023stability}.
Accordingly, the bound in \eqref{eq:thm_2_bound}, applied with $\gamma=0$, exactly recovers their result.
\end{remark}

\begin{corollary}\label{corollary:main_theorem}
Under the assumptions from Theorem~\ref{thm:2_shortened}, iterative fine-tuning with any amount of correction outperforms iterative fine-tuning without correction--in the sense that it is exponentially more stable, and it results in better model weights.
\end{corollary}
\begin{proof}[Proof of Corollary~\ref{corollary:main_theorem}]
We apply Theorem~\ref{thm:2_shortened} with $\gamma=0$, which corresponds to no correction, as well as with $\gamma>0$, which corresponds to any amount of correction.
For any $\gamma>0$, we notice that the RHS of \eqref{eq:thm_2_bound} is strictly smaller than when $\gamma=0$.
This guarantees better stability as $t\to\infty$, as well as model weights $\theta_t^n$ closer to $\theta^\star$.
\end{proof}


\begin{example}
If we apply Theorem~\ref{thm:2_shortened} with correction strength $\gamma\to\infty$, then the bound \eqref{eq:thm_2_bound} in Theorem~\ref{thm:2_shortened} limits to $\tau_n(\delta/t)$.
This implies that the practical iterate $\theta_t^n$ approaches the ideal model paramaters, and is at worst some constant away, that depends on error from the optimization procedure, as well as statistical error from using finitely many ground truth data samples $n$.
\end{example}

Note that Theorem~\ref{thm:2_shortened} relies on the assumption that the initial model parameters $\theta_0$ are sufficiently close to the ideal model parameters $\theta^\star$, and also that the augmentation percentage $\lambda$ is sufficiently small.
We hypothesize that these assumptions can be relaxed in the case where a correction function participates in the iterative fine-tuning procedure--intuitively, the correction function should compensate for errors that arise from $\theta_0^n$ being worse, as well as errors that arise from incorporating more synthetic data.
We frame this in the following conjecture.

\begin{conjecture}\label{conjecture:body}
In the case of iterative fine-tuning with correction, we may relax how close the initial model parameters $\theta_0^n$ need to be to the optimal model parameters $\theta^\star$, as well as choose a larger synthetic augmentation percentage $\lambda$, while still retaining the improved stability estimate \eqref{eq:thm_2_bound}.
\end{conjecture}

We provide empirical evidence for Conjecture~\ref{conjecture:body} in Section~\ref{sec:experiments} on the human motion synthesis task.
In fact, Theorem~\ref{thm:2_shortened} represents partial progress towards this conjecture.
Namely, according to Theorem~\ref{thm:2_shortened}, for large correction strength $\gamma$, we can effectively choose a synthetic augmentation percentage that is twice as large as we would be able to without any correction, and still be able to meet the assumptions of the theorem.
This is because $\lim_{\gamma\to\infty}\frac{1+\gamma}{2+\gamma}=1$, which is twice as large as the bound when $\gamma=0$.

\section{Toy Example: Gaussian}\label{sec:experiments_gaussian}


We first assume oracle knowledge of the ground truth distribution, and use a toy example to directly demonstrate the impact of the correction strength $\gamma$ on model performance and stability as stated in Theorem~\ref{thm:2_shortened} and Corollary~\ref{corollary:main_theorem}. Our ground truth distribution is a 2-dimensional isotropic Gaussian centered at the origin, i.e., $\theta^\star=((0,0),I_2)$,  and our correction is ``distribution-wise'' in this idealized scenario. We consider the more practical setting, where we don't have oracle knowledge of the target distribution a priori, and where the data correction is ``point-wise'', in the empirical studies in the following two sections.
Further, in Appendix~\ref{appendix:self_correction_pointwise}, we show that, in theory, sufficiently well-behaved pointwise correction functions indeed correspond to distribution-wise correction functions.

\begin{figure}[t]
\begin{center}
\centerline{\includegraphics[width=0.95\columnwidth]{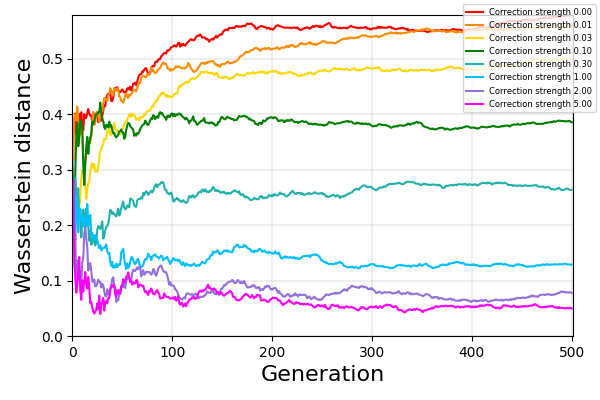}}
\vspace{-1em}
\caption{Empirical results from our Gaussian toy example.
The graph demonstrates that increasing the correction strength $\gamma$, with a fixed augmentation ratio of $\lambda=0.5$, improves performance and stability after self-consuming iterations.}
\label{fig:toy_example_w2}
\vspace{-3em}
\end{center}
\end{figure}

Concretely, our ground truth dataset contains $50$ points sampled from the target distribution, which are used to estimate $\theta_0^{50}=(\mu_0,\Sigma_0)\in\mathbb R^6$. We fix our synthetic augmentation percentage to be $\lambda=0.5$, and inductively synthesize a new dataset $\mathcal D_{\mathrm{synth}}=\{y_i\sim\mathcal N(\mu_t,\Sigma_t)\}_{i=1}^{25}$.
We implement a correction function to map $\mathcal D_{\mathrm{synth}}$, which was sampled from $p_{\theta_t^{50}}$, to a dataset $\mathcal D_{\mathrm{corrected}}$, which is likelier to have been sampled from the target density $p_{\theta^\star}$.
We do this by sampling $\mathcal D_{\mathrm{corrected}}$ from the \emph{middle density} corresponding to a given correction strength $\gamma$: 
\begin{equation}\label{eq:middle_density}
\pi_\gamma \hat p_{\theta_t^{50}}(x)
:=\frac{\hat p_{\theta_t^{50}}(x)+\gamma p_{\theta^\star}(x)}{1+\gamma},
\end{equation}
where $\hat p_{\theta_t^{50}}$ is the empirical PDF obtained from $\mathcal D_{\mathrm{synth}}$. 

We logarithmically accrue synthetic data points to simulate the case of fine-tuning.
We obtain the updated model parameters $\theta_{t+1}^{50}$ by computing the sample mean and covariance on this augmented dataset.
In Figure~\ref{fig:toy_example_w2}, we present the Wasserstein distance between the origin-centered isotropic Gaussian target distribution and the distribution defined by the parameters $\theta_t^{50}$ at each iteration $t$.
Our results illustrate how increasing the correction strength $\gamma$ adds stability and results in convergence near better Wasserstein scores in later generations, in accordance with Theorem~\ref{thm:2_shortened}.
The experiments also demonstrate how even a very small increase in $\gamma$ can improve performance over the baseline, in accordance with our claim of exponential improvement in Corollary~\ref{corollary:main_theorem}.

\section{Toy Example: MNIST}\label{sec:experiments_mnist}

Our proof uses the optimal target PDF $p_{\theta^\star}$ to define the correction function $\pi_\gamma$.
This is empirically validated by the Gaussian toy experiment, which assumes knowing the true target distribution.
In practice, the correction function only depends on the ability to map synthesized data to data which is \textit{likelier} to have been sampled from the ground truth distribution.
Crucially, \textit{this can be achieved without having a complete description of the target distribution. }
For example, with our human motion experiments, we will demonstrate that point-wise correction based on the laws of physics is one proxy approach to make a sample more likely, without knowing the true target distribution.

One has the freedom to explore alternative approaches to data correction for more general data types, such as images. 
For example, one simple heuristic is to identify the ``anchor'' or ``exemplar'' images, which are intuitively representative and likely. 
The correction function can then be implemented as mapping or morphing synthesized data towards its nearest anchor, to make the synthesized data more representative and likely. 
In this section, we implement this approach on MNIST and study its performance.

\begin{figure}[t]
\begin{center}
\centerline{\includegraphics[width=0.7\columnwidth]{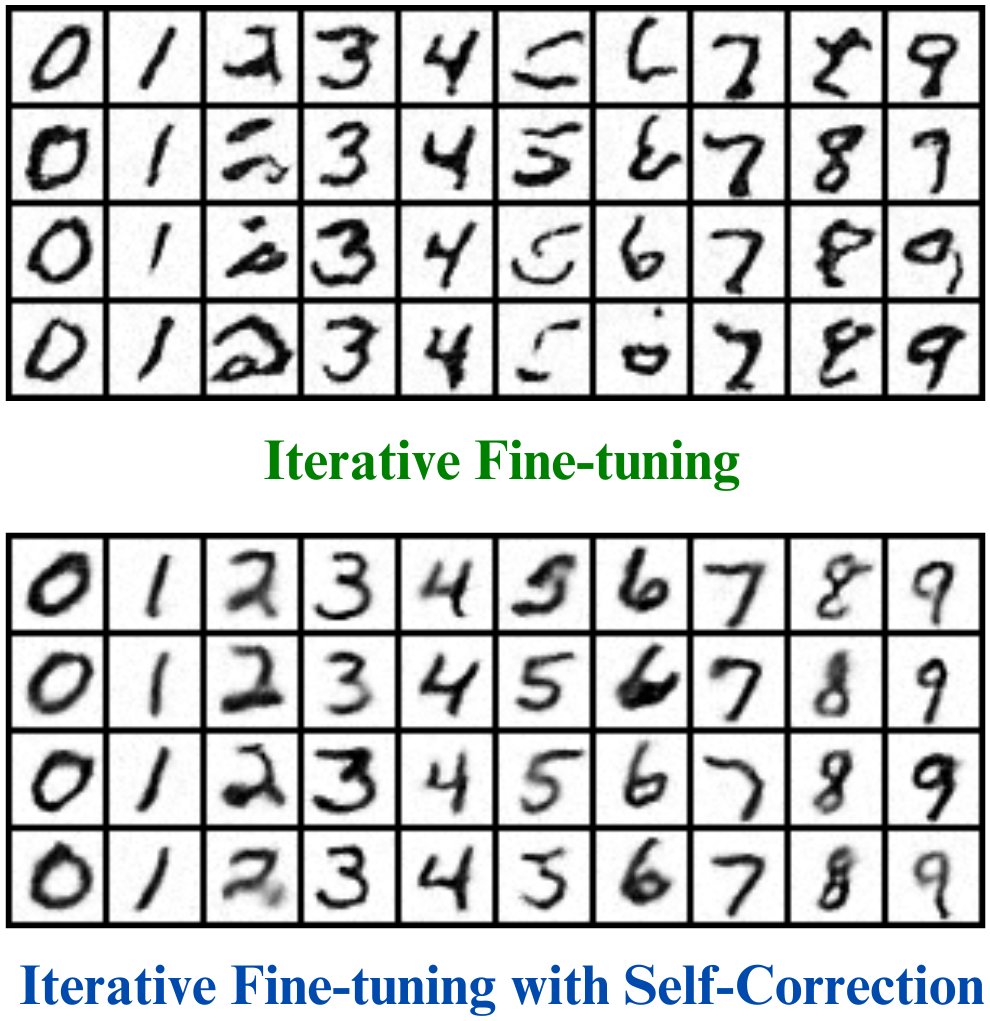}}
\vspace{-1em}
\caption{Empirical results from our MNIST toy example.
These synthesized images demonstrate that after 50 self-consuming iterations at 150\% augmentation percentage, the model which is trained using \textbf{\color{blue_better}iterative fine-tuning with self-correction} is able to generate higher quality samples than the model trained using \textbf{\color{green_better}iterative fine-tuning} without any self-correction.}
\label{fig:toy_example_mnist}
\vspace{-3em}
\end{center}
\end{figure}

For our MNIST~\cite{lecun1998gradient} experiments, we train a diffusion model~\cite{ho2020denoising} for class-conditional image generation, using a train split of size $n=12000$.
For our iterative fine-tuning experiments, we train the model for 20 epochs, then synthesize $\lambda\cdot 12000/10$ images for each digit, and then augment the ground truth dataset with these to train on for the next generation; every following generation follows the same procedure, but only trains for a single epoch.
We vary our experiments over augmentation percentages $\lambda\in\{0.2, 0.5, 1.0, 1.5\}$.
To define our self-correction operation, we first compute $K$-means clusters over the training split for each digit.
Our iterative fine-tuning with self-correction experiments use the same setup described above, except instead of training on the synthesized images, we train on the synthesized and then corrected images, where ``correcting'' an image means finding the nearest centroid in the $K$ centroids for that digit that we computed at the start of training.
We swept the values $K\in\{1,2,4,\dots,1024\}$, and we found that any reasonably large $K$ results in the same general trend where self-correction improves the metrics and stability. We report our results for $K=16$, which performs the best.

We present images synthesized using our trained models in Figure~\ref{fig:toy_example_mnist}.
These synthesized images demonstrate that iterative fine-tuning eventually generates many low quality and illegible digits, and this problem is solved by applying our self-correction operation.
Further experiment details, including graphs of the FID metrics for each generation that provide rigorous evidence for this trend across augmentation percentages, can be found in Appendix~\ref{appendix:MNIST}.
Our empirical results demonstrate that applying self-correction improves performance during iterative fine-tuning for our MNIST image generation task across self-consuming generations, and this relative performance is amplified when the augmentation percentage is larger.
The behavior that we observe is consistent with our theoretical results in Section~\ref{sec:math}, as well as our human motion experiments in Section~\ref{sec:experiments}.

\section{Human Motion Synthesis}\label{sec:experiments}

\begin{figure*}[ht]
\begin{center}
\centerline{\includegraphics[width=2.1\columnwidth]{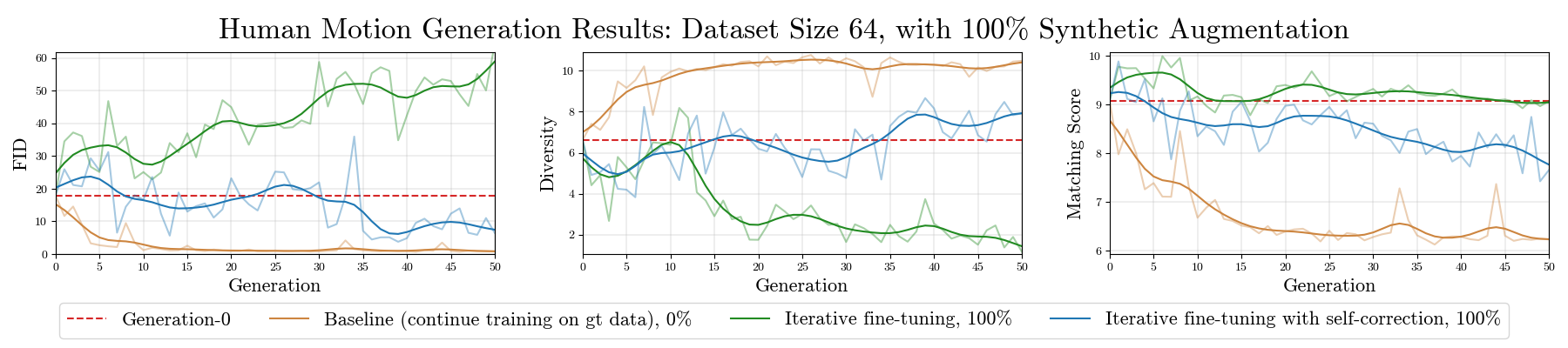}}
\centerline{\includegraphics[width=2.1\columnwidth]{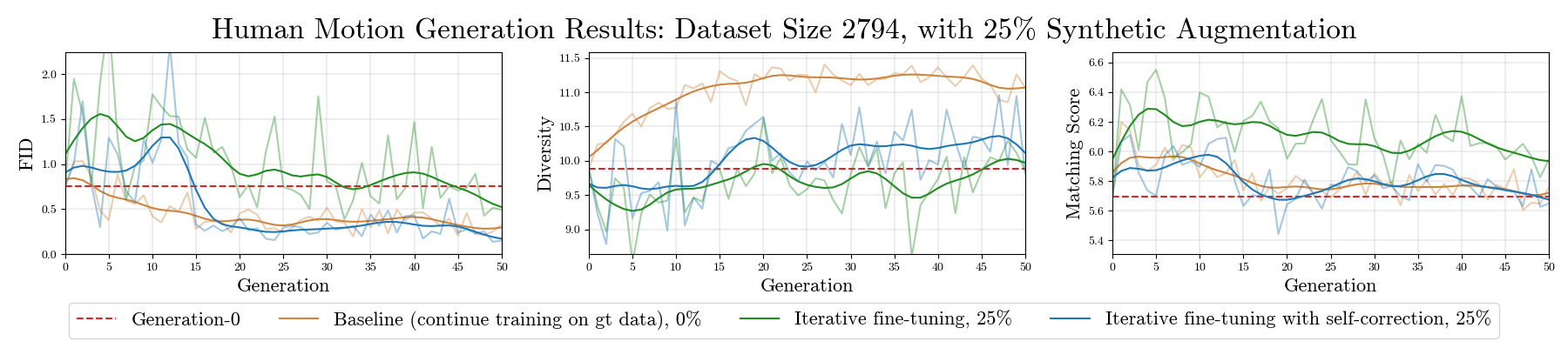}}
\vspace{-1em}
\caption{Results from our human motion experiments on iterative fine-tuning with self-correction.
These graphs show evaluation metrics for the last checkpoint for every generation.
This is the checkpoint used for sampling in the iterative fine-tuning experiments, and it is also the checkpoint where training is resumed with this new partially synthesized dataset.
We can see that \textbf{\color{blue_better}with self-correction}, the iterative fine-tuning procedure more stably converges to a better FID score, and more quickly.
When the dataset size is smaller ($n=64$, above) we can see that iterative fine-tuning \textbf{\color{green_better}with no self-correction} has a flat Matching score, as well as diverging FID and Diversity scores, indicating model collapse.
And when the dataset size is larger ($n=2794$, below), there is less collapse for iterative fine-tuning with no self-correction, although the variance of the FID score is worse, as is the average FID across generations. 
In both cases, we see that iterative fine-tuning \textbf{\color{blue_better}with self-correction} outperforms iterative fine-tuning \textbf{\color{green_better}with no self-correction}, and is competitive with the \textbf{\color{orange}baseline} after many generations.
}
\vspace{-2em}
\label{fig:human_motion_metrics_body_of_paper}
\end{center}
\end{figure*}

Theorem~\ref{thm:2_shortened} states that, in theory, iterative fine-tuning with correction should be more stable than iterative fine-tuning without correction.
Crucially, the stability estimates that we prove rely on the dataset size, the synthetic augmentation percentage, how expressible the generative model class is, and having an idealized correction function. To validate how our theory works beyond toy examples, we conduct a case study on human motion synthesis with diffusion models~\cite{tevet2023human}.
We believe this is a natural setting to test our iterative fine-tuning with correction framework, because synthesizing natural motions is a challenging problem, but there is a natural and intuitive way to automatically correct them at scale--namely, using a physics simulator.

\subsection{Generative Model}

For our generative model, we use the Human Motion Diffusion Model (MDM) \cite{tevet2023human}.
This is a classifier-free diffusion-based generative model for the text-to-motion generation task, where the model receives as input a description of a motion sequence (e.g. ``get down on all fours and crawl across the floor''), and outputs a sequence of skeleton poses which attempt to embody that prompt. Synthesizing human motion is challenging not only for the diverse and compositional text prompts, but also due to failure of physics obeying-ness (e.g. feet skating, floating, penetrating a surface), which is not explicitly enforced by deep generative models.

\subsection{Physics Simulator as Self-Correction Function}

For our self-correction function, we use Universal Humanoid Control (UHC)~\citep{Luo2021DynamicsRegulatedKP}, which is an imitation policy that operates inside the MuJoCo physics simulator~\citep{mujoco}.
Given an input sequence of humanoid skeleton poses, UHC attempts to imitate the motion sequence, constrained by the laws of physics imposed by the physics simulator, and it outputs a new motion sequence that is the closest possible approximation it can replace it with.
For example, if an input motion sequence violates the laws of physics by having a foot penetrate through the floor, then the motion sequence output by UHC will attempt to remove that physically impossible artifact while maintaining the semantic integrity of the original input motion. We use VPoser~\citep{SMPL-X:2019} and SMPL~\citep{SMPL:2015} to translate joint representations between the human motion generator and the physics simulator.

The physics simulator allows us to self-correct a synthesized motion automatically. Our underlying assumption is that by enforcing the physics obeying-ness (via the simulator) and closeness to the synthesized motion (via the imitation objective), the self-correction function would act as similar as an idealized corrector as possible.

\begin{figure*}[ht]
\begin{center}
\centerline{\includegraphics[width=2.1\columnwidth]{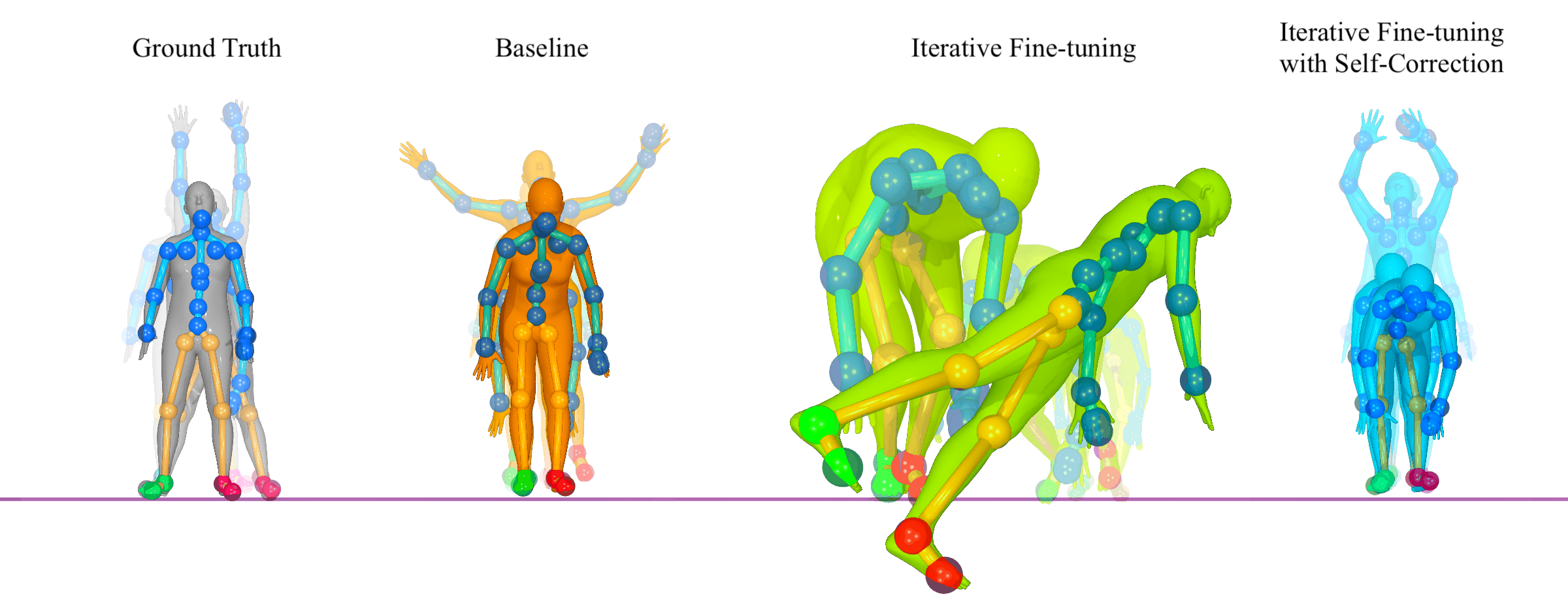}}
\vspace{-1em}
\caption{How does the self-correction operation affect iterative fine-tuning, qualitatively?
Here we present some visualizations.
The prompt which describes the ground truth motion, and which we use to generate the three other motions, is: \textit{``a person stands with feet wide, stretches both hands up over his head and then swings down by the waist and hangs arms down before standing up''}.  
We can see that the \textbf{\color{green_better}iterative fine-tuning model} produces a motion where the human moves closer to the camera than the others; this is evidence of model collapse, as moving feet is irrelevant to the prompt.
Additionally, this motion produces single frames that suddenly snap to a physically impossible position--note the leg penetration through the ground plane. 
These negative artifacts do not exist in the motions synthesized from the \textbf{\color{gray}ground truth}, \textbf{\color{orange}baseline model}, or \textbf{\color{blue_better}iterative fine-tuning with self-correction model}.
Lastly, we note that the iterative fine-tuning motion depicted here is semantically similar to crawling. 
We observe in our experiments with smaller dataset sizes that the iterative fine-tuning model generates less diverse outputs than the baseline model and the iterative fine-tuning with self-correction model, and that this crawling pattern appears more often in the latter.
Each snapshot is taken at exactly frame 105 of their respective videos. 
The two motions on the right come from models that were iteratively fine-tuned for 50 generations, with a train set of size $n=64$, and a synthetic augmentation percentage of $25\%$.
For all pictures of the human, the camera is fixed at the same position, and for consistency the image is not resized.}
\label{fig:qualitative_human_motion_picture_3}
\vspace{-2em}
\end{center}
\end{figure*}

\subsection{Experimental setup}

We preprocess the MoVi~\citep{ghorbani2021movi} subset of HumanML3D~\citep{Guo_2022_CVPR} using the official code implementation of HumanML3D. We filter out movements involving interactions with chairs, as UHC by default does not handle human-object interactions.
We take as our train split the train split from HumanML3D, intersected with our filtered subset of MoVi, and likewise for the test split.
This procedure yields a train set of size $n=2794$ and a test set of size $546$. We further randomly select a smaller training set of $n\in\{64, 128, 256\}$ examples, to simulate the more challenging scenario when the initial generative model is sub-optimal (due to data scarcity). The smaller data also enables us to explore larger synthetic augmentation percentage due to compute constraints.
From here, the iterative re-training procedure follows Algorithm~\ref{alg:physics_augmentation_loop}.
We spell it out in this concrete experimental setup.

We first train on the ground truth train split until the model is nearly converged, using all the default hyperparameters from MDM.
We evaluate and save this last checkpoint from generation $0$.
From here, for each generation $t\in\{1,2,\dots,50\}$, we run three sets of experiments.
\vspace{-1em}
\begin{enumerate}[A.]
\setlength\itemsep{0.0em}
	\item \emph{Baseline}: fine-tune the latest checkpoint from generation $t-1$ for $m$ batches on ground truth dataset $\mathcal D_{\mathrm{real}}$.
	
	\item \emph{Iterative fine-tuning:} fine-tune the latest checkpoint  from generation $t-1$ on $\mathcal D_{\mathrm{real}}\cup \mathcal D_{\mathrm{synth},t-1}$ for $m$ batches.
    Here, $\mathcal D_{\mathrm{synth},t-1}$ is a synthetic dataset of size $\lfloor\lambda\cdot n\rfloor$ generated from the checkpoint for generation $t-1$, using randomly chosen prompts from the train split.

	\item \emph{Iterative fine-tuning with self-correction:} fine-tune the latest checkpoint  from generation $t-1$ on $\mathcal \mathcal D_{\mathrm{real}}\cup \mathrm{UHC}(\mathcal D_{\mathrm{synth},t-1})$ for $m$ batches.
    Here, $\mathrm{UHC}(\mathcal D_{\mathrm{synth},t-1})$ denotes a synthetic dataset of size $\lfloor\lambda\cdot n\rfloor$ generated from the latest checkpoint for generation $t-1$, using randomly chosen prompts from the train split, which is then corrected by UHC.

\end{enumerate}

\vspace{-3mm}
We experiment with synthetic augmentation percentages $\lambda\in\{0.05, 0.10, 0.15, 0.20, 0.25\}$ on the larger dataset; we set the number of batches seen during generation $0$ to be $3125$, and the number of batches seen for each later generation to be $m=625$. 
Separately, we experiment with synthetic augmentation percentages $\lambda\in\{0.25, 0.50, 0.75, 1.00\}$ on the smaller datasets; we set the number of batches seen during generation $0$ to be $78*k$ for dataset size $64*k$, and the number of batches seen for each later generation $t>0$ to be $m=16$. 
We choose to control how many data points the model sees across each generation, rather than controlling some other quantity like the number of epochs, as this allows each experiment to compare against its baseline in a controlled way, which in turn allows them to compare against each other in a controlled way.

We compute every evaluation one time for each checkpoint using the evaluation script provided in the original MDM codebase.
Regardless of the train split size, we perform sampling for evaluation using all 546 motion sequences from the test split, since the FID score is sensitive to generated dataset size.
We use the same hyperparameters as those used for MDM, including batch size $64$, AdamW \cite{loshchilov2018decoupled} with learning rate $1e-4$, and classifier-free guidance parameter $2.5$.
And for UHC we used the \texttt{uhc\_explicit} model for imitation.

\vspace{-4mm}
\subsection{Quantitative Analysis of Results}\label{subsec:quantative_analysis}

\noindent For each of these experiments we report the metrics from MDM, as used by \cite{Guo_2022_CVPR}:
        FID measures how similar the distribution of generated motions is to the ground truth distribution;
        Diversity measures the variance of the generated motions;
	and Matching Score measure how well the generated motions embody the given text prompt.
In Figure~\ref{fig:human_motion_metrics_body_of_paper} we present results from experiments on our $64$-size dataset with $100\%$ synthetic augmentation, as well as our $2794$-size dataset with $25\%$ synthetic augmentation.

Our experimental results confirm our theoretical results, that iterative fine-tuning with self-correction outperforms iterative fine-tuning without self-correction, in the sense that the graphs are generally more stable across generations, and approach better evaluation metric values.
In particular, Theorem~\ref{thm:2_shortened} and Corollary~\ref{corollary:main_theorem} claim that any amount of idealized self-correction will improve the stability bound during iterative fine-tuning.
Our results in Figure~\ref{fig:human_motion_metrics_body_of_paper} demonstrate that the FID score is lower and more stable across generations when applying self-correction, and generally higher and less stable than the baseline, where there is no self-consuming training at all.
We conduct experiments across multiple seeds, and we find empirically that this general phenomenon holds consistently, where the self-correction technique consistently yields improved training dynamics over iterative fine-tuning with no correction.
Graphs from these runs can be found in Appendix~\ref{appendix:multiple_seeds}.

Our experimental results also provide empirical evidence for Conjecture~\ref{conjecture:body}.
Observe that in the baseline experiments in Figure~\ref{fig:human_motion_metrics_body_of_paper}, the FID score decreases across generations, which indicates that the initial model parameters $\theta_0^n$ are not that close to the optimal model parameters $\theta^\star$; additionally, the augmentation percentages considered in the graph are $25\%$ and $100\%$.
Conjecture~\ref{conjecture:body} claims that performing self-correction during iterative fine-tuning improves performance, even when the initial model weights are sub-optimal and simultaneously the synthetic augmentation percentage is large.
This claim is confirmed by Figure~\ref{fig:human_motion_metrics_body_of_paper}.
We direct the curious reader to Appendix~\ref{appendix:more_human_motion_graphs}, where we present graphs for all of the above listed training set sizes and augmentation percentages, providing additional empirical evidence for Theorem~\ref{thm:2_shortened}, Corollary~\ref{corollary:main_theorem}, and Conjecture~\ref{conjecture:body}.

\vspace{-2mm}
\subsection{Qualitative Analysis of Results}\label{subsec:qualitative_analysis}

We visually inspect the generated human motion sequences in order to analyze what concrete effect the self-correction has on iterative fine-tuning.
We find that the correctness and diversity of synthesized motions are improved by the self-correction procedure, in agreement with our quantitative analysis in Subsection~\ref{subsec:quantative_analysis}.
We present snapshots of our synthesized motions in Figure~\ref{fig:qualitative_human_motion_picture_3}, and we analyze the motions in the caption.
In short, we find that physics-disobeying artifacts such as floor penetration or floating become more pronounced without the self-correction. 
We also find that in the model without self-correction, the humanoid sometimes performs movements completely unrelated to the prompt; our model with self-correction fixes these negative phenomena.
We direct the curious reader to Appendix~\ref{appendix:human_motion_qualitative}, where we present more examples from our qualitative analysis, as well as our project webpage, where we provide side-by-side video comparisons.

\vspace{-2mm}

\section{Conclusion}

\vspace{-1mm}
Our paper investigates the learning of generative models when the training data includes machine-generated contents. We investigate how self-correction functions, which automatically correct synthesized data points to be more likely under the true data distribution, can stabilize self-consuming generative model training. Our theoretical results show that self-correction leads to exponentially more stable model training and smaller variance, which we illustrate with a Gaussian toy example. We then demonstrate how physics simulators can serve as a self-correction function for the challenging human motion synthesis task, where models trained with our self-correcting self-consuming loops generate higher quality motions, and manage to avoid collapse even at a high synthetic data to real data ratio.
Future work includes exploring self-correcting functions for more diverse applications, such as language modeling and text-to-image generation, and investigating when self-consuming training may lead to overall better generative models.

\section*{Acknowledgments}

We would like to thank Stephen H. Bach, Quentin Bertrand, Carsten Eickhoff, Gauthier Gidel, Jeff Hoffstein, Zhengyi Luo, Singh Saluja, and Ye Yuan for useful discussions. 
We would also like to thank the anonymous reviewers.
This work is supported by the Samsung Advanced Institute of Technology, Honda Research Institute, and a Richard B. Salomon Award for Chen Sun.
Our research was conducted using computational resources at the Center for Computation and Visualization at Brown University.

\section*{Impact Statement}
This paper presents work whose goal is to provide theoretical analysis and practical tools to address the data contamination issue caused by machine-generated content. There are many potential societal consequences of our work, none which we feel must be specifically highlighted here.

\bibliography{bib}
\bibliographystyle{icml2024}

\newpage
\appendix
\onecolumn

\section{Mathematical Theory: The Proof of Theorem \ref{thm:2_shortened}}\label{appendix:our_math}

In this appendix, we provide a full account of the mathematical details of the theorems and their proofs appearing in the main body of the paper.
Our proof technique has the same framework as \cite{bertrand2023stability} because our theoretical analysis generalizes theirs to the case where you have a self-correction function in the self-consuming loop.

\subsection{Mathematical Setup and Notation}
\begin{definition}
Define the optimal model parameters to be
\begin{equation}\label{eq:theta_star}
\theta^\star\in\argmax_{\theta'\in\Theta}\mathbb E_{x\sim p_{\mathrm{data}}}[\log p_{\theta'}(x)],
\end{equation}
chosen so that $\|\theta^\star\|$ has minimal norm within this set.
Let $\theta$ be any model parameters.
Then the \emph{correction of strength $\gamma$} of distribution $p_\theta$ towards $p_{\theta^\star}$ is a new distribution, denoted $\pi_\gamma p_\theta$, defined according to the rule
\[\pi_\gamma p_\theta(x):=\frac{{p_\theta}(x)+\gamma p_{\theta^\star}(x)}{1+\gamma}.\]
This is illustrated in Figure~\ref{fig:projection_gaussian}.
Let $\theta_t$ be the parameters of the model trained after $t$ generations.
We define the \emph{iterative fine-tuning with correction update mapping} to be
\begin{align}\label{eq:pi-G-infinity}
\pi_\gamma\mathcal G_\lambda^\infty(\theta)
&:=\localargmax_{\theta'\in\Theta}\mathcal H(\theta,\theta')
:=
\localargmax_{\theta'\in\Theta}[
\mathbb E_{x\sim p_{\mathrm{data}}}[\log p_{\theta'}(x)]]+\lambda \mathbb E_{x\sim\pi_\gamma p_{\theta}}[\log p_{\theta'}(x)]]\\
\pi_\gamma\mathcal G_\lambda^n(\theta)
&:=\localargmax_{\theta'\in\Theta}\hat{\mathcal H}(\theta,\theta')
:=\localargmax_{\theta'\in\Theta}[
\mathbb E_{x\sim \hat p_{\mathrm{data}}}[\log p_{\theta'}(x)]]+\lambda \mathbb E_{x\sim {\widehat{\pi_\gamma p_{\theta}}}}[\log p_{\theta'}(x)]].
\end{align}
Notice that in the finite case, we're optimizing by taking samples from an empirical distribution.
In contrast, in the infinite case, there is zero statistical error, since the parameter update is done with access to an infinite sampling budget at each generation $t$.
The finite case is the more practical case, when we have some statistical error (so we only have access to finite sampling at each generation). Since the parameter space of the generative model class might be limited, there might be a small difference between the distribution corresponding to the optimal parameters and the target distribution $p_{\mathrm{data}}$; we capture this difference via the Wasserstein-2 distance and denote
\begin{equation}\label{eq:wasserstein_defn}
    \varepsilon:=d_W(p_{\theta^\star}, p_{\mathrm{data}}).
\end{equation}
Let
\begin{align}\label{eq:defn_of_H}
    \mathcal{H}_1(\theta') := \mathbb E_{x\sim p_{\mathrm{data}}}[\log
     p_{\theta'}(x)], \qquad 
     \mathcal{H}_2(\theta, \theta') := \mathbb E_{x\sim \pi_\gamma p_{\theta}}[\log p_{\theta'}(x)].
\end{align}
and note that $\mathcal H(\theta,\theta')
   =\mathcal H_1(\theta')
       + \mathcal \lambda \mathcal H_2(\theta,\theta')$.
\end{definition}

We first establish that the correction map is truly a mapping of probability distributions as well as some of its elementary properties. 
\begin{lemma}\label{lem:physics_projection_fun_facts}
The correction map has the following properties.
    \begin{enumerate}
    
        \item $\pi_\gamma p_\theta$ is a probability distribution.
        
        \item Strengths $0, 1,\infty$ correspond to $p_\theta$, the average of $p_\theta$ and $p_{\theta^\star}$, and $p_{\theta^\star}$, respectively.
        
        \item For any $x\in\mathbb R^n$, if $\gamma>1$, then 
        \[\|\pi_\gamma p_\theta(x)-p_{\theta^\star}(x)\|
        \le 
        \|\pi_\gamma p_\theta(x)-p_{\theta}(x)\|,\] 
        and if $\gamma<1$, then the inequality is flipped.
        In other words, $\pi_\gamma p_\theta$ is a better estimate of the ideal distribution $p_{\theta^\star}$ than $p_{\theta}$ is, precisely when the projection strength is more than $1$.
    \end{enumerate}
\end{lemma}
\begin{proof}
For the first point, $\pi_\gamma p_\theta$ is a probability distribution because it is a convex combination of probability distributions.
For example, we can compute that
\begin{align*}
    \int_{\mathbb R^d} \pi_\gamma p_\theta dx
    &=\frac{1}{1+\gamma}\int_{\mathbb R^d} {p_\theta}(x)dx+ \frac{\gamma}{1+\gamma}\int_{\mathbb R^d} {p_{\theta^\star}}(x)dx
    =\frac{1}{1+\gamma}\cdot 1+\frac{\gamma}{1+\gamma}\cdot 1
    =1.
\end{align*}
The second point follows immediately from the definition of $\pi_\gamma p_\theta$.
For the third point, we can estimate that
\begin{align*}
    \|\pi_\gamma p_\theta(x)-p_{\theta^\star}(x)\|
    &=\left\|\frac{p_\theta(x)+\gamma p_{\theta^\star}(x)}{1+\gamma}
    -\frac{p_{\theta^\star}(x)(1+\gamma)}{1+\gamma}\right\|\\
    &=\frac{1}{1+\gamma}\cdot\|p_{\theta}(x)-p_{\theta^\star}(x)\|\\
    &\le\frac{\gamma}{1+\gamma}\cdot\|p_{\theta^\star}(x)-p_{\theta}(x)\|\\
    &=\left\|\frac{p_\theta(x)+\gamma p_{\theta^\star}(x)}{1+\gamma}
    -\frac{p_{\theta}(x)(1+\gamma)}{1+\gamma}\right\|\\
    &=\|\pi_\gamma p_\theta(x)-p_{\theta}(x)\|
\end{align*}
when $\gamma>1$.
The inequality flips when $\gamma< 1$.
\end{proof}

Intuitively, it is clear that we cannot hope to prove general results about generative models without assuming something about the mapping $\theta\mapsto p_{\theta}$. We now state the two assumptions we require in order to make our theoretical arguments; note that they are precisely the same assumptions made in \cite{bertrand2023stability}. The first assumption is a local Lipschitzness property that we will exploit via the Kantorovich-Rubenstein duality:
\begin{assumption}\label{assumption:1}
For $\theta$ close enough to $\theta^\star$, the mapping $x\mapsto\nabla_\theta\nabla_\theta \log p_\theta(x)$ is $L$-Lipschitz.
\end{assumption}

The second assumption is a local regularity and concavity condition:
\begin{assumption}\label{assumption:2}
The mapping $\theta\mapsto\mathbb E_{x\sim p_{\mathrm{data}}}[\log p_\theta(x)]$ is continuously twice differentiable locally around $\theta^\star$ and $\mathbb{E}_{x\sim p_{\text{data}}}\left[\nabla_\theta\nabla_\theta \log p_{\theta}(x)\right]_{\theta^\star} \preceq -\alpha I_d \prec 0.
$
\end{assumption}

We next show the existence and uniqueness of $\pi_{\gamma}\mathcal{G}_{\lambda}^{\infty}(\infty)$ locally around $\theta^{\star}$.

\begin{proposition}[The Local Maximum Likelihood Solution is Unique]\label{prop:implicit_function_theorem_application}
The following are true:
\begin{enumerate}[A.]
    \item 
    There exists an open neighborhood $U\subset\mathbb R^d$ containing $\theta^\star$ and a continuous function $g:U\to\mathbb R^d$ such that $g(\theta^\star)=\theta^\star$, and 
\begin{equation}
    \nabla_{\theta'}\mathcal H(\theta,\theta')|_{\theta,g(\theta)} = 0
\end{equation}
for every $\theta\in U$.

    \item Given optimal model parameters $\theta^\star$ as in \eqref{eq:theta_star} that follow Assumptions \ref{assumption:1} and \ref{assumption:2}, we have that, if $\varepsilon L < \alpha$, then for all $\lambda>0$ and $\theta$ in a small enough neighborhood $U$ around $\theta^\star$, there exists a unique local maximizer $\pi_\gamma\mathcal G_\lambda^\infty(\theta)$ in $U$.
\end{enumerate}
\end{proposition}

\begin{proof}
\textbf{We first prove part A.}
It suffices to apply the Implicit Function Theorem to the map
\begin{align}\label{eq:derivative_of_H_at_maximum}
    \mathbb R^{2d}\to\mathbb R^d:(\theta,\theta')\mapsto \nabla_{\theta'}\mathcal H(\theta,\theta')|_{\theta,\theta'}
\end{align}
in an open neighborhood of $(\theta^\star,\theta^\star)$.
To do this, we need to show the following:
\begin{enumerate}[i)]

    \item The map vanishes at $(\theta^\star,\theta^\star)$, i.e.
    \begin{equation}
        \nabla_{\theta'}\mathcal H(\theta,\theta')|_{\theta^\star,\theta^\star} = 0.
    \end{equation}

    \item The Jacobian matrix at $(\theta^\star,\theta^\star)$ is invertible, i.e.,
    \begin{equation}\label{eq:Jacobian_of_H}
        \nabla_{\theta'}\nabla_{\theta'}\mathcal H(\theta,\theta')|_{\theta^\star,\theta^\star}\qquad \text{is invertible.}
    \end{equation}
    
\end{enumerate}

\textbf{We first prove i).}
Recall from the definition \eqref{eq:pi-G-infinity} that $\pi_\gamma\mathcal G_\lambda^\infty(\theta)=\argmax_{\theta'\in\Theta}\mathcal H(\theta,\theta')$.
This means that for any $\theta$, $\pi_\gamma\mathcal G_\lambda^\infty(\theta)$ is the choice of $\theta'$ which maximizes $\mathcal H(\theta,\theta')$.
In particular, for $\theta=\theta^\star$, we have that $\theta'=\pi_\gamma\mathcal G_\lambda^\infty(\theta^\star)$ is the choice which maximizes $\mathcal H(\theta^\star,\theta')$.
But $\pi_\gamma\mathcal G_\lambda^\infty(\theta^\star)=\theta^\star$ by Proposition~\ref{prop:fixed_point}.
This implies that its derivative is zero at $\theta'=\theta^\star$, meaning $\nabla_{\theta'}\mathcal H(\theta,\theta')|_{\theta^\star,\theta^\star}=0$, as needed.

\textbf{Now we prove ii).}
In order to show that the matrix \eqref{eq:Jacobian_of_H} is invertible, it suffices to show it is close to another matrix which is invertible.
A natural choice is the matrix
\begin{equation}\label{eq:close_to_Jacobian}
M = (1+\lambda)\nabla_{\theta'}\nabla_{\theta'}\mathbb{E}_{x\sim p_{\text{data}}} [\log p_{\theta'}(x)]|_{\theta^\star}.
\end{equation}
First of all, note that this matrix indeed exists; by Assumption 2 \ref{assumption:2}, we know the map $\theta' \mapsto \mathbb{E}_{x\sim p_{\text{data}}} [\log p_{\theta'}(x)]$ is continuously twice differentiable locally near $\theta^{\star}$.
We can estimate that the matrices \eqref{eq:Jacobian_of_H} and \eqref{eq:close_to_Jacobian} are indeed close as follows:
\begin{align*}
    \|\nabla_{\theta'}\nabla_{\theta'}&\mathcal H(\theta,\theta')|_{\theta^\star,\theta^\star}
    - (1+\lambda)\nabla_{\theta'}\nabla_{\theta'}\mathbb{E}_{x\sim p_{\text{data}}} [\log p_{\theta'}(x)]_{\theta^\star}\|\\
    &=\|\nabla_{\theta'}\nabla_{\theta'}[\mathbb E_{x\sim p_{\mathrm{data}}}\log p_{\theta'}(x)+\lambda \mathbb E_{x\sim \pi_\gamma p_{\theta}}p_{\theta'}(x)]|_{\theta^\star,\theta^\star}
    - (1+\lambda)\nabla_{\theta'}\nabla_{\theta'}\mathbb{E}_{x\sim p_{\text{data}}} [\log p_{\theta'}(x)]_{\theta^\star}\|\\
    &=\lambda \|[\nabla_{\theta'}\nabla_{\theta'}\mathbb E_{x\sim \pi_\gamma p_{\theta}}\log p_{\theta'}(x)]|_{\theta^\star,\theta^\star}
    - \nabla_{\theta'}\nabla_{\theta'}\mathbb{E}_{x\sim p_{\text{data}}} [\log p_{\theta'}(x)]_{\theta^\star}\|\\
    & = \lambda\|[\nabla_{\theta'}\nabla_{\theta'} \mathbb{E}_{x\sim p_{\theta^\star}}\log p_{\theta'}(x)]_{\theta^\star}
    -[\nabla_{\theta'}\nabla_{\theta'}\mathbb{E}_{x\sim p_{\text{data}}} \log p_{\theta'}(x)]_{\theta^\star}\|\\
    & = \lambda\|[ \mathbb{E}_{x\sim p_{\theta^\star}}\nabla_{\theta'}\nabla_{\theta'}\log p_{\theta'}(x)]_{\theta^\star}
    -[\mathbb{E}_{x\sim p_{\text{data}}}\nabla_{\theta'}\nabla_{\theta'} \log p_{\theta'}(x)]_{\theta^\star}\|\\
    &\le L\lambda\mathbb E_{(x,x')\sim p_{\theta^\star}\times p_{\mathrm{data}}}
    \|[\nabla_{\theta'}\nabla_{\theta'}\log p_{\theta'}(x)]_{\theta^\star}
    -\nabla_{\theta'}\nabla_{\theta'}\log p_{\theta'}(x)]_{\theta^\star}\|\\
    &\le\lambda\varepsilon L
\end{align*}
where 
    the first equality follows from the definition of $\mathcal H$ in \eqref{eq:defn_of_H};
    the second equality follows from some cancellation;
    the third equality follows the fact that the derivatives are constant with respect to $\theta$, and $\pi_\gamma p_{\theta^\star}=p_{\theta^\star}$ by Lemma \ref{lem:physics_projection_fun_facts};
    we exchange the derivative and the expectation in equation 4 using the Dominated Convergence Theorem, since Assumption 1 \ref{assumption:1} says that $x\mapsto\nabla_\theta\nabla_\theta\log p_\theta(x)$ is $L$-Lipschitz;
    the fifth estimate follows from Kantorovich-Rubinstein Duality; 
    and the final estimate is the definition of Wasserstein distance \eqref{eq:wasserstein_defn}.
    
Finally, we verify $M$ is indeed invertible.
Assumption 2 \ref{assumption:2} implies that the largest eigenvalue of $M$ is at most $-(1+\lambda)\alpha$.
Therefore, since all eigenvalues of $M$ are nonzero, $M$ is invertible.
We can now apply the implicit function theorem to \eqref{eq:derivative_of_H_at_maximum}, and part A follows immediately.

\textbf{Next, we prove part B.}
Let $d_U = \sup_{\theta\in U} d_W(p_{\theta^{\star}},p_{\theta})$. To verify that $g(\theta)$ is a local maximizer of \eqref{eq:derivative_of_H_at_maximum}, it suffices to show that $\nb_{\theta'}\nb_{\theta'} \mathcal{H}(\theta, g(\theta)) \prec 0$. By Assumption 2 \ref{assumption:2}, we know $\nb_{\theta'}\nb_{\theta'}\mathcal{H}_1(\theta^{\star}) \prec -\alpha I_{d}$ and since $\theta' \mapsto \nb_{\theta'}\nb_{\theta'}\mathcal{H}_1(\theta')$ is \textit{continuously} twice differentiable locally near $\theta^{\star}$, we also have $\nb_{\theta'}\nb_{\theta'}\mathcal{H}_1(g(\theta)) \prec - \alpha I_{d}$. Thus, we have
\begin{align*}
    \nb_{\theta'}\nb_{\theta'} \mathcal{H}(\theta, g(\theta)) & = \nb_{\theta'}\nb_{\theta'} \mathcal{H}_1(g(\theta')) + \lambda \nb_{\theta'}\nb_{\theta'} \mathcal{H}_2(\theta, g(\theta))\\
    & = (1+\lambda) \nb_{\theta'}\nb_{\theta'} \mathcal{H}_1(g(\theta)) + \lambda(\nb_{\theta'}\nb_{\theta'} \mathcal{H}_2(\theta, g(\theta))- \nb_{\theta'}\nb_{\theta'} \mathcal{H}_1(g(\theta)))\\
    &\preceq -\alpha(1+\lambda) I_{d}+ \lambda L \left(\frac{1}{1+\gamma}d_W(p_{\theta}, p_{\theta^{\star}}) + \varepsilon\right)I_{d}, 
\end{align*}
where the last step follows from Kantorovich-Rubsenstein duality:
\begin{align*}
    \|\nb_{\theta'}\nb_{\theta'} &\mathcal{H}_2(\theta,\theta')- \nb_{\theta'}\nb_{\theta'} \mathcal{H}_1(\theta') \|\\ 
    &\leq  \|\nb_{\theta'}\nb_{\theta'} \mathcal{H}_2(\theta, \theta')- \nb_{\theta'}\nb_{\theta'} \mathcal{H}_2(\theta^{\star}, \theta') \| +  \|\nb_{\theta'}\nb_{\theta'} \mathcal{H}_2(\theta^{\star}, \theta')-\nb_{\theta'}\nb_{\theta'} \mathcal{H}_1(\theta')\|\\
    & = \|\int_{\mathbb{R}^d} \nb_{\theta'}\nb_{\theta'} \log p_{\theta'}(x) \frac{p_\theta(x) + \gamma p_{\theta^{\star}}(x)}{1+\gamma} \,dx - \int_{\mathbb{R}^d} \nb_{\theta'}\nb_{\theta'} \log p_{\theta'}(x) p_{\theta^{\star}}(x) \,dx\|\\
    & \;\;\; + \|\mathbb{E}_{x\sim p_{\text{data}}} [\log p_{\theta'}(x)] - \mathbb{E}_{x\sim p_{\theta^{\star}}} [\log p_{\theta'}(x)]\|\\
    & \leq   \frac{1}{1+\gamma}\|\int_{\mathbb{R}^d} \nb_{\theta'}\nb_{\theta'} \log p_{\theta'}(x) \left(p_\theta(x) -p_{\theta^{\star}}(x)\right) \,dx \| + L\varepsilon\\
    & = \frac{1}{1+\gamma}\|\mathbb{E}_{x\sim p_{\theta}} [\log p_{\theta'}(x)] - \mathbb{E}_{x\sim p_{\theta^{\star}}} [\log p_{\theta'}(x)] \| + L\varepsilon \\
    &\leq \frac{L}{1+\gamma}d_W(p_{\theta}, p_{\theta^{\star}}) + L\varepsilon\\
    & \leq \frac{L}{1+\gamma}d_U + L\varepsilon
\end{align*}
Thus, to have $\nb_{\theta'}\nb_{\theta'} \mathcal{H}(\theta, g(\theta)) \prec 0$, it is sufficient that
\begin{align*}
    -\alpha(1+\lambda) + \lambda L \left(\frac{1}{1+\gamma}d_U + \varepsilon\right) < 0,
\end{align*}
which is guaranteed for all $\lambda > 0$ by $\alpha > L\varepsilon$ and $d_U \leq \frac{\alpha(1+\gamma)}{\lambda}$. This concludes the proof.
\end{proof}

Further, as we would expect, $\theta^{\star}$ is a fixed point of $\pi_{\gamma} \mathcal{G}_{\lambda}^{\infty}$:
\begin{proposition}[The optimal parametric generative model is a fixed point]\label{prop:fixed_point}
For any given data distribution $p_{\mathrm{data}}$, any $\theta^\star$ as defined by \eqref{eq:theta_star}, and for all $\lambda>0$, we have $\pi_\gamma \mathcal{G}_\lambda ^\infty (\theta^\star)=\theta^\star$.
\end{proposition}
\begin{proof}
Unpacking definition \eqref{eq:pi-G-infinity} shows that $\pi_\gamma \mathcal{G}_\lambda ^\infty (\theta^\star) = \mathcal{G}_\lambda ^\infty (\theta^\star)$, and we know by Proposition 4 from \cite{bertrand2023stability} that $\mathcal{G}_\lambda ^\infty (\theta^\star)=\theta^\star$.
\end{proof}

\subsection{Convergence of Iterative Fine-tuning with Correction for Infinite Sampling}

We now have the required setup to state and prove a convergence result for iterative fine-tuning assuming infinite access to underlying probablity distributions. We need the following result, which is a technical lemma that provides a computation of the Jacobian of $\pi_{\gamma}G_{\lambda}^{\infty}$ at $\theta^{\star}$ as well as a spectral bound, both essential for the proof of Theorem \ref{thm:1}.

\begin{lemma} \label{spectral}
We define the matrices
\begin{align}
    A&:=(\nabla_{\theta',\theta'}^2\mathcal H_1(\theta'))|_{\theta^\star}\\
    B&:= \nabla_{\theta,\theta'}^2\mathbb{E}_{x\sim p_{\theta}}[\log p_{\theta'}(x)]\big|_{\theta^\star,\theta^\star} \\
    C&:=\nabla_{\theta',\theta'}^2 \mathbb{E}_{x\sim p_{\theta}}[\log p_{\theta'}(x)]\big|_{\theta^*,\theta^*}
\end{align}
Recall the definition of $\pi_\gamma\mathcal G_\lambda^\infty(\theta)$ from \eqref{eq:pi-G-infinity}.
Since $\gamma$ and $\lambda$ are fixed, denote $\pi \mathcal{G}(\theta) = \pi_{\gamma}\mathcal{G}_{\lambda}^{\infty}(\theta).$ Finally, let $\mathcal{J}(\pi \mathcal{G}(\theta)):=\nabla_\theta \pi_\gamma \mathcal{G}_\lambda^\infty(\theta)|_\theta$ denote the Jacobian of $\pi_\gamma \mathcal{G}_\lambda^\infty(\theta)$.
\begin{enumerate}[I.]
\item There exists an open neighborhood $U \subseteq \Theta$ containing $\theta^{\star}$ such that for all $\theta \in U$, we have 
\begin{align}\label{eq:implicit_differentiation}
    \mathcal{J}(\mathcal{\pi G}(\theta)) =
    &\,\,-\left(\nb^2_{\theta', \theta'} \mathcal{H}(\theta, \pi\mathcal{G}(\theta))\right)^{-1} \cdot \lambda \nb^{2}_{\theta, \theta'} \mathcal{H}_2(\theta, \pi\mathcal{G}(\theta)).
\end{align}

  \item We have that $\nb^{2}_{\theta, \theta'} \mathcal{H}_2(\theta^\star,\theta^\star)=\frac{B}{1+\gamma}$, and $B=-C$, so the Jacobian of $\pi \mathcal{G}$ at $\theta^{\star}$ is
   \begin{equation}\label{eq:jacobian_at_theta_star}
      \mathcal{J}(\pi\mathcal{G}(\theta^\star)) = (I + \lambda A^{-1} C)^{-1} \cdot \frac{\lambda}{1+\gamma} A^{-1} C
  \end{equation}
  \item The spectral norm of $A^{-1} C$ can be bounded as 
  \begin{equation}\label{eq:AinvCbound}
      \|A^{-1}C\| \leq 1+ \frac{L \varepsilon}{\alpha}.
  \end{equation}
\end{enumerate}
\end{lemma}

\begin{proof}
\textbf{We first prove I.}
We apply Proposition \ref{prop:implicit_function_theorem_application}.
Part A of that proposition gives us a function $g:U\to\mathbb R^d$ such that $\nabla_{\theta'}\mathcal H(\theta,\theta')_{\theta,g(\theta)}=0$.
But part $B$ of that proposition says that there exists a unique local maximizer inside $U$, and this local maximizer is $\pi_\gamma\mathcal G_\lambda^\infty$.
This implies that $\nabla_{\theta'}\mathcal H(\theta,\theta')_{\theta,\pi_\gamma\mathcal G_\lambda^\infty(\theta)}=0$.
Next, we implicitly differentiate this equation with respect to $\theta$.
Recall that when you have an equation of the form $f(x,y)=0$, and implicitly differentiate it in the form $f(x,g(x))=0$ with respect to $x$, you obtain $\frac{\partial f}{\partial x}+\frac{\partial f}{\partial y}\frac{\partial g}{\partial x}=0$, and solving for $\frac{\partial g}{\partial x}$ yields $\frac{\partial g}{\partial x} = -\left(\frac{\partial f}{\partial y}\right)^{-1}\frac{\partial f}{\partial x}$.
We apply this formula with 
\[(x,f,g)=
(\theta,
\theta\mapsto\nabla_{\theta'}\mathcal H(\theta,\theta')_{\theta,\pi_\gamma\mathcal G_\lambda^\infty(\theta)}, 
\theta\mapsto\pi_\gamma\mathcal G_\lambda^\infty(\theta))\] 
and obtain \eqref{eq:implicit_differentiation}, as desired.

\textbf{Now we prove II.} 
We can compute that
\begin{align}
   \nb^{2}_{\theta', \theta} \mathcal{H}_2(\theta, \theta') 
   &=\nb_{\theta'}\nb_{\theta} \mathbb{E}_{x\sim\pi_\gamma p_\theta}[\log p_{\theta'}(x)]\\
   &=\nb_{\theta'}\nb_{\theta} \int_{x\in\mathbb R^d}\log p_{\theta'}(x)\left(\frac{p_\theta(x)+\gamma p_{\theta^\star}(x)}{1+\gamma}\right)dx\\
   &=\frac{1}{1+\gamma}\nb_{\theta'}\nb_{\theta} \int_{x\in\mathbb R^d}\log p_{\theta'}(x)p_\theta(x)dx\\
   & = \frac{1}{1+\gamma} \nb^2_{\theta', \theta} \mathbb{E}_{x\sim p_\theta}[\log p_{\theta'}(x)]\\
   &=\frac{1}{1+\gamma}B
\end{align} 
where the third equality holds because the integral containing $p_{\theta^\star}$ is constant with respect to $\theta$. 
Next, we can compute that
\begin{align}
   B
    &=\int_X \nabla_{\theta'}\log p_{\theta'}(x)\nabla_\theta p_\theta(x)dx\Big|_{\theta^*,\theta^*}\\
    &=\int_X [\nabla_{\theta}\log p_{\theta}(x)][\nabla_\theta p_\theta(x)]dx\Big|_{\theta^*,\theta^*}\\
    &=\int_X \nabla_\theta[p_{\theta}(x)\nabla_{\theta}\log p_{\theta}(x)]dx\Big|_{\theta^*,\theta^*}
    -\int_X p_{\theta}(x)(\nabla_{\theta}\nabla_{\theta}\log p_{\theta}(x))dx\Big|_{\theta^*,\theta^*}\\
    &=\int_X \nabla_\theta\left[p_\theta(x)\frac{\nabla_\theta p_\theta(x)}{p_\theta(x)}\right]dx\Big|_{\theta^*,\theta^*}
    -
    \nabla_{\theta',\theta'}^2 \mathbb{E}_{x\sim p_{\theta}}[\log p_{\theta'}(x)]\Big|_{\theta^*,\theta^*}\\
    &=-C,
\end{align}
where the third equality follows from the product rule for gradients,
\begin{align}
    \nabla_\theta[p_\theta(x)\nabla_{\theta}\log p_{\theta}(x)]
    &=p_\theta(x)(\nabla_{\theta}\nabla_{\theta}\log p_{\theta}(x))
    +
    [\nabla_\theta p_\theta(x)][\nabla_{\theta}\log p_{\theta}(x)].
\end{align}

Finally, we will prove the formula \eqref{eq:jacobian_at_theta_star} by manipulating \eqref{eq:implicit_differentiation}.
We begin with the rightmost factor in \eqref{eq:implicit_differentiation}.
If we apply these equalities that we just obtained, then we get
\begin{align*}
    \mathcal{J}(\pi \mathcal{G}(\theta^{\star})) &= -\left(\nb^2_{\theta', \theta'} \mathcal{H}(\theta^{\star}, \theta^{\star})\right)^{-1} \cdot \lambda \nb^{2}_{\theta', \theta} \mathcal{H}_2(\theta^{\star}, \theta^{\star})\\
    & = - (A+ \lambda C)^{-1}\cdot \frac{\lambda}{1+\gamma} B\\
    & = - (I+ \lambda A^{-1}C)^{-1} \cdot \frac{\lambda}{1+\gamma} A^{-1} B\\
    & = (I+ \lambda A^{-1}C)^{-1} \cdot \frac{\lambda}{1+\gamma} A^{-1} C
\end{align*}
where the first equality follows from \eqref{eq:jacobian_at_theta_star} along with the fixed point Proposition \ref{prop:fixed_point}, and we are using that $A$ is invertible by Assumption 2 \ref{assumption:2}, which implies all eigenvalues of $A$ are nonzero; in the fourth step we used that $B=-C$.
This proves part II.

\textbf{Now we prove III.} We can bound the operator norm $\|A^{-1}C\|$ as follows:
\begin{align}\label{mult}
\|A^{-1}C\| 
= \| I + A^{-1}(C- A)\|
\leq \|I\| + \|A^{-1}\|\cdot\|C-A\|
\leq 1 + \alpha^{-1} \|C-A\|,
\end{align}
where the first estimate comes from subadditivity and submultiplicativity, and the second comes from the fact that, since $A$ is symmetric, $\|A\| = \max_{\lambda\in \sigma(A)} |\lambda|$, where $\sigma(A)$ is the spectrum of $A$.
Formally, we know by Assumption \ref{assumption:2} that $A$ has eigenvalues $e_1 < e_2 < \dots < e_n \le -\alpha < 0$ and so $|e_n| > \alpha$. 
Therefore, $A^{-1}$ has eigenvalues $1/e_n < 1/e_{n-1} < \dots < 1/e_1 < 0$ and thus $1/|e_n| > 1/|e_{n-1}| > \dots > 1/|e_1|$, which gives us the bound $\|A^{-1}\| = 1/|e_n| < 1/\alpha$ on the matrix norm.
Next, we can estimate that
\begin{align*}
||C-A||
    &= \|\nb^{2}_{\theta', \theta'} \mathbb E_{x\sim p_{\theta^{\star}}}[\log p_{\theta'}(x)]|_{\theta^\star} - \nb^{2}_{\theta', \theta'} \mathbb E_{x\sim p_{\text{data}}}[\log p_{\theta'}(x)]|_{\theta^\star}\| \\
    & = \|\mathbb E_{x\sim p_{\theta^{\star}}}[\nb^2_{\theta',\theta'}\log p_{\theta^{\star}}(x)] - \mathbb{E}_{x\sim p_{\text{data}}}[\nb^2_{\theta',\theta'}\log p_{\theta^{\star}}(x)]\|\\
    &\leq L d_{W}(p_{\theta^{\star}}, p_{\text{data}}) \\
    &= L\varepsilon,
\end{align*}
where in the second equality we exchange the derivative and the expectation in equation 4 using the Dominated Convergence Theorem, since Assumption 1 \ref{assumption:1} says that $x\mapsto\nabla_\theta\nabla_\theta\log p_\theta(x)$ is $L$-Lipschitz;
and in the last estimate, we used Kantorovich-Rubenstein duality. 
This, combined with the estimate \eqref{mult}, yields the bound in \eqref{eq:AinvCbound}.
\end{proof}

We are finally ready to prove our theorem that guarantees convergence to the optimal parameters in the infinite sampling case under certain assumptions, one being the that the initial model parameters $\theta_0$ are sufficiently close to $\theta^{\star}$:

\begin{theorem}[Convergence of Iterative Fine-tuning, Infinite Sampling Case]\label{thm:1}
Suppose we have an iterative fine-tuning procedure defined by the rule  $\theta_{t+1}^\infty=\pi_\gamma\mathcal G_\lambda^\infty(\theta_t^\infty)$.
Let $\theta^\star$ be the parameter vector for the optimal generative model, as in \eqref{eq:theta_star}.
We assume that $\theta^\star$ follows Assumptions \ref{assumption:1} and \ref{assumption:2} from \cite{bertrand2023stability}. 
Suppose also that $\lambda\left(1+\frac{\varepsilon L}{\alpha}\right)<\frac{1+\gamma}{2+\gamma}$.
Then, the Jacobian of $\pi_\gamma G_\lambda^\infty$ satisfies the following bound:
\begin{align}\label{eq:jacobian_bound_physics}
    \|\nabla_\theta \pi_\gamma \mathcal{G}_\lambda^\infty(\theta^\star)\|_2
    &\le \frac{1}{1+\gamma}\cdot \frac{\lambda(\alpha + \varepsilon L)}{\alpha-\lambda (\alpha+\varepsilon L)} < 1.
\end{align}

Consequently, there exists a $\delta>0$ such if $\theta_0\in\Theta$ satisfies $\|\theta_0-\theta^\star\|\le\delta$, then starting training at $\theta_0$ and having $\theta_{t+1}=\pi_\gamma\mathcal G_\lambda^\infty(\theta_t)$, we have that $\lim_{t\to\infty}\theta_t\to\theta^\star$.
Furthermore, if we define 
\begin{align}\label{eq:rho_of_lambda_definition}
    \rho(\lambda) =  \frac{\lambda(\alpha+\varepsilon L)}{\alpha-\lambda (\alpha+\varepsilon L)},
\end{align}
then we obtain the asymptotic stability estimate\footnote{\cite{bertrand2023stability} could have presented their results in this stronger form, without the big $O$ notation, with very little extra work.}
\begin{align}\label{eq:iterative_bound}
    \|\theta_t-\theta^\star\| \le \left(\frac{\rho(\lambda)}{1+\gamma} \right)^t \|\theta_0-\theta^\star\|.
\end{align}
\end{theorem}
\begin{proof}\label{proof:thm1}
\textbf{We first prove the Jacobian bound \eqref{eq:jacobian_bound_physics}.}
By hypothesis, we know $\lambda (1 + \frac{L\varepsilon}{\alpha}) < 1$, so by Lemma \ref{spectral}(III), we have $\lambda ||A^{-1}C|| < 1$. 
Thus, we can write
\begin{align*}
    (I+ \lambda A^{-1}C)^{-1} & = \sum_{k=0}^{\infty} (-\lambda A^{-1}C)^{k}
\end{align*}
and so
\begin{align*}
    \|(I+\lambda A^{-1} C)^{-1}\| 
    &\leq \sum_{k=0}^{\infty} \lambda^k ||A^{-1}C||^{k}
    = \frac{1}{1 - \lambda ||A^{-1}C||}.
\end{align*}
Applying Lemma \ref{spectral}(2), we get
\begin{align*}
    ||\mathcal{J}(G(\theta^{\star}))|| &\leq ||(I+\lambda A^{-1}C)^{-1}|| \cdot \frac{\lambda}{1+\gamma} ||A^{-1}C||
    \leq \frac{\lambda}{1+\gamma}\cdot \frac{||A^{-1}C||}{1 - \lambda ||A^{-1}C||}.
\end{align*}
Now, it is straightforward to see the RHS above is at most the bound in \eqref{eq:jacobian_bound_physics} if and only if $\alpha\|A^{-1}C\|< \alpha+\varepsilon L$.
But this bound holds because of Lemma \ref{spectral}(III).
This proves the Jacobian bound \eqref{eq:jacobian_bound_physics}, but does not prove that the bound is less than $1$.
For this, we must show that
\begin{align}
\frac{1}{1+\gamma}\cdot \frac{\lambda(\alpha + \varepsilon L)}{\alpha-\lambda (\alpha+\varepsilon L)} < 1.
\end{align}
By clearing denominators and grouping like terms, we can see that this is equivalent to
\begin{align}
\lambda\left(1+\frac{\varepsilon L}{\alpha}\right)<\frac{1+\gamma}{2+\gamma},
\end{align}
which is precisely guaranteed by our hypothesis.

\textbf{We now apply the the Jacobian bound \eqref{eq:jacobian_bound_physics} to prove the asymptotic stability estimate \eqref{eq:iterative_bound}.}
Assume $\lambda$ is sufficiently small so that $\rho(\lambda)/(1+\gamma)<1$.
Then for every $\rho'\in(\rho(\lambda)/(1+\gamma),1)$, there exists $\delta>0$ sufficiently small so that every $\theta_0\in\Theta$ which satisfies $\|\theta_0-\theta^\star\| < \delta$ has the property that $\|\nabla_\theta \pi_\gamma G_\lambda^\infty(\theta_0)\|_2< \rho'$.
Because the map $\pi_\gamma G_\lambda^\infty$ has Jacobian matrix norm less than $1$ in the $\delta$-ball around $\theta^\star$, it is a contraction mapping in this neighborhood.
Concretely, this means that
\begin{equation}\label{eq:contraction_mapping}
    \|\pi_\gamma \mathcal G_\lambda^\infty(\theta)-\pi_\gamma \mathcal G_\lambda^\infty(\theta')\|
    \le 
    \frac{\rho(\lambda)}{1+\gamma}\|\theta-\theta'\|,
\end{equation}
for every $\theta,\theta'$ in the $\delta$-ball around $\theta^\star$.
In particular, for $(\theta,\theta')=(\theta_t,\theta^\star)$ we obtain
\begin{align*}
    \|\theta_{t+1}-\theta^\star\|=\|\pi_\gamma\theta_t-\theta^\star\|
    &=\|\pi_\gamma \mathcal G_\lambda^\infty(\theta_{t})-\pi_\gamma \mathcal G_\lambda^\infty(\theta^\star)\|
    \le\frac{\rho(\lambda)}{1+\gamma}\cdot\|\theta_{t}-\theta^\star\|.
\end{align*}
By induction, the above estimate implies that if $\theta_0$ is in a $\delta$-ball around $\theta^\star$, then so is every successive $\theta_t$.
Therefore the desired estimate \eqref{eq:iterative_bound} now follows by induction on $t$.
\end{proof}

\begin{remark}\label{remark:after_thm_1}
    Taking $\gamma=0$ recovers exactly the result in \cite{bertrand2023stability}. 
    Importantly, the correction function $\pi_{\gamma}$ provides leverage in determining how large the augmentation percentage $\lambda$ can be: choosing a larger correction strength $\gamma$ allows us to choose a larger augmentation percentage $\lambda$ while still retaining theoretical guarantees for convergence. Additionally, for the same choice of augmentation percentage $\lambda$, a larger correction strength $\gamma$ provides a guarantee for an improved rate of convergence.
    See Conjecture~\ref{conjecture:body}.
\end{remark}

\subsection{Stability of Iterative Fine-tuning with Correction for Finite Sampling}

Finally, we prove a stability result for iterative fine-tuning with correction in the presence of statistical error. To do this, we require an assumption that essentially provides probabilistic guarantee that the chosen generative model learns the underlying distribution increasingly better if it has access to more samples: 
\begin{assumption}\label{assumption:3}
    There exist $a,b,\varepsilon_{\text{OPT}}\ge 0$ and a neighborhood $U$ of $\theta^\star$ such that, for any $\delta\in(0,1)$, with probability $1-\delta$ over the samplings, we have
    \begin{equation}
        (\forall \theta \in U)(\forall n\in\mathbb N)\qquad
        \|\pi_\gamma \mathcal G_\lambda ^n(\theta)-\pi_\gamma \mathcal G_\lambda ^\infty(\theta)\|
        \le \varepsilon_{\text{OPT}}+\frac{a}{\sqrt n}\sqrt{\log\frac{b}{\delta}}.
    \end{equation}
\end{assumption}

See Appendix~\ref{appendix:assumption_3} for a discussion about this assumption; we investigated whether to assume a similar bound to the one they assumed in \cite{bertrand2023stability}, or prove our bound from theirs.
In fact, we prove in Appendix~\ref{appendix:assumption_3} that you can in fact deduce something nearly as strong as Assumption \ref{assumption:3} from Assumption 3 in their paper, so we made Assumption \ref{assumption:3} for the sake of a cleaner, more parallel exposition.

\begin{theorem}[Iterative Fine-Tuning Stability Under Correction]\label{thm:2}
Suppose we have an iterative fine-tuning procedure defined by the rule  $\theta_{t+1}^n=\pi_\gamma\mathcal G_\lambda^n(\theta_t^n)$.
In words, this means that the augmentation percentage is $\lambda\in(0,\infty)$ and the correction strength is $\gamma\in[0,\infty)$.
Under the same assumptions of Theorem \ref{thm:1} and Assumption \ref{assumption:3}, there exist $0<\rho < 1$ and $\delta_1>0$ such that if $\|\theta_0^n - \theta^\star\| \le \delta_1$, then for any $\delta_2\in(0,1)$, with probability $1-\delta_2$, we have
\begin{align}\label{eq:thm_2_estimate_appendix}
    \|\theta_{t}^n-\theta^\star\|
    \le 
    \left(\varepsilon_{\text{OPT}}+\frac{a}{\sqrt n}\sqrt{\log\frac{bt}{\delta}}\right)\sum_{i=0}^t
    \left(\frac{\rho(\lambda)}{1+\gamma}\right)^i
    + \left(\frac{\rho(\lambda)}{1+\gamma}\right)^t
    \|\theta_0^n-\theta^\star\|.
\end{align}
\end{theorem}
\begin{proof}
By the triangle inequality, we can estimate that
\begin{align}\label{eq:triangle_inequality}
    \|\theta_{t}^n-\theta^\star\|
    &\le \|\theta_t^n-\pi_\gamma\mathcal G_\lambda^\infty(\theta_{t-1}^n))\|
    +\|\pi_\gamma\mathcal G_\lambda^\infty(\theta_{t-1}^n)-\theta^\star\|\nonumber\\
    &=\|\pi_\gamma\mathcal G_\lambda^n(\theta_{t-1}^n)
    -\pi_\gamma\mathcal G_\lambda^\infty(\theta_{t-1}^n)\|
    +
    \|\pi_\gamma\mathcal G_\lambda^\infty(\theta_{t-1}^n)
    -\pi_\gamma\mathcal G_\lambda^\infty(\theta^\star)\|,
\end{align}
where we applied the fixed point Proposition \ref{prop:fixed_point}.
By Assumption \ref{assumption:3}, the left summand in \eqref{eq:triangle_inequality} is at most $\varepsilon_{\text{OPT}}+\frac{a}{\sqrt n}\sqrt{\log\frac{b}{\delta}}$, with probability $1-\delta$.
Next, recall that in \eqref{eq:contraction_mapping} in the proof of Theorem \ref{thm:1}, we proved that that $\pi_\gamma G_\lambda^\infty$ is a contraction mapping of factor $\rho(\lambda)/(1+\gamma)$ sufficiently close to $U$; this implies that the right summand in \eqref{eq:triangle_inequality} is at most 
$\frac{\rho(\lambda)}{1+\gamma}\|\theta_{t-1}^n-\theta^\star\|$.
Together, these yield the recurrence estimate
\begin{align}
    \mathbb P
    \left(
    \|\theta_{t}^n-\theta^\star\|
    \le \varepsilon_{\text{OPT}}+\frac{a}{\sqrt n}\sqrt{\log\frac{b}{\delta}}
    + \frac{\rho(\lambda)}{1+\gamma}\|\theta_{t-1}^n-\theta^\star\|
    \right) \ge 1-\delta.
\end{align}
Iterating this recurrence for successive time steps yields
\begin{align}\label{eq:recurrence_probability}
    \mathbb P
    \left(
    \|\theta_{t}^n-\theta^\star\|
    \le 
    \left(\varepsilon_{\text{OPT}}+\frac{a}{\sqrt n}\sqrt{\log\frac{b}{\delta}}\right)\sum_{i=0}^t
    \left(\frac{\rho(\lambda)}{1+\gamma}\right)^i
    + \left(\frac{\rho(\lambda)}{1+\gamma}\right)^t
    \|\theta_0^n-\theta^\star\|
    \right) \ge (1-\delta)^t.
\end{align}
Note that \eqref{eq:recurrence_probability} holds for any $\delta\in(0,1)$.
In particular, we can apply \eqref{eq:recurrence_probability} with $\delta:=\delta/t$.
In this case, the Bernoulli inequality lets us estimate that $(1-\delta/t)^t\ge 1-\delta$.
This completes the proof, with $\delta_2=\delta$.
\end{proof}

\begin{remark}
    Theorem \ref{thm:2} recovers the result from \cite{bertrand2023stability} in the case where the correction strength is $\gamma=0$.
    But for a fixed augmentation percentage $\lambda$, for any correction strength $\gamma>0$, this gives stronger stability guarantees than in \cite{bertrand2023stability}.
\end{remark}

\begin{remark}
In a previous version of this manuscript, we claimed that there was an error in the statement of the corresponding theorem in \citep{bertrand2023stability}.
In this version, we retract that claim; we have corresponded with those authors, and they updated their manuscript with additional details to justify their statement.
\end{remark}

\subsection{Discussion: The Main Limitation}\label{app_subsec:limitations}

Our empirical results are for generative modeling tasks where we have access to some ``self-correction'' operation that is easy to compute, as well as automatic; see Sections~\ref{sec:experiments_mnist} and ~\ref{sec:experiments} for more details about these correction functions.
Therefore, the main limitation of our work is that one can only hope to use this self-correction procedure to stabilize training in scenarios where there is some ``self-correction'' function.
For our MNIST experiments, we built a self-correction function from scratch using clustering statistics.
And for our human motion experiments, we used an off-the-shelf human motion imitation model that other researchers built.

\section{Discussion about Assumption~\ref{assumption:body_of_paper}}\label{appendix:assumption_3}

In this section, we show how with a mild boundedness assumption on our generative model parameter update function, we can deduce our Assumption \ref{assumption:3} (which is the same as Assumption \ref{assumption:body_of_paper}, part 3) from the following assumption used in \cite{bertrand2023stability}.
\begin{assumption}\label{assumption:3-old}
    There exist $a,b,\varepsilon_{\text{OPT}}\ge 0$ and a neighborhood $U$ of $\theta^\star$ such that, for any $\delta\in(0,1)$, with probability $1-\delta$ over the samplings, we have
    \begin{equation}
        (\forall \theta \in U)(\forall n\in\mathbb N)\qquad
        \|\mathcal G_\lambda ^n(\theta)-\mathcal G_\lambda ^\infty(\theta)\|
        \le \varepsilon_{\text{OPT}}+\frac{a}{\sqrt n}\sqrt{\log\frac{b}{\delta}}.
    \end{equation}
\end{assumption}
Now, if we make the additional assumption that our generative model parameter update function is locally bounded near $\theta^{\star}$ then we obtain the following.
\begin{proposition}
    Suppose Assumption \ref{assumption:3-old} holds. Suppose also that there exists $B < \infty$ such that for all $n > 0$ and $\theta$ sufficiently close to $\theta^\star$,
\begin{align*}
    \|\mathcal{G}_{\lambda}^{n}(\theta) - \mathcal{G}_{\lambda}^n(\theta^{\star})\| < B \|\theta - \theta^{\star}\|. 
\end{align*}
 Then there exist $a,b,c, \varepsilon_{\text{OPT}}\ge 0$ and a neighborhood $U$ of $\theta^\star$ such that, for any $\delta\in(0,1)$, with probability $1-\delta$ over the samplings, we have
    \begin{equation}
        (\forall \theta \in U)(\forall n\in\mathbb N)\qquad
        \|\pi_{\gamma}\mathcal G_\lambda ^n(\theta)-\pi_{\gamma}\mathcal G_\lambda ^\infty(\theta)\|
        \le c \cdot d_U + \varepsilon_{\text{OPT}}+\frac{a}{\sqrt n}\sqrt{\log\frac{b}{\delta}},
    \end{equation}
    where $d_U = \sup_{\theta \in U} \|\theta - \theta^\star\|.$
\end{proposition}
\begin{proof}
   By the triangle inequality, we have
\begin{align}\label{eq:tri}
\|\pi_{\gamma}\mathcal{G}_{\lambda}^{n}(\theta) - \pi_{\gamma}\mathcal{G}_{\lambda}^{\infty}(\theta) \| &\leq \|\pi_{\gamma}\mathcal{G}_{\lambda}^{n}(\theta) - \mathcal{G}_{\lambda}^{n}(\theta)\| + \|\mathcal{G}_{\lambda}^{n}(\theta) - \mathcal{G}_{\lambda}^{\infty}(\theta)\| + \|\mathcal{G}_{\lambda}^{\infty}(\theta) -\pi_{\gamma}\mathcal{G}_{\lambda}^{\infty}(\theta) \|.
\end{align}
We bound each term in the RHS: firstly, note the middle term is bounded by Assumption \ref{assumption:3-old}.The first term is bounded as follows:
\begin{align*}
    \|\mathcal{G}_{\lambda}^{n}(\theta) -\pi_{\gamma}\mathcal{G}_{\lambda}^{n}(\theta) \| &\leq \|\mathcal{G}_{\lambda}^{n}(\theta) -\mathcal{G}_{\lambda}^{n}(\theta^{\star}) \| + \|\pi_{\gamma}\mathcal{G}_{\lambda}^{n}(\theta^{\star}) -\pi_{\gamma}\mathcal{G}_{\lambda}^{n}(\theta) \|\\
    & \leq B\|\theta - \theta^{\star}\| + B\|\theta - \theta^{\star}\| \\
    & \leq 2B d_U,
\end{align*}
where in the first step we used that $\mathcal{G}_{\lambda}^{\infty}(\theta^{\star}) = \pi_{\gamma}\mathcal{G}_{\lambda}^{\infty}(\theta^{\star})$.
Similarly, the last term is bounded as follows:
\begin{align*}
    \|\mathcal{G}_{\lambda}^{\infty}(\theta) -\pi_{\gamma}\mathcal{G}_{\lambda}^{\infty}(\theta) \| &\leq \|\mathcal{G}_{\lambda}^{\infty}(\theta) -\mathcal{G}_{\lambda}^{\infty}(\theta^{\star}) \| + \|\pi_{\gamma}\mathcal{G}_{\lambda}^{\infty}(\theta^{\star}) -\pi_{\gamma}\mathcal{G}_{\lambda}^{\infty}(\theta) \|\\
    & \leq \rho(\lambda)\|\theta - \theta^{\star}\| + \frac{\rho(\lambda)}{1+\gamma}\|\theta - \theta^{\star}\| \\
    &= \rho(\lambda)\frac{2+\gamma}{1+\gamma}\|\theta-\theta^\star\| \\
    & \leq \rho(\lambda)\frac{2+\gamma}{1+\gamma}d_U,
\end{align*}
where in the second step we applied \eqref{eq:contraction_mapping}. Using these bounds in \eqref{eq:tri} and taking $c = 2B + \rho(\lambda)\frac{2+\gamma}{1+\gamma}$ completes the proof.
\end{proof}

Note that the constant $c \cdot d_U < c$ (for $U$ sufficiently small) can really be viewed as a part of the optimization constant $\varepsilon_{\text{OPT}}$ since it is controlled by the choice of generative model class.

\section{Point-wise correction corresponds to distribution-wise correction}\label{appendix:self_correction_pointwise}

In this section we provide a sufficient condition under which you can associate a distribution-wise correction mapping (like the one we consider in the paper, $\pi_\gamma$) to a point-wise correction mapping (which is the one you are more likely to find in the wild).

\begin{definition}
    Let $X = \{x_1,\dots,x_n\} \subset \mathbb R^m$ and define the \textit{empirical cumulative distribution} function $\Phi_X$ by 
\begin{align*}
    \Phi_X(v) 
    := \Phi_X(v;\{x_1,\dots,x_n\}) 
    := \frac{1}{n}\sum_{i=1}^{n} \chi_{v}(x_i),
\end{align*}
where for $v\in \mathbb R^m$, $\chi_{v}: \mathbb R^m \to \{0,1\}$ is the indicator function for the set $\prod_{i=1}^{n}(-\infty, v_i]$.
For a continuous distribution, the cumulative distribution function is defined in the usual way.
\end{definition}

\begin{definition}\label{def:pointwise_projection_function_existence_condition}
Suppose that we have a model $p_\theta$ and an arbitrary function $\Pi:\mathbb R^m\to \mathbb R^m$.
Then we say that $\Pi$ is a \emph{valid point-wise correction function} for $p_\theta$ if there exists a $\gamma\in[0,\infty]$ such that
\begin{equation}\label{eq:pointwise_projection_function_existence_condition}
    \lim_{n\to\infty}\left(
    \mathbb E_{X^n \sim p_\theta^{n}} 
    \sup_{v\in \mathbb{R}^m}
    \|\Phi_{\Pi(X^n)}(v) - \Phi_{\pi_{\gamma}p_\theta}(v) \|\right) \to 0,
\end{equation}
almost surely, where the expectation is over all samplings $X^n = \{x_1,\dots,x_n\}$ of size $n$ from $p_\theta$.
\end{definition}

\begin{intuition}
    This is saying that the CDFs for $\pi_\gamma p_\theta$ and $\Pi(X\sim p_\theta^n)$ are equal in expectation, for large enough $n$.
    This is one way of saying that $\pi_\gamma p_\theta$ and $\Pi(X\sim p_\theta^n)$, for large enough $n$, are nearly identical probability distributions.
\end{intuition}

\begin{definition}
    If the limit in \eqref{eq:pointwise_projection_function_existence_condition} exists, then we define the \emph{distribution-wise projection function} corresponding to $\Pi$ to be 
\begin{equation}
    \pi_\gamma p_\theta=\frac{1}{1+\gamma}p_\theta+\frac{\gamma}{1+\gamma}p_{\theta^\star},
\end{equation}
and we define the \emph{projection strength of the point-wise correction function} $\Pi$ to be $\gamma$.
Recall that $\pi_\gamma p_\theta=\frac{1}{1+\gamma}p_\theta+\frac{\gamma}{1+\gamma}p_{\theta^\star}$. 
So intuitively, \eqref{eq:pointwise_projection_function_existence_condition} implies that the projection function $\Pi$ maps samples from $p_\theta$ to a different space such that they look like they come from a combination of the original distribution $p_\theta$ and $p_{\theta^\star}$, at least at the level of CDFs.
\end{definition}

\begin{remark}
Such a $\gamma$, if it exists, is unique.
Furthermore, if $p_\theta=p_{\theta^\star}$, then $\gamma=\infty$. 
\end{remark}

The limit condition in Definition~\ref{def:pointwise_projection_function_existence_condition} is abstract, and can be hard to swallow.
We present an example of a simple point-wise correction for the Gaussian toy example that we consider in Section~\ref{sec:experiments_gaussian}, whose corresponding distribution-wise correction is exactly one would expect it to be--the weighted average of the corresponding Gaussians.
Recall that we demonstrated empirically in Figure~\ref{fig:toy_example_w2} that Theorem~\ref{thm:2_shortened} holds for that example.
The projection function is depicted in Figure~\ref{fig:projection_gaussian}.

\begin{example}
 Let $G_1(x)$ be the pdf of  $\mathcal{N}(0, \sigma_1^2 I_d)$ (initial distribution, corresponds to $\theta$) and $G_2(x)$ the pdf of $\mathcal{N}(0, \sigma_2^2 I_d)$ (target distribution, corresponds to $\theta^\star$). 
Given $x_1,\dots,x_n \sim G_1$, we define $\Pi^\gamma$ as follows: 
Fix any $\gamma\in\mathbb{R}_{\ge 0}$, and let $y_1,\dots,y_n \sim (\hat G_1^{(n)}(x)+\gamma G_2(x))/(1+\gamma)$, where $\hat G_1^{(n)}$ is the PDF of the empirical distribution defined by $\{x_1,\dots,x_n\}$; in practice we implement $\hat G_1^{(n)}$ as a histogram.
Then choose a random $\sigma \in S_n$ ($S_n$ = group of permutations on $n$ symbols). 
Finally, we define $\Pi^\gamma(x_i):=y_{\sigma(i)}$ for $1\le i\le n$.

Next, we define the projection set $\Pi X^{(n)} := \{\Pi^{\gamma}(x_i)\}_{1\leq i \leq n}$, and define the PDF $\pi_\gamma \hat G_1^{(n)}(x):=\frac{1}{1+\gamma}\hat G_1^{(n)}(x)+\frac{\gamma}{1+\gamma}G_2(x)$, and let $\Phi_{\pi_{\gamma} \hat G_1^{(n)}}$ represent the cumulative distribution function of the Gaussian $\pi_{\gamma} \hat G_1^{(n)}$. Then, since $\Pi^{\gamma}(x_i) \sim \pi_{\gamma }\hat G_1^{(n)}$, we have by the uniform law of large numbers that
\begin{equation}
    \lim_{n\to\infty}
    \left(\mathbb E_{\{x_i\sim G_1\}_{i=1}^n}
    \mathrm{sup}_{v\in \mathbb R^m}
    \left\|
    \Phi_{\Pi X^{(n)}}(v)
    -\Phi_{\pi_{\gamma} G_1}(v)
    \right\|\right)\to 0
\end{equation}
 almost surely. Therefore $\Pi^\gamma$ is a valid point-wise correction function, and its corresponding distribution-wise projection function is $\pi_\gamma$.
\end{example}

\begin{remark}
    In the example we considered in Section~\ref{sec:experiments_gaussian}, we could have included a total distance traveled minimization condition, but here for this proof we don't even need to use that hypothesis. 
    (In the proof, this would have corresponded to the additional assumption that we've chosen a $\sigma \in S_n$ such that $\sum_{i=1}^{n} \|x_i - y_{\sigma(i)}\|$ is minimized.)
    This implies that different point-wise correction functions can correspond to the same distribution-wise correction function.
\end{remark}

\begin{figure}[ht]
\vskip 0.2in
\begin{center}
\centerline{\includegraphics[width=0.8\columnwidth]{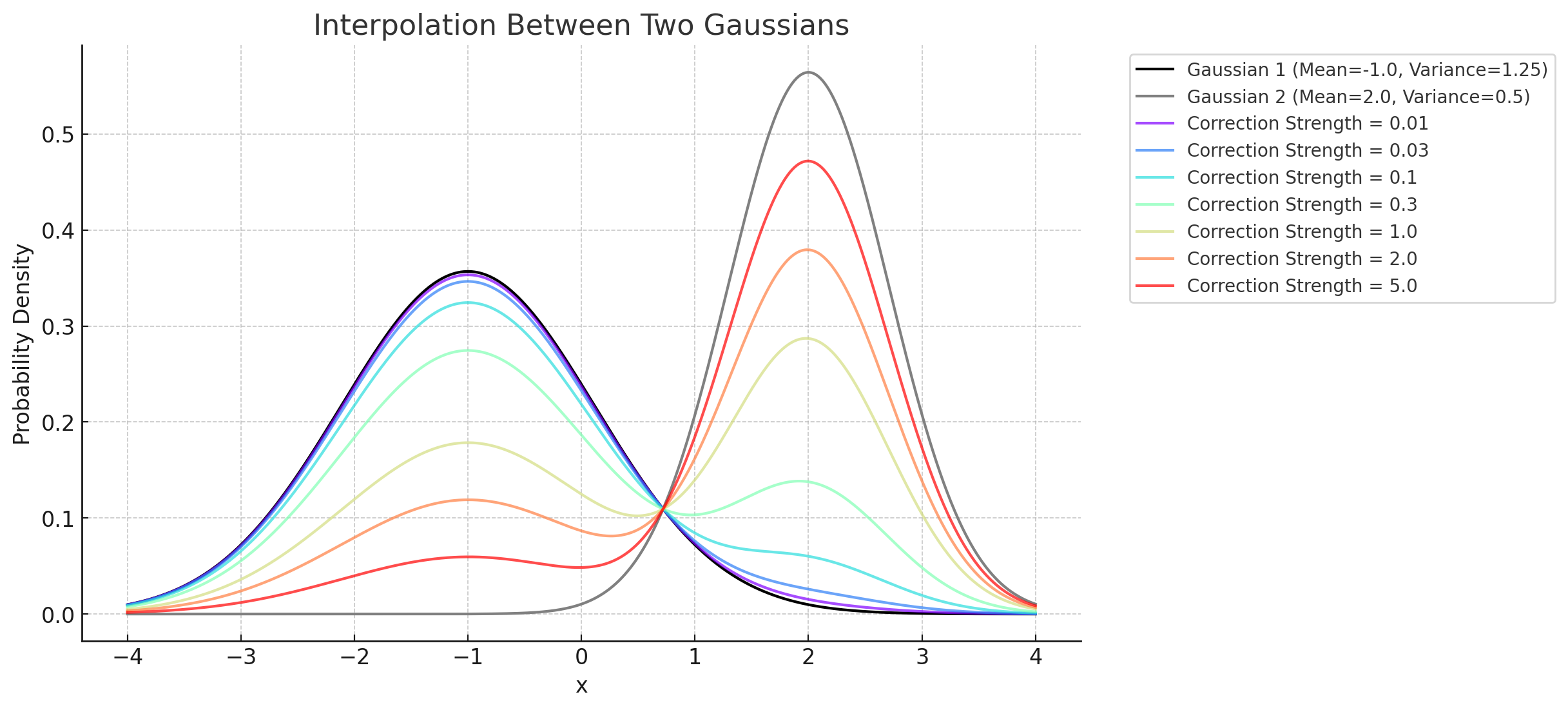}}
\caption{Illustration of the distribution-wise projection function, like in our Gaussian toy example. 
Correcting one Gaussian in the direction of another, like we consider in Section~\ref{sec:experiments_gaussian}, corresponds to finding the ``(weighted) average Gaussian'' that lives between the two.}
\label{fig:projection_gaussian}
\end{center}
\vskip -0.2in

\end{figure}

\section{More MNIST Experiment Details}\label{appendix:MNIST}

We train a Denoising Diffusion Probabilistic Model (DDPM) \cite{ho2020denoising} on the $20\%$ of the MNIST dataset \cite{lecun1998gradient}.
We use classifier-free guidance \cite{ho2021classifierfree} with guidance parameter $0.5$, and 400 diffusion steps.
We used a batch size of 256.
We train generation $0$ for 20 epochs, with a linear decay learning rate schedule starting at $1e-4$ and ending at $(1e-4)/20.$
We train each following generations for a single epoch, with a fixed learning rate of $(1e-4)/20^2$.

To compute our metrics, we first train a LeNet model \cite{lecun1998gradient} on MNIST, and then we sample an equal number of digits from each class using the checkpoint that we're trying to evaluate.
To compute the FID score, we extract embeddings from the last fully connected LeNet layer for the synthesized examples, as well as for the held out test examples, and compute FID score as normal, by computing the Wasserstein distance between the Gaussians.
Note that we use embeddings for LeNet trained on MNIST, rather than the Inception network trained on ImageNet, because MNIST isn't comprised of natural images. 
This is consistent with the convention in \cite{alemohammad2023self}.

For the self-correction operation, we compute the $K$-means clusters, with $K=16$, once at the start of training.
And we ``correct'' a synthesized motion by mapping it to the nearest cluster mean corresponding to its digit.
In Figure~\ref{fig:mnist_clusters} we present the clusters, and we present graphs of our FID scores across augmentation percentages in Figure~\ref{fig:mnist_graphs}.

\begin{figure*}
\vskip 0.2in
\begin{center}
\centerline{\includegraphics[width=0.5\columnwidth]{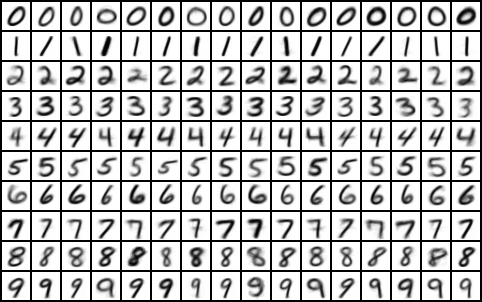}}
\caption{For every digit, we perform $K$-means clustering with $K$=16. We show here the cluster centroids, which intuitively are anchor images within the manifold of all possible images.}
\label{fig:mnist_clusters}
\end{center}
\vskip -0.2in
\end{figure*}

\begin{figure*}
\vskip 0.2in
\begin{center}
\begin{tabular}{cc}
\includegraphics[width=0.4\columnwidth]{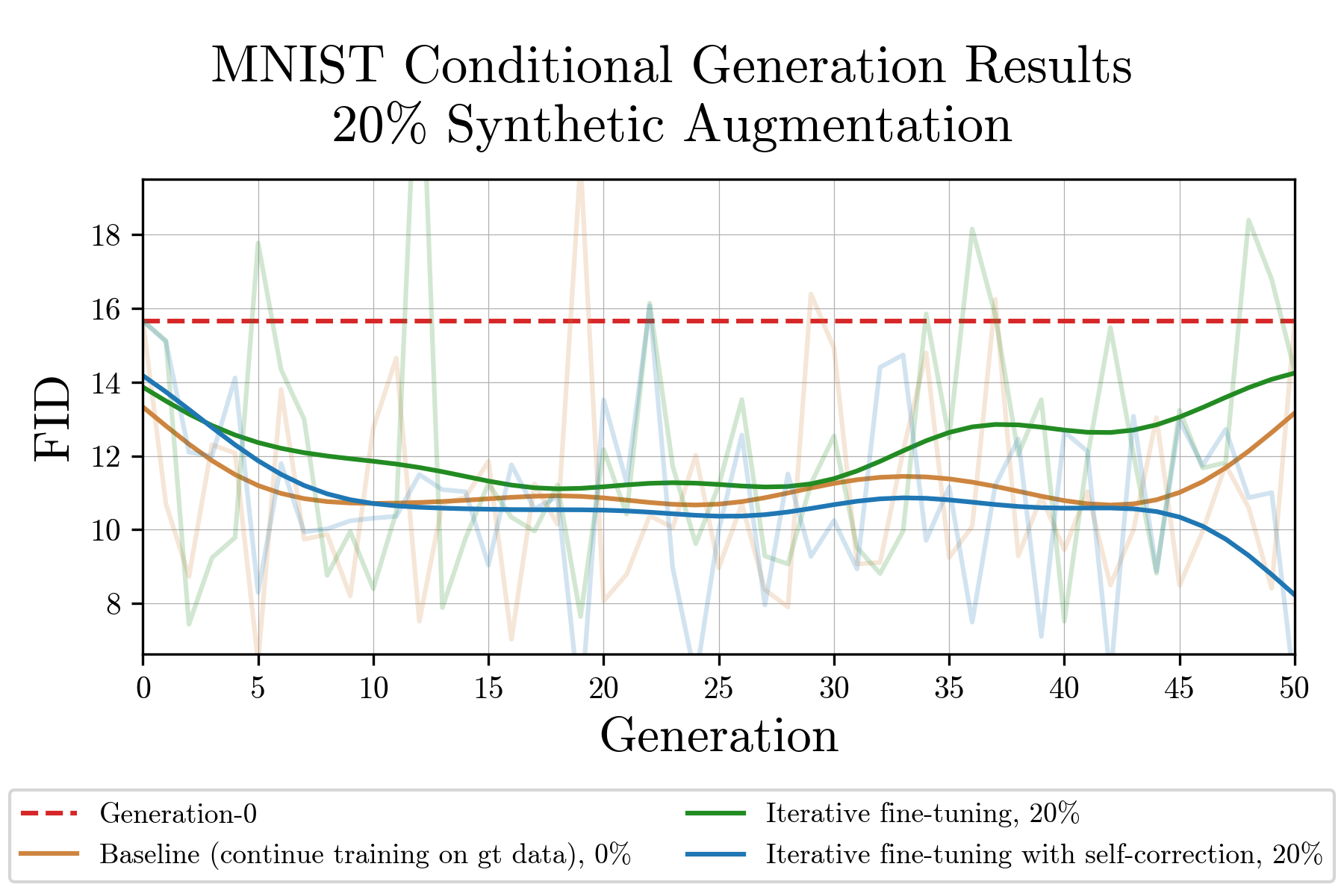} &
\includegraphics[width=0.4\columnwidth]{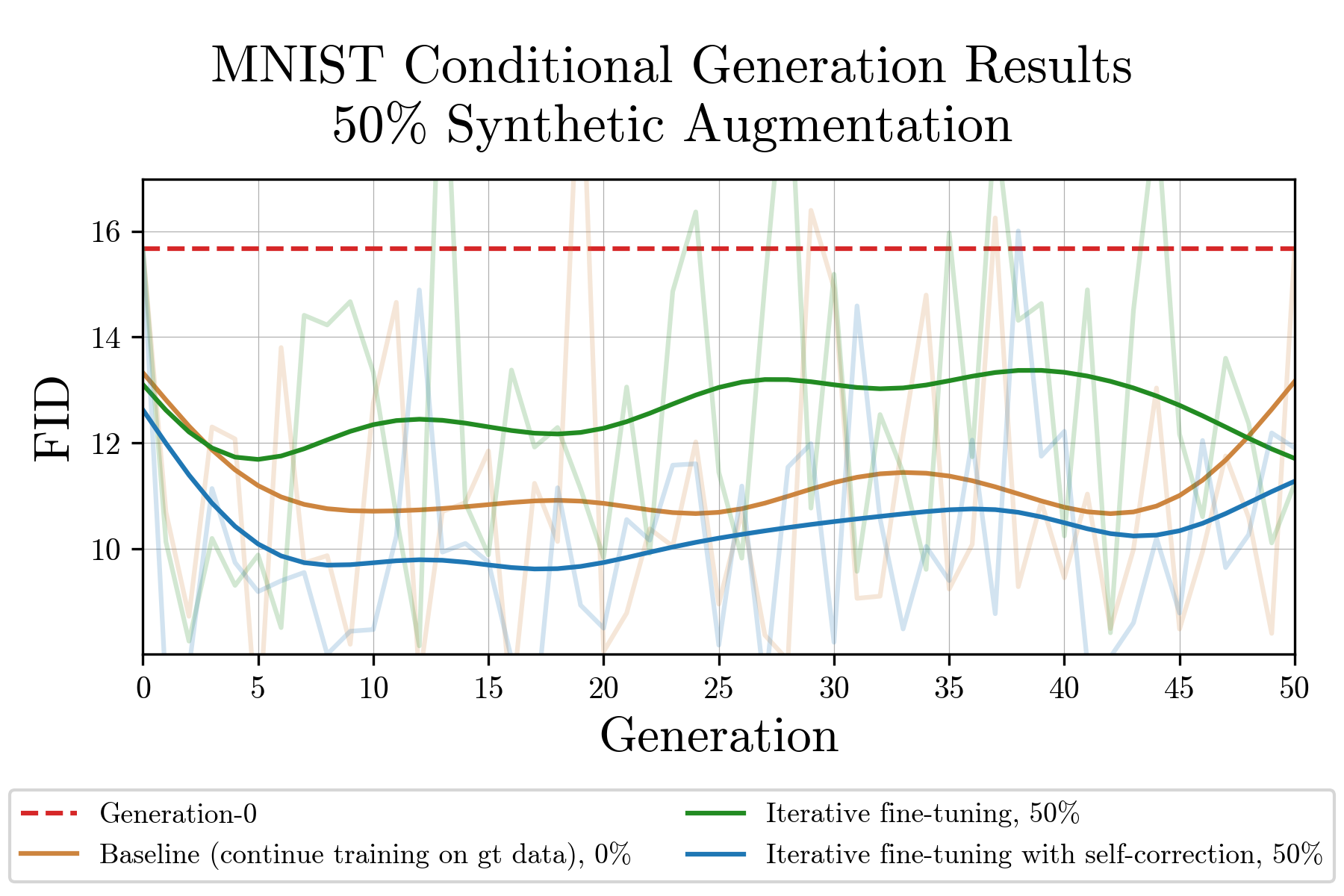} \\
\includegraphics[width=0.4\columnwidth]{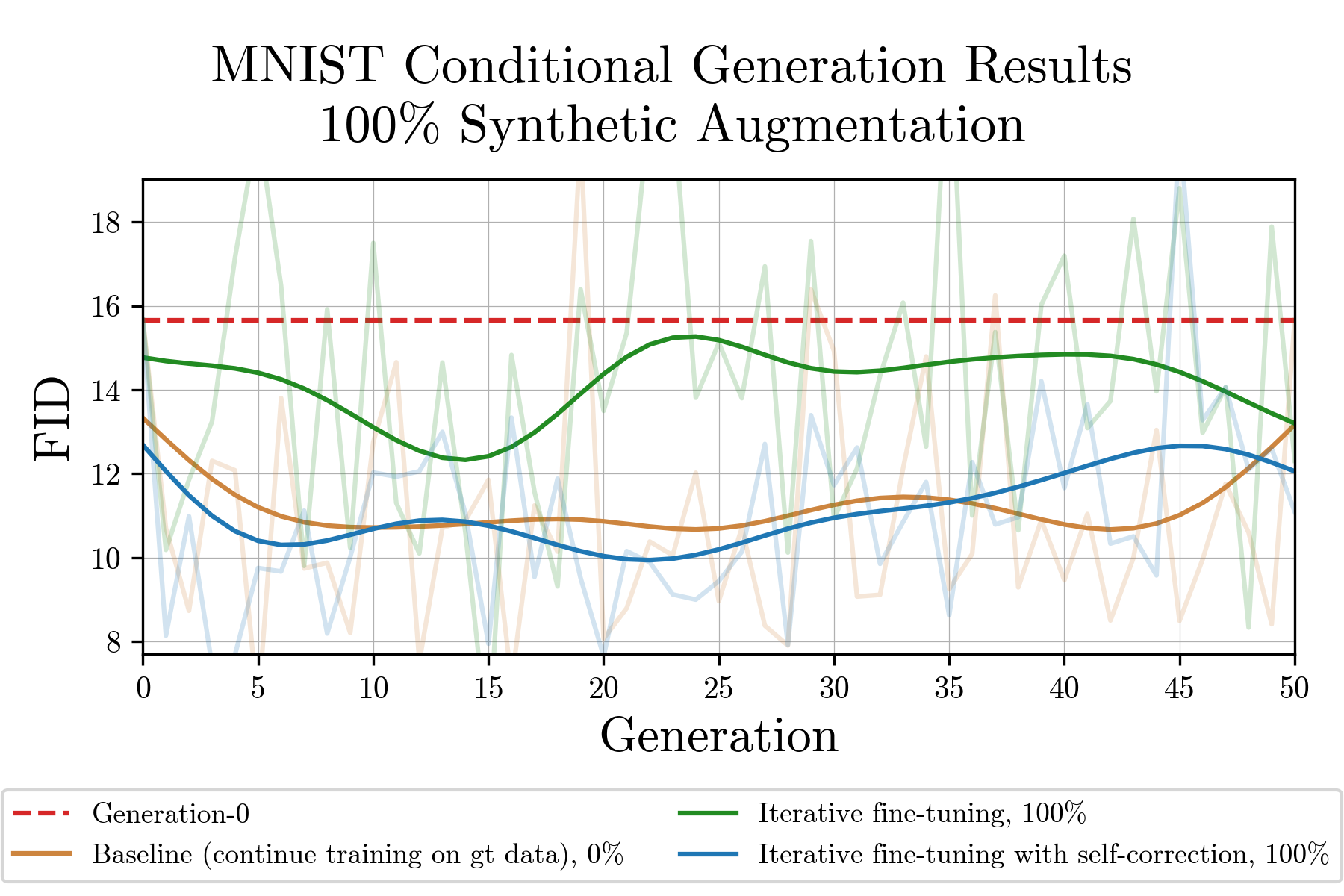} & 
\includegraphics[width=0.4\columnwidth]{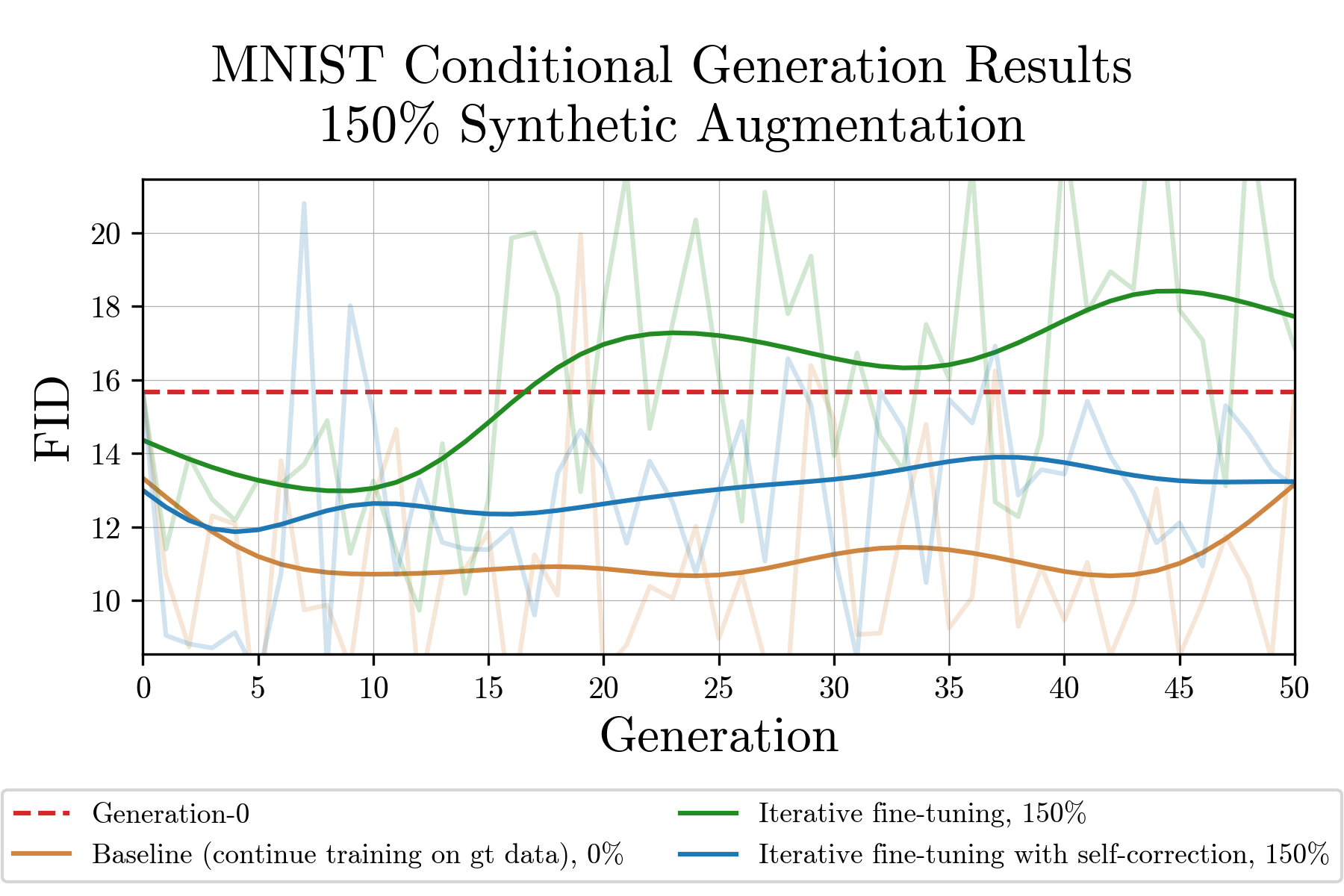} \\
\end{tabular}
\caption{Results from MNIST experiments with iterative fine-tuning with and without self-correction.
These graphs show the FID score on the last checkpoint for every generation; this is the checkpoint used for sampling in the self-consuming loop experiments, and it is also the checkpoint where training is resumed with this new partially synthesized dataset. 
These results demonstrate that \textbf{\color{blue_better}iterative fine-tuning with self-correction} generally outperforms \textbf{\color{green_better}iterative fine-tuning.}}
\label{fig:mnist_graphs}
\end{center}
\vskip -0.2in
\end{figure*}

\section{Additional Human Motion Generation Qualitative Results}\label{appendix:human_motion_qualitative}

In Figures~\ref{fig:qualitative_human_motion_picture_1}, ~\ref{fig:qualitative_human_motion_picture_4}, and ~\ref{fig:qualitative_human_motion_picture_2}, we present additional qualitative observations and analysis of our synthesized motions. 
We present more evidence that iterative fine-tuning with self-correction yields physically plausible motions comparable to the baseline, whereas iterative fine-tuning without self-correction yields motions that are incorrect for various reasons.
See the captions of the referenced figures for analysis of some characteristic failure modes of the iterative fine-tuning loop \textit{without} self-correction.

A technical note: for all figures, we render the motions from the same environment and camera position. 
We consolidate each render into the same image \textit{without} resizing it. 
This means that if a figure appears larger relative to the others, the human moved closer to the camera. 
Some motions will have transparent frames of past positions; the more transparent the image, the farther back in the past it was in the motion sequence. 
Finally, in each figure, the text prompt for all generated motions was the same --the prompt being the one associated with the ground truth motion in the HumanML3D \cite{Guo_2022_CVPR} training data, which we also visualize.
Note that the coloring in the humanoid figures corresponds to the coloring in the graphs.

\begin{figure*}[ht]
\begin{center}
\centerline{\includegraphics[width=0.8\columnwidth]{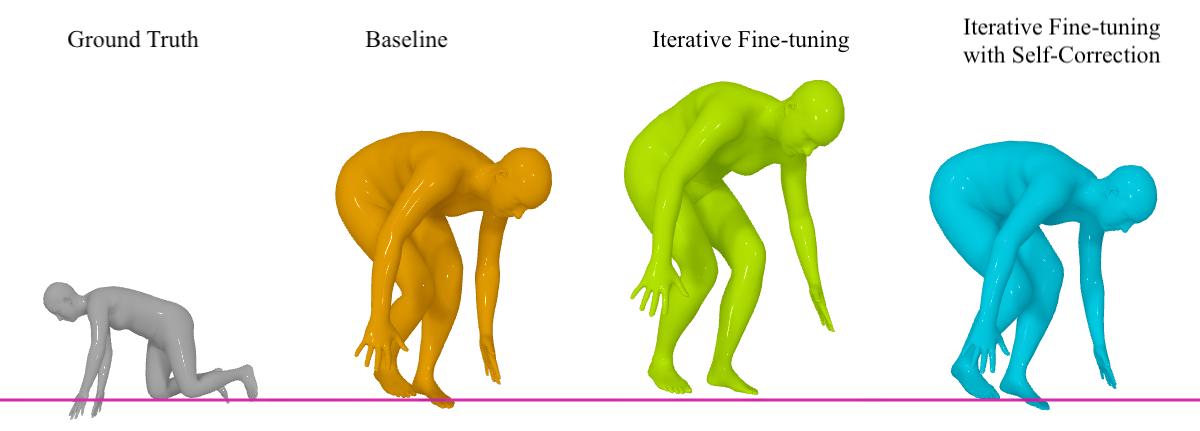}}
\caption{Here we see the negative \emph{floating} phenomenon exacerbated by \textbf{\color{green_better}iterative fine-tuning}, whereas \textbf{\color{blue_better}iterative fine-tuning with self-correction} generates a motion with floor contact integrity comparable to the \textbf{\color{gray}ground truth} and  \textbf{\color{orange}baseline}. 
The floatic metric is formally defined in \cite{yuan2023physdiff} as the distance between the lowest vertex on the human mesh and the floor plane. 
All three sequences were generated using the same prompt: \textit{person got down and is crawling across the floor.} 
Each snapshot was taken at exactly frame 87. The green figure appears larger than the other two only because it is closer to the camera.
The two motions on the right were synthesized after 50 generations training with $25\%$ synthetic augmentation, trained on $n=64$ data points.}
\label{fig:qualitative_human_motion_picture_1}
\end{center}
\end{figure*}

\begin{figure*}[ht]
\begin{center}
\centerline{\includegraphics[width=0.8\columnwidth]{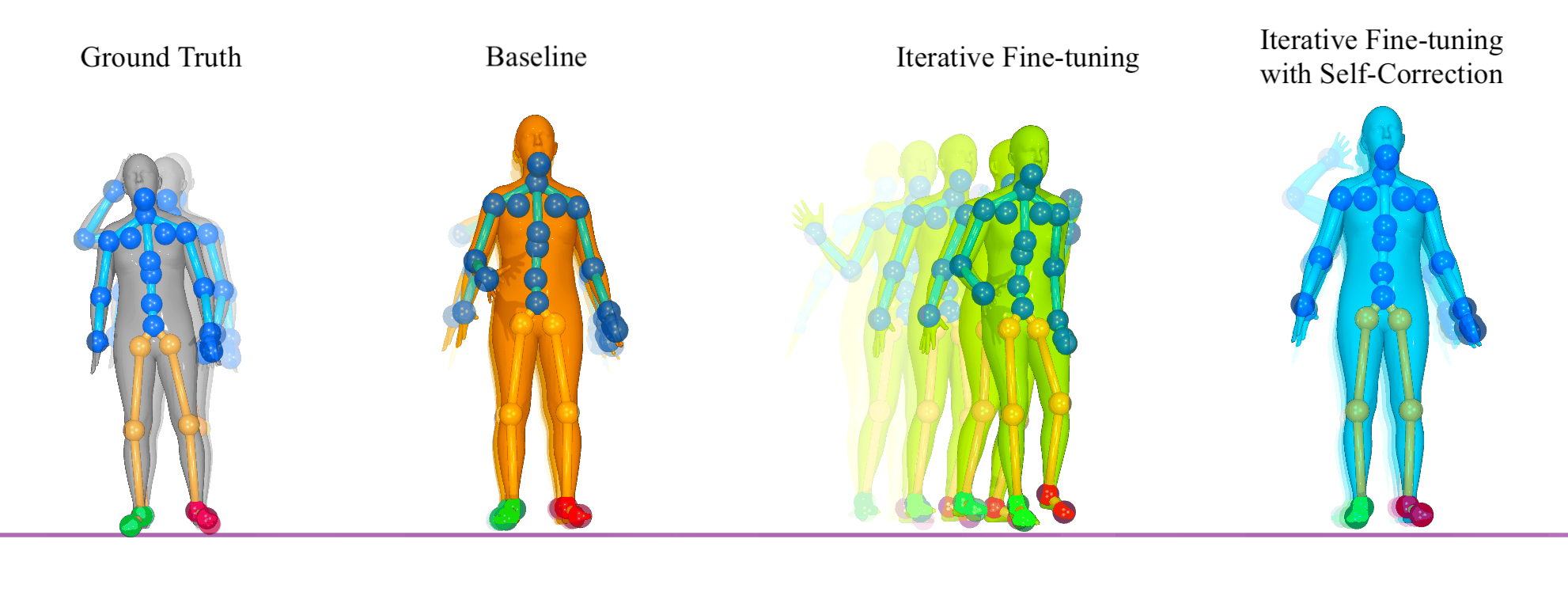}}
\caption{All four of the above motions correspond to the prompt: \textit{a person raises right hand to face looks around and puts hand down back to side.}. 
The model which is trained with \textbf{\color{green_better}iterative fine-tuning} outputs spurious motion that slides the figure to the right. 
And in the video for this example, the human rotates their forearm unnaturally and forcefully.
In contrast, the \textbf{\color{orange}baseline} and \textbf{\color{blue_better}iterative fine-tuning with self-correction} models' motions both accurately embody the prompt.
Each generated snapshot is taken at exactly frame 142 while the ground truth's image is frame 70 in its sequence. 
The two motions on the right were synthesized after 42 generations with $10\%$ synthetic augmentation, where the ground truth dataset has size $n=2794$.}
\label{fig:qualitative_human_motion_picture_4}
\end{center}
\end{figure*}

\begin{figure*}[ht]
\begin{center}
\centerline{\includegraphics[width=0.8\columnwidth]{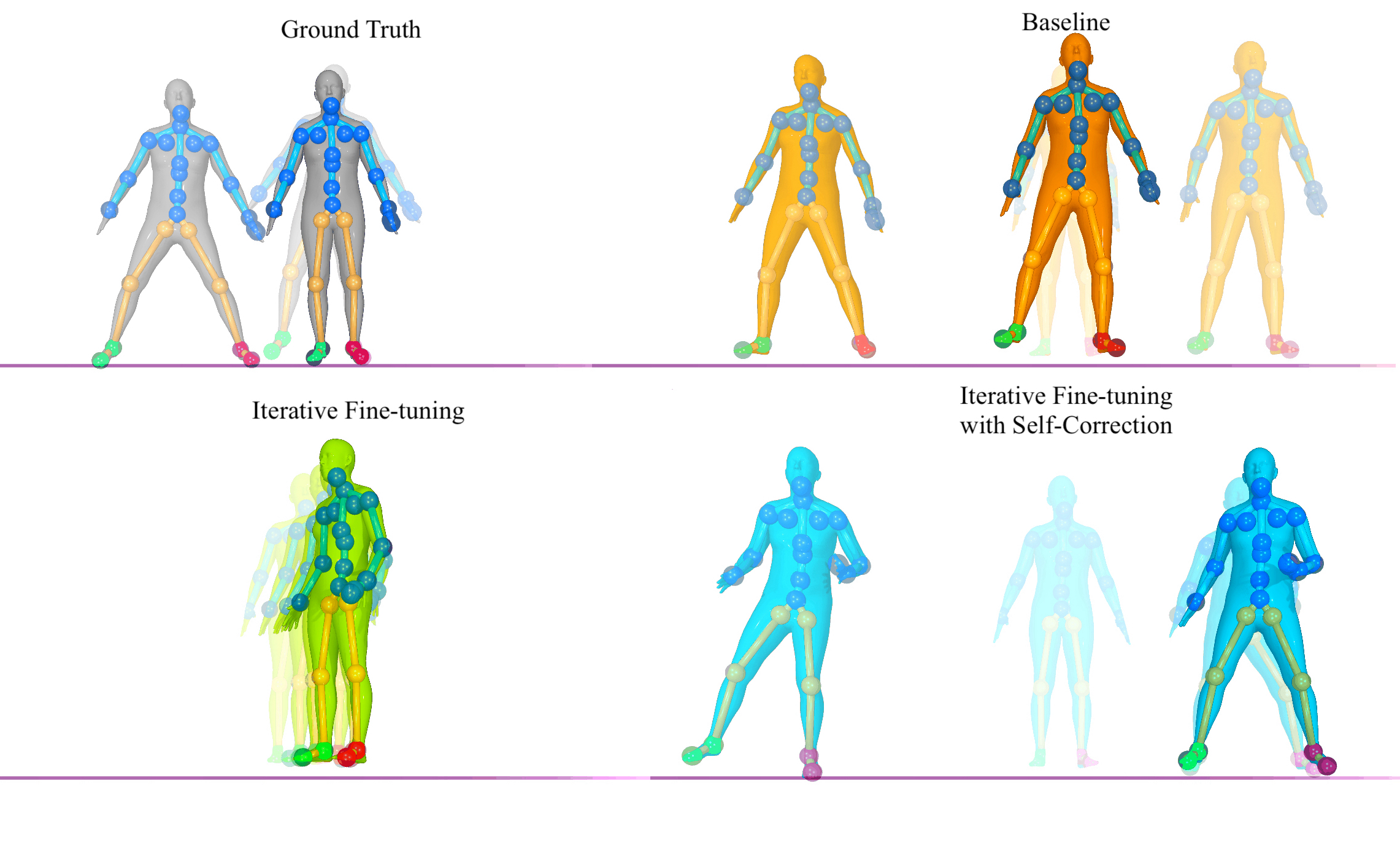}}
\caption{Here we observe that \textbf{\color{green_better}iterative fine-tuning} fails to produce any meaningful motion sequence, but the \textbf{\color{blue_better}iterative fine-tuning with self-correction} and \textbf{\color{orange}baseline} models generate results consistent with their prompt: \textit{walks side ways but back and forth}. 
Each snapshot for the generated motions was taken at exactly frame 120 while the ground truth image is a snapshot from frame 69.
These images were synthesized after 50 generation of the model that was trained on $n=64$ data points at $25\%$ synthetic augmentation.}
\label{fig:qualitative_human_motion_picture_2}
\end{center}
\end{figure*}

\section{Additional Human Motion Generation Quantitative Results}\label{appendix:more_human_motion_graphs}

See Figures~\ref{fig:few_shot_0064_latest}, \ref{fig:few_shot_0128_latest}, \ref{fig:few_shot_0256_latest} for results when the dataset size is $n\in\{64, 128, 256\}$ and the synthetic augmentation percentage is $\lambda\in\{0.25, 0.50, 0.75, 1.00\}$.
And see Figures~\ref{fig:large_scale_latest} and ~\ref{fig:large_scale_average} for additional results on our iterative fine-tuning experiments when the dataset size is $n=2794$ and the synthetic augmentation percentage is $\lambda\in\{0.05, 0.10, 0.15, 0.20, 0.25\}$.
The graphs provide evidence across $17$ experiment settings that our iterative fine-tuning procedure with self-correction yields better training performance than iterative fine-tuning with no self-correction for the motion synthesis task, in accordance with Theorem~\ref{thm:2_shortened}.

\begin{figure*}
\vskip 0.2in
\begin{center}
\centerline{\includegraphics[width=0.9\columnwidth]{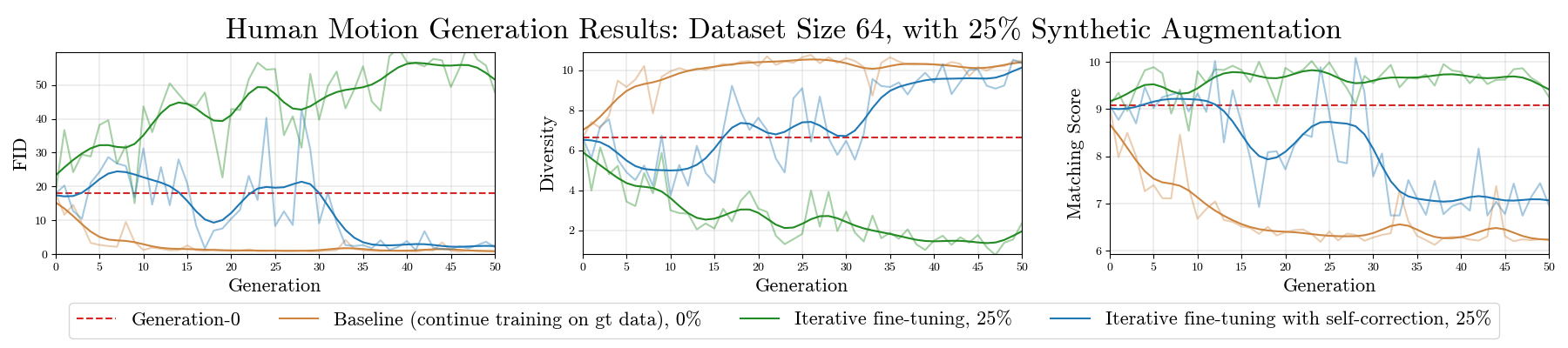}}
\centerline{\includegraphics[width=0.9\columnwidth]{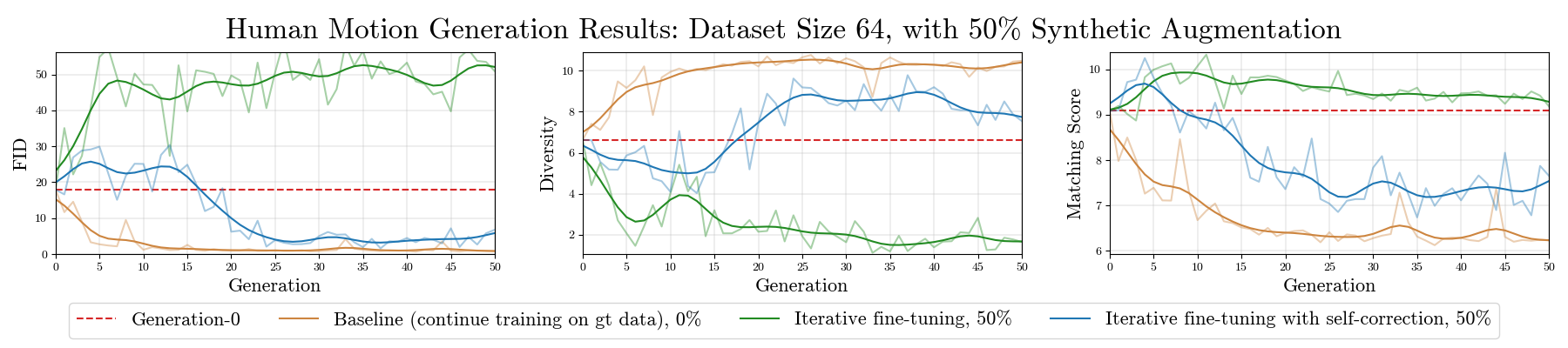}}
\centerline{\includegraphics[width=0.9\columnwidth]{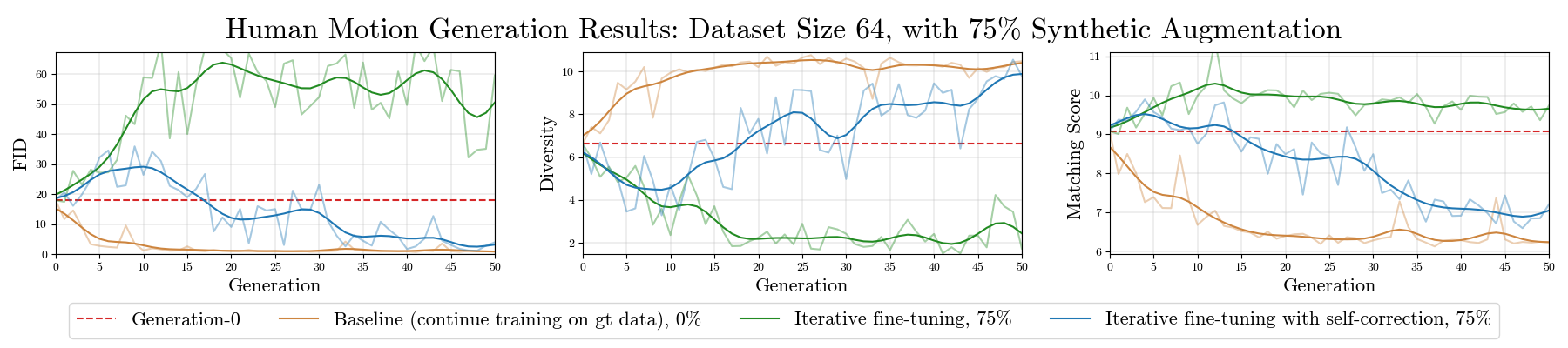}}
\centerline{\includegraphics[width=0.9\columnwidth]{images/graphs_with_smoothed/graphs-iterative-finetuning_0064_shot_1k_results_type_latest/0064_100_1k_iterative_finetuning.png}}
\caption{Results from our human motion experiments with iterative fine-tuning with and without self-correction, where the training set has size $64$.
These are graphs for evaluation metrics on the last checkpoint for every generation; this is the checkpoint used for sampling in the self-consuming loop experiments, and it is also the checkpoint where training is resumed with this new partially synthesized dataset. 
These results demonstrate that \textbf{\color{blue_better}iterative fine-tuning with self-correction} generally outperforms \textbf{\color{green_better}iterative fine-tuning}, and is sometimes even competitive with \textbf{\color{orange}baseline} performance.}
\label{fig:few_shot_0064_latest}
\end{center}
\vskip -0.2in
\end{figure*}

\begin{figure*}
\vskip 0.2in
\begin{center}
\centerline{\includegraphics[width=0.9\columnwidth]{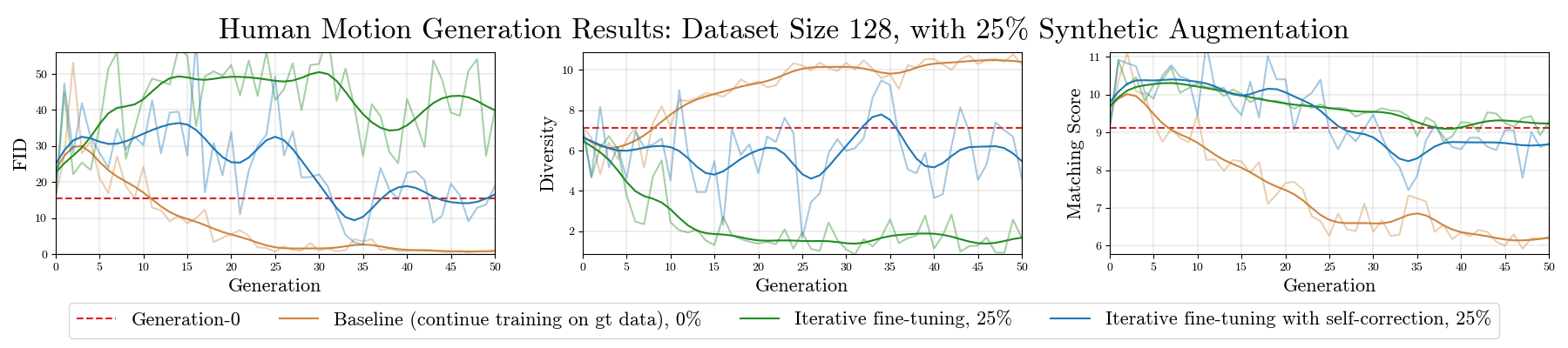}}
\centerline{\includegraphics[width=0.9\columnwidth]{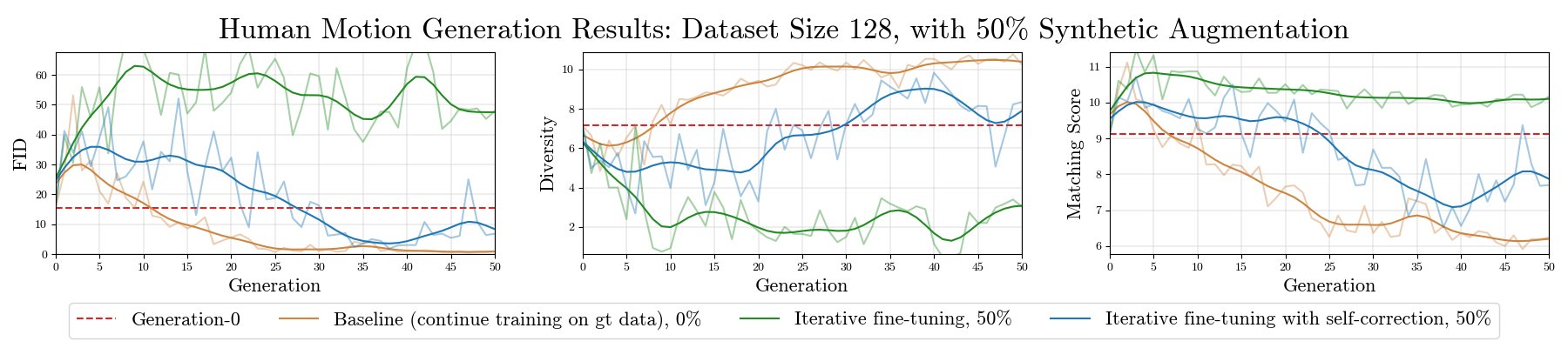}}
\centerline{\includegraphics[width=0.9\columnwidth]{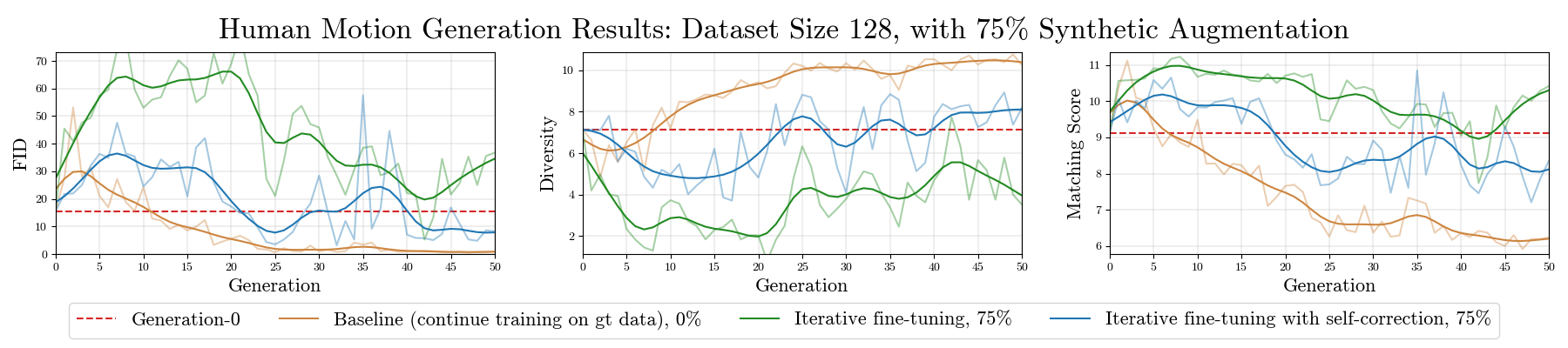}}
\centerline{\includegraphics[width=0.9\columnwidth]{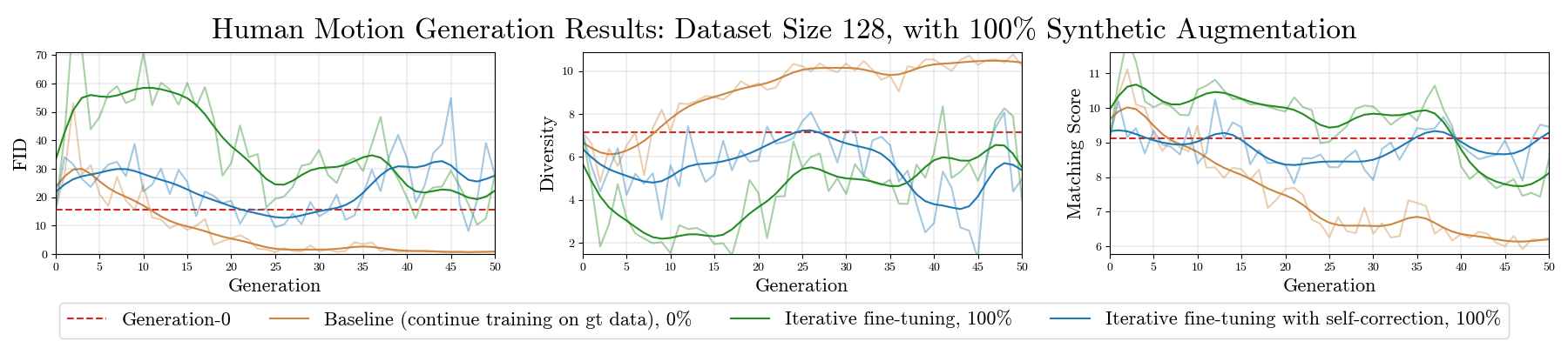}}
\caption{Results from our human motion experiments with iterative fine-tuning with and without self-correction, where the training set has size $128$.
These are graphs for evaluation metrics on the last checkpoint for every generation; this is the checkpoint used for sampling in the self-consuming loop experiments, and it is also the checkpoint where training is resumed with this new partially synthesized dataset. 
These results demonstrate that \textbf{\color{blue_better}iterative fine-tuning with self-correction} generally outperforms \textbf{\color{green_better}iterative fine-tuning}, and is sometimes even competitive with \textbf{\color{orange}baseline} performance.
Notably, the performance gain of \textbf{\color{blue_better}iterative fine-tuning with self-correction} over \textbf{\color{green_better}iterative fine-tuning} is less pronounced than when the dataset size is $n=64$.}
\label{fig:few_shot_0128_latest}
\end{center}
\vskip -0.2in
\end{figure*}

\begin{figure*}
\vskip 0.2in
\begin{center}
\centerline{\includegraphics[width=0.9\columnwidth]{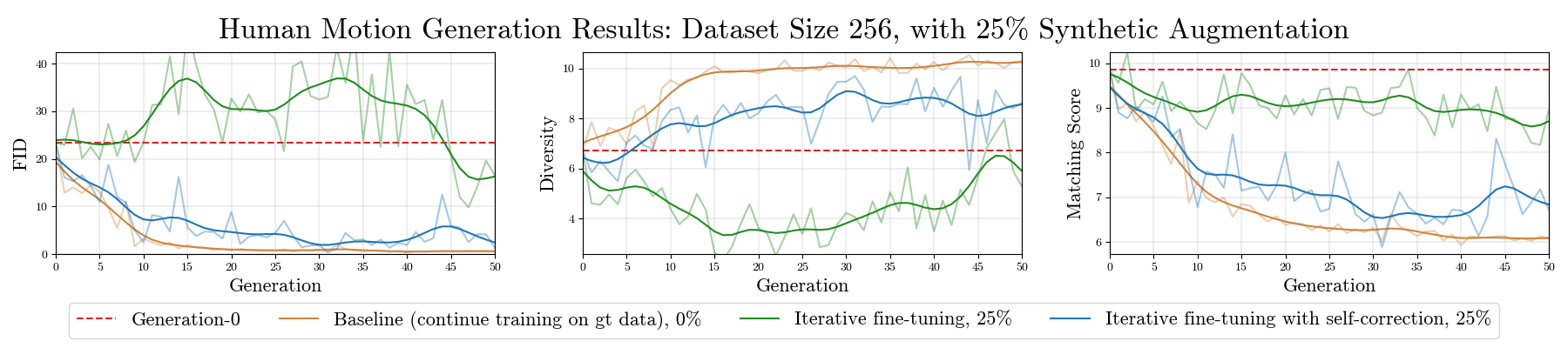}}
\centerline{\includegraphics[width=0.9\columnwidth]{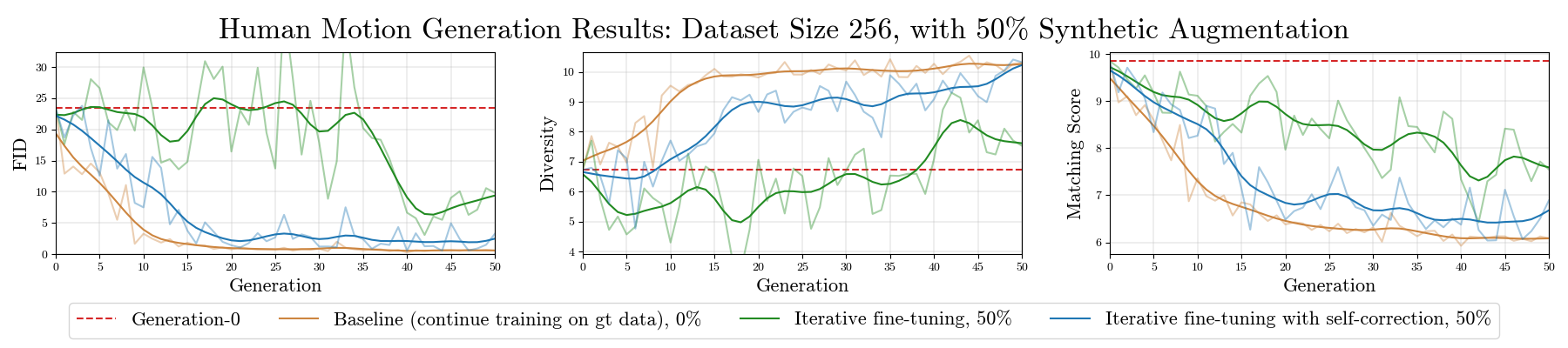}}
\centerline{\includegraphics[width=0.9\columnwidth]{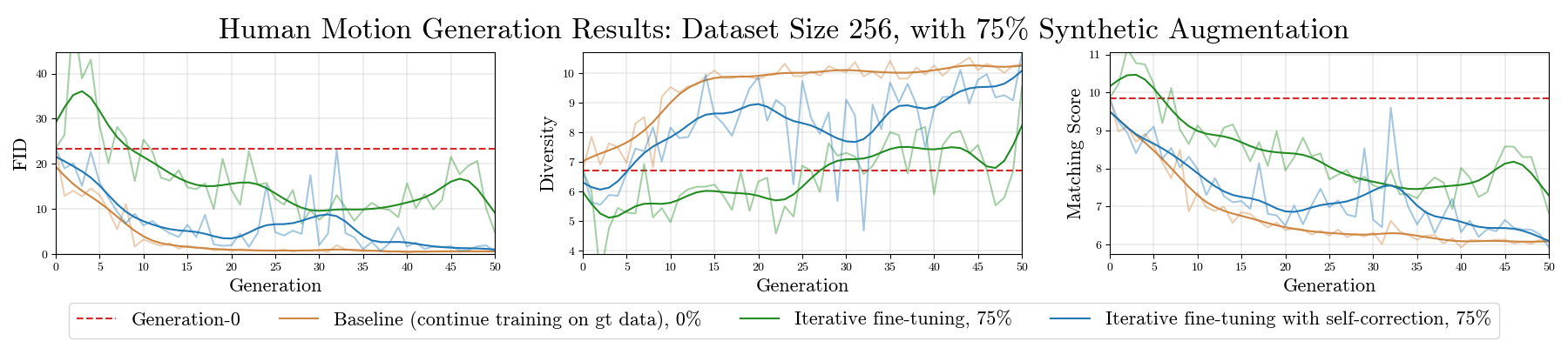}}
\centerline{\includegraphics[width=0.9\columnwidth]{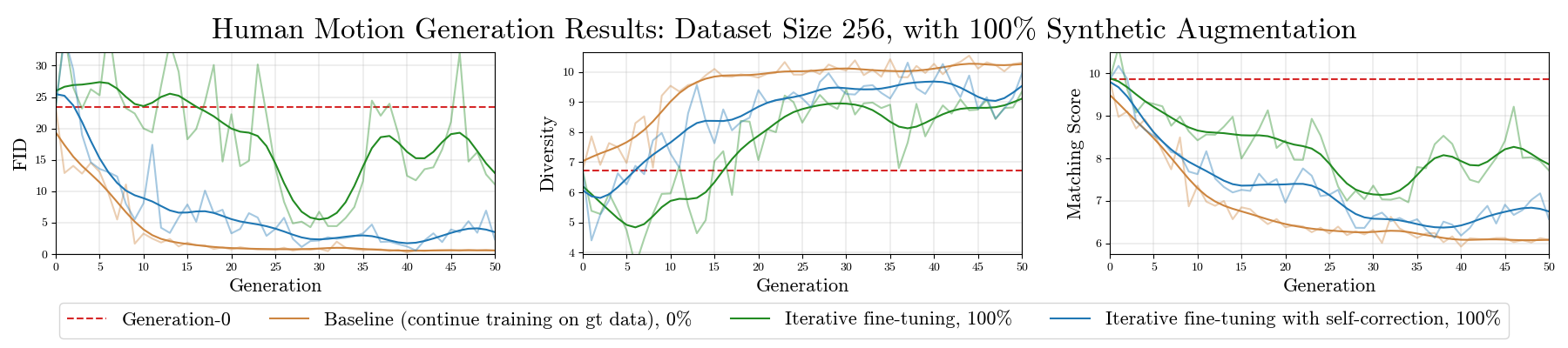}}
\caption{Results from our human motion experiments with iterative fine-tuning with and without self-correction, where the training set has size $256$.
These are graphs for evaluation metrics on the last checkpoint for every generation; this is the checkpoint used for sampling in the self-consuming loop experiments, and it is also the checkpoint where training is resumed with this new partially synthesized dataset. 
These results demonstrate that \textbf{\color{blue_better}iterative fine-tuning with self-correction} generally outperforms \textbf{\color{green_better}iterative fine-tuning}, and is sometimes even competitive with \textbf{\color{orange}baseline} performance.}
\label{fig:few_shot_0256_latest}
\end{center}
\vskip -0.2in
\end{figure*}

\begin{figure*}
\vskip 0.2in
\begin{center}
\centerline{\includegraphics[width=0.9\columnwidth]{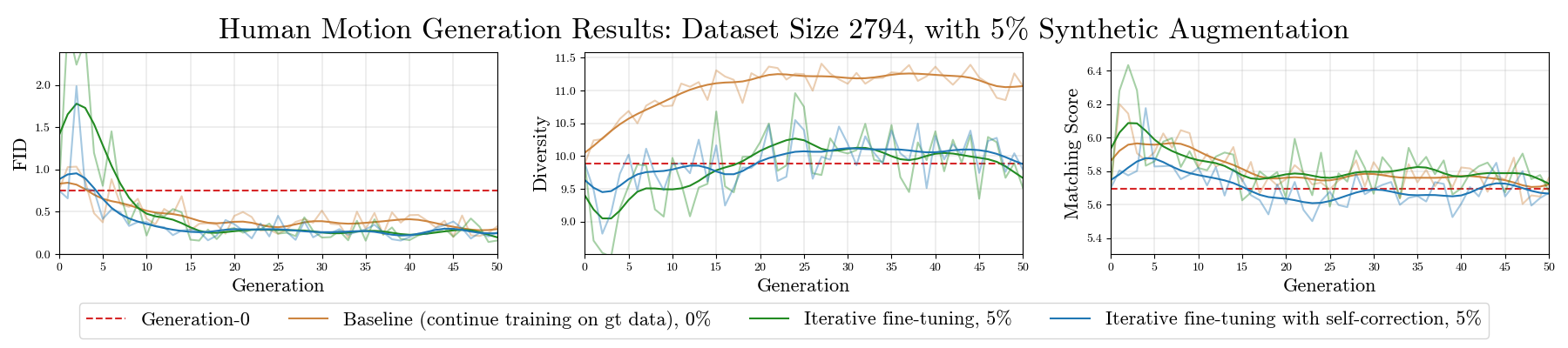}}
\centerline{\includegraphics[width=0.9\columnwidth]{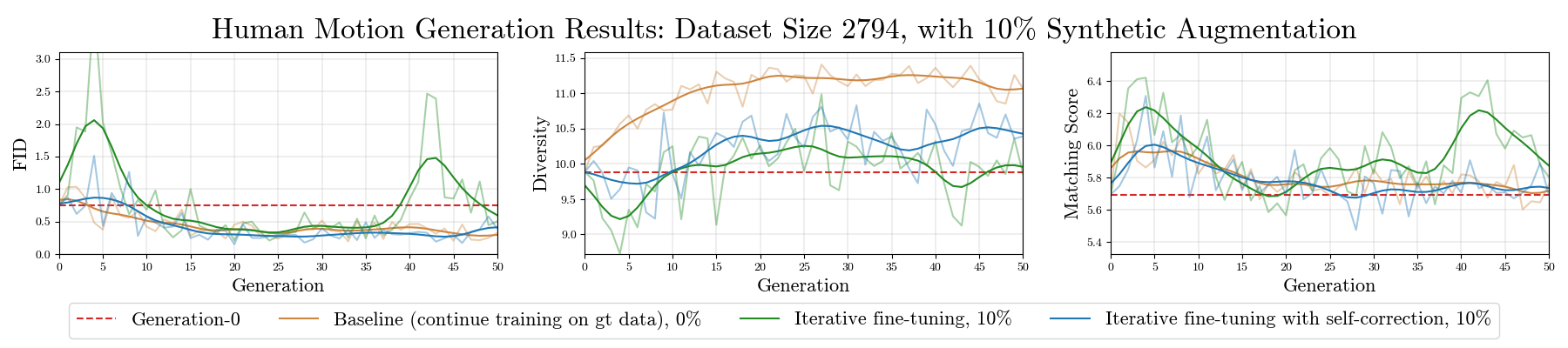}}
\centerline{\includegraphics[width=0.9\columnwidth]{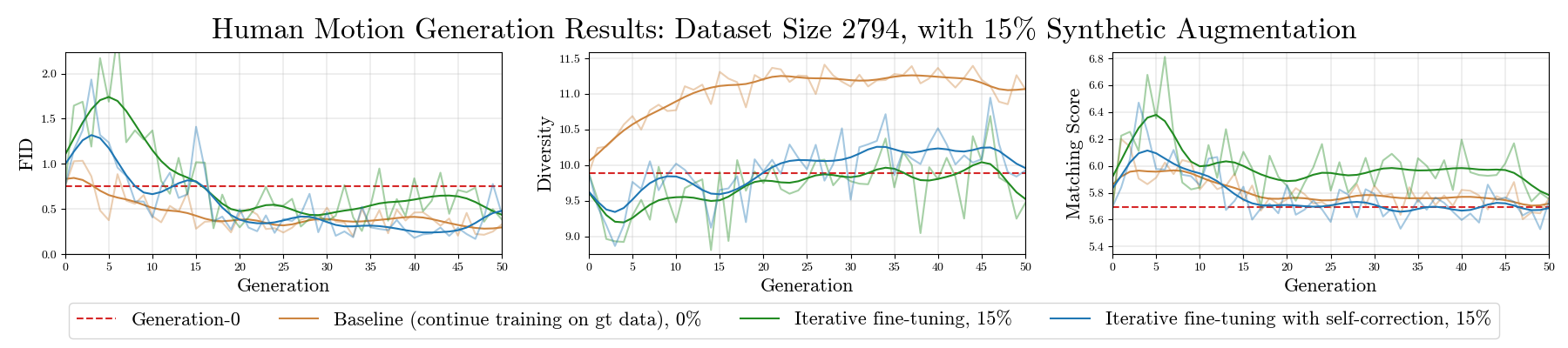}}
\centerline{\includegraphics[width=0.9\columnwidth]{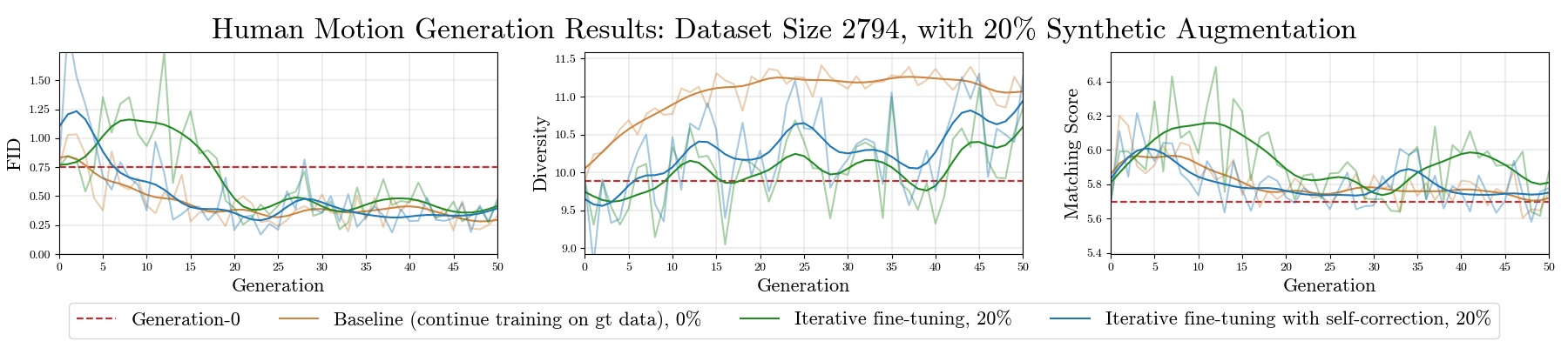}}
\centerline{\includegraphics[width=0.9\columnwidth]{images/graphs_with_smoothed/graphs-iterative-finetuning_40k_results_type_latest/40k_25_iterative_finetuning.png}}
\caption{Results from our human motion experiments on iterative fine-tuning with dataset size $n=2794$.
These are graphs for evaluation metrics on the last checkpoint for every generation; this is the checkpoint used for sampling in the augmentation loop experiments, and it is also the checkpoint where training is resumed with this new synthesized dataset. In these results, it appears as though iterative fine-tuning \textbf{\color{blue_better}with self-correction} has less variance during training than iterative fine-tuning with \textbf{\color{green_better}with no self-correction}, and generally has better FID scores later in training. Notably, the these two curves are closer together than they were in the cases $n\in\{64,128,256\}$.}
\label{fig:large_scale_latest}
\end{center}
\vskip -0.2in
\end{figure*}

\begin{figure*}
\vskip 0.2in
\begin{center}
\centerline{\includegraphics[width=0.9\columnwidth]{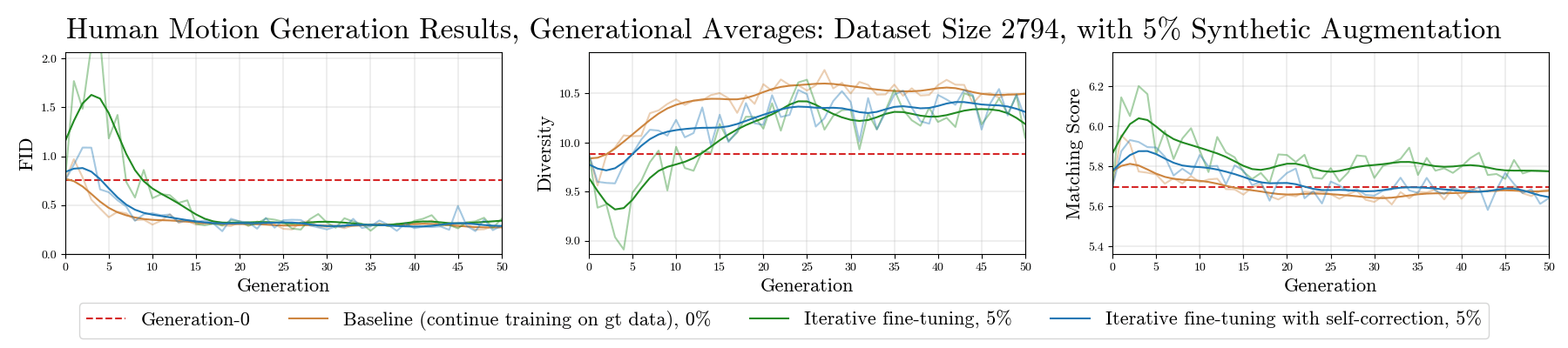}}
\centerline{\includegraphics[width=0.9\columnwidth]{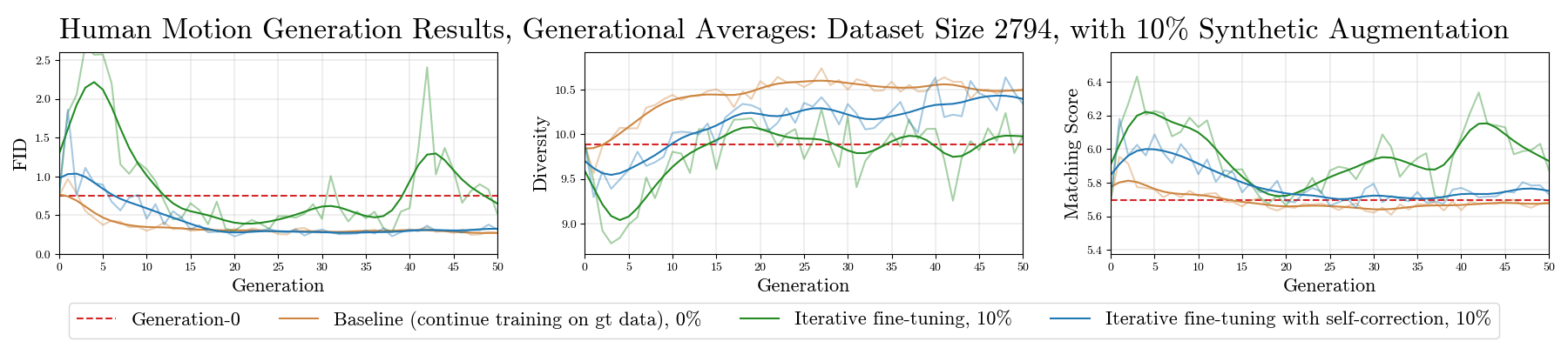}}
\centerline{\includegraphics[width=0.9\columnwidth]{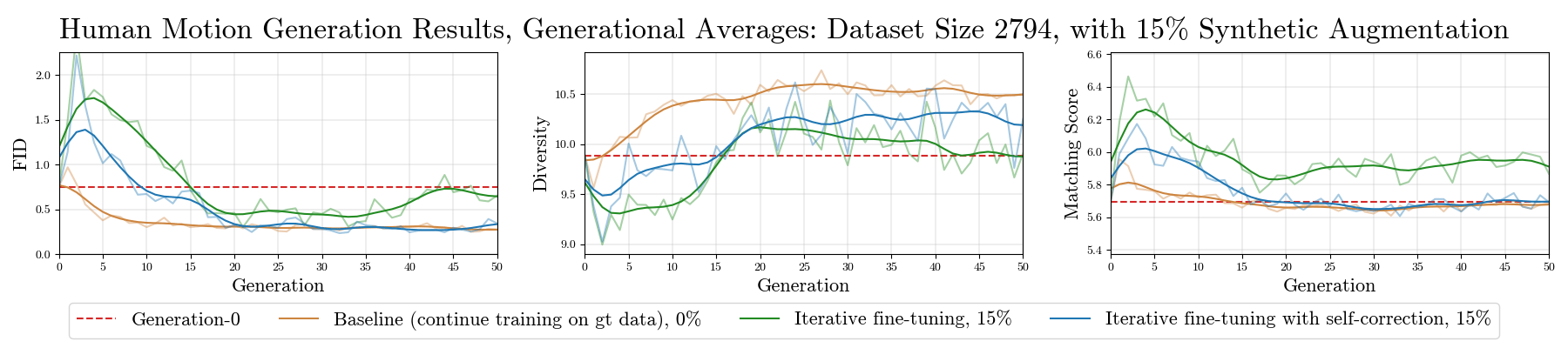}}
\centerline{\includegraphics[width=0.9\columnwidth]{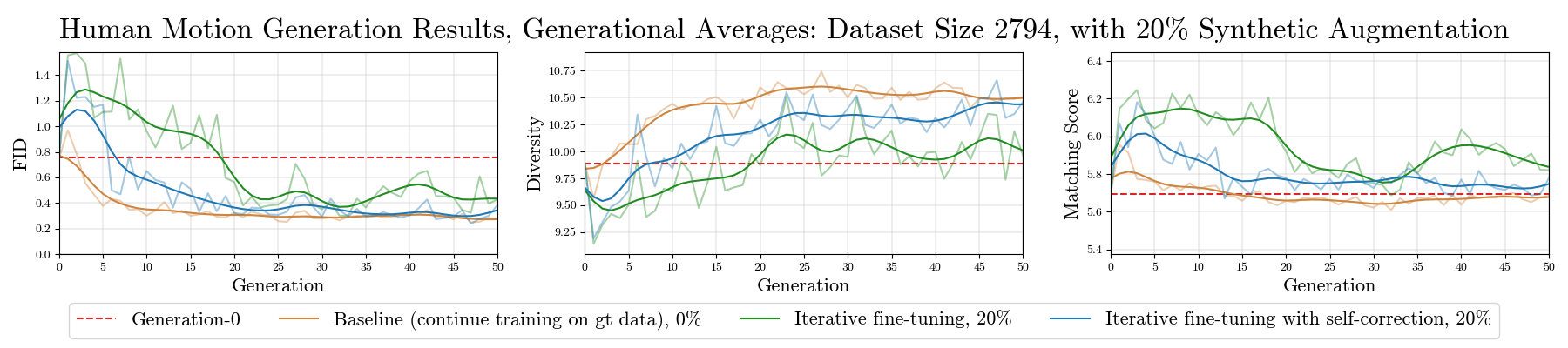}}
\centerline{\includegraphics[width=0.9\columnwidth]{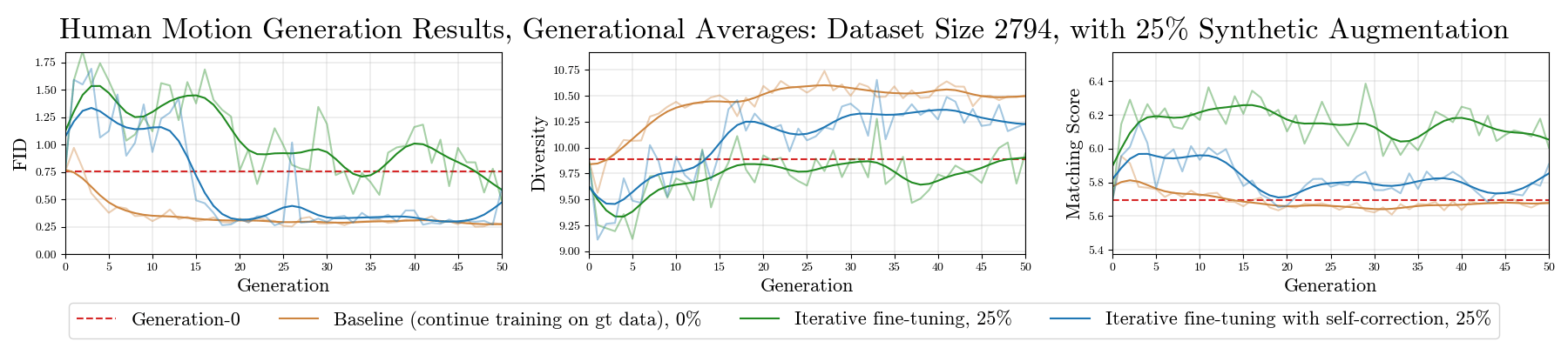}}
\caption{Results from our human motion experiments on iterative fine-tuning with dataset size $n=2794$.
These are graphs of the average evaluation metrics for every generation.
Graphing the average evaluation metrics makes the training dynamics trend over time more clear.
With this additional smoothing, it is more clear that iterative fine-tuning \textbf{\color{blue_better}with self-correction} outperforms iterative fine-tuning \textbf{\color{green_better}with no self-correction}, and is competitive with the \textbf{\color{orange}baseline} after many generations; in fact, it appears to converge to the baseline (on average) for every synthetic augmentation percentage that we considered.}
\label{fig:large_scale_average}
\end{center}
\vskip -0.2in
\end{figure*}

\newpage

\section{Consistency Across Seeds: Additional Human Motion Generation Quantitative Results}\label{appendix:multiple_seeds}

In Figures~\ref{fig:0064_025_seeds},~\ref{fig:0064_050_seeds},~\ref{fig:0064_075_seeds}, and~\ref{fig:0064_100_seeds}, we present experimental results from runs across three more seeds for our human motion experiments when the dataset size is $n=64$.
We find that the self-correction technique consistently yields improved training dynamics over iterative
fine-tuning without correction.

\begin{figure*}
\vskip 0.2in
\begin{center}
\centerline{\includegraphics[width=0.7\columnwidth]{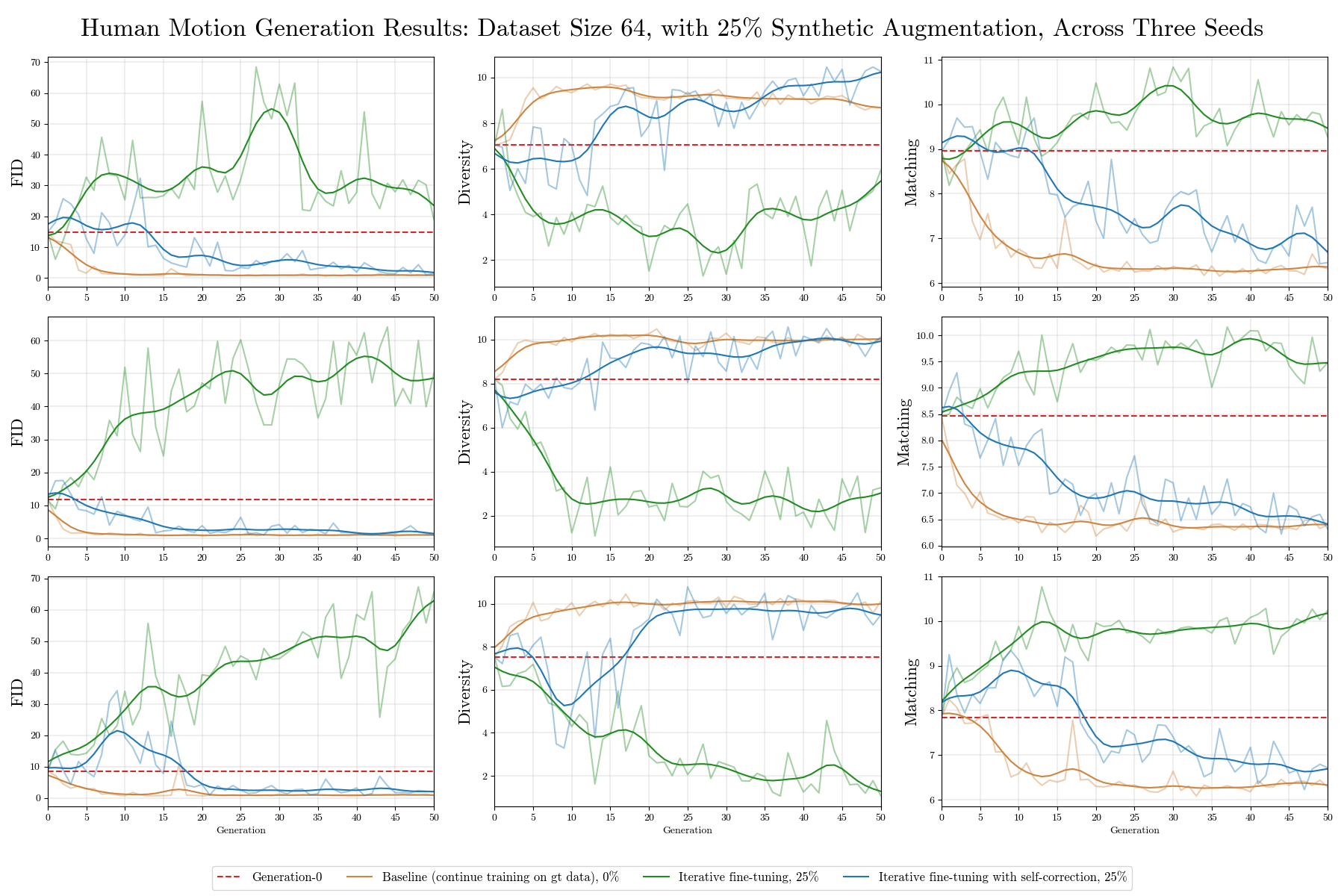}}
\caption{Results from our human motion experiments on iterative fine-tuning, with dataset size $n=64$ and $25\%$ augmentation percentage.
Each row corresponds to a different random seed.
We can see that iterative fine-tuning \textbf{\color{blue_better}with self-correction} consistently outperforms iterative fine-tuning \textbf{\color{green_better}with no self-correction}, and the FID score appears to converge to the \textbf{\color{orange}baseline} after many generations.}
\label{fig:0064_025_seeds}
\end{center}
\vskip -0.2in
\end{figure*}

\begin{figure*}
\vskip 0.2in
\begin{center}
\centerline{\includegraphics[width=0.7\columnwidth]{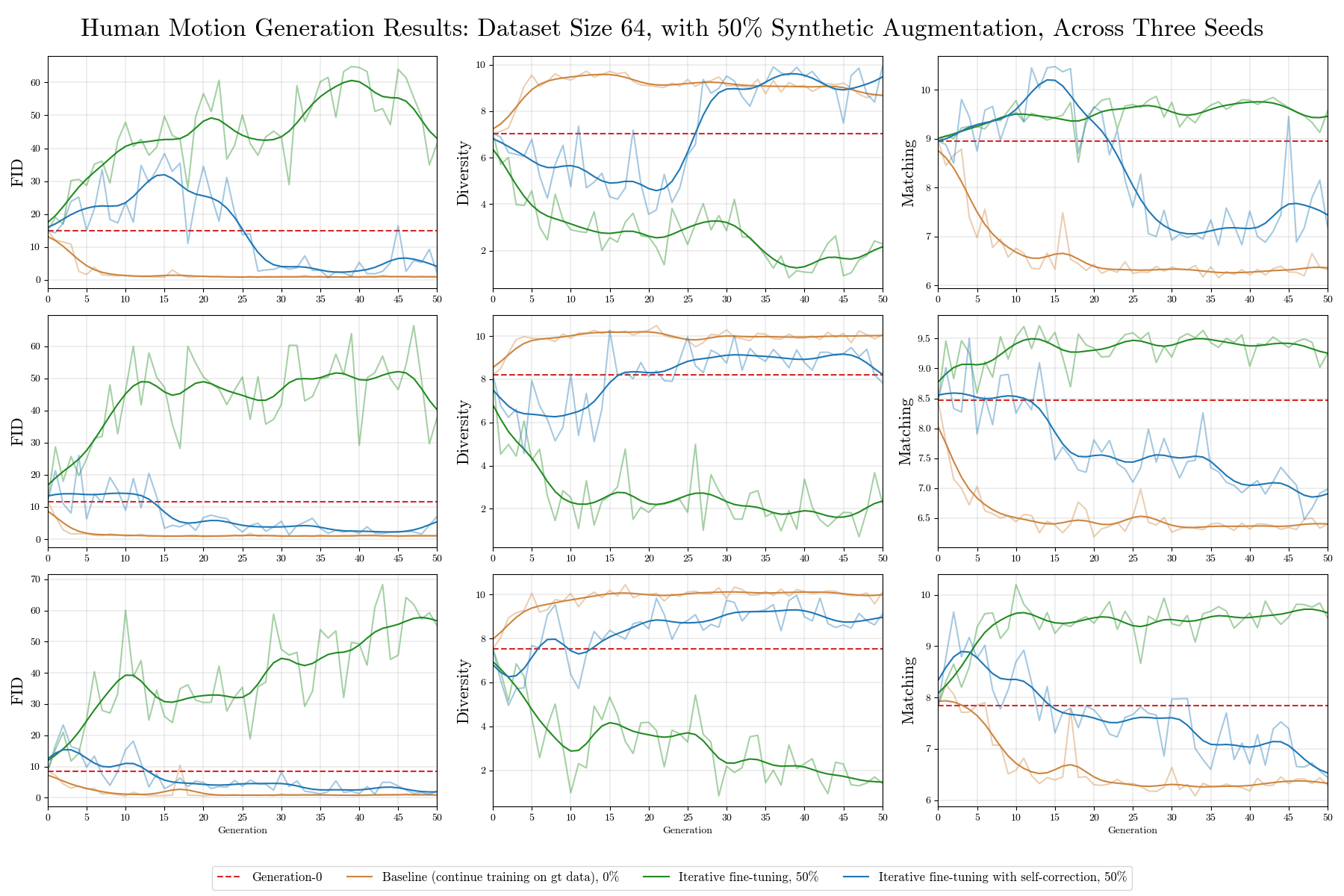}}
\caption{Results from our human motion experiments on iterative fine-tuning, with dataset size $n=64$ and $50\%$ augmentation percentage.
Each row corresponds to a different random seed.
We can see that iterative fine-tuning \textbf{\color{blue_better}with self-correction} consistently outperforms iterative fine-tuning \textbf{\color{green_better}with no self-correction}, and the FID score appears to converge to the \textbf{\color{orange}baseline} after many generations.}
\label{fig:0064_050_seeds}
\end{center}
\vskip -0.2in
\end{figure*}

\begin{figure*}
\vskip 0.2in
\begin{center}
\centerline{\includegraphics[width=0.7\columnwidth]{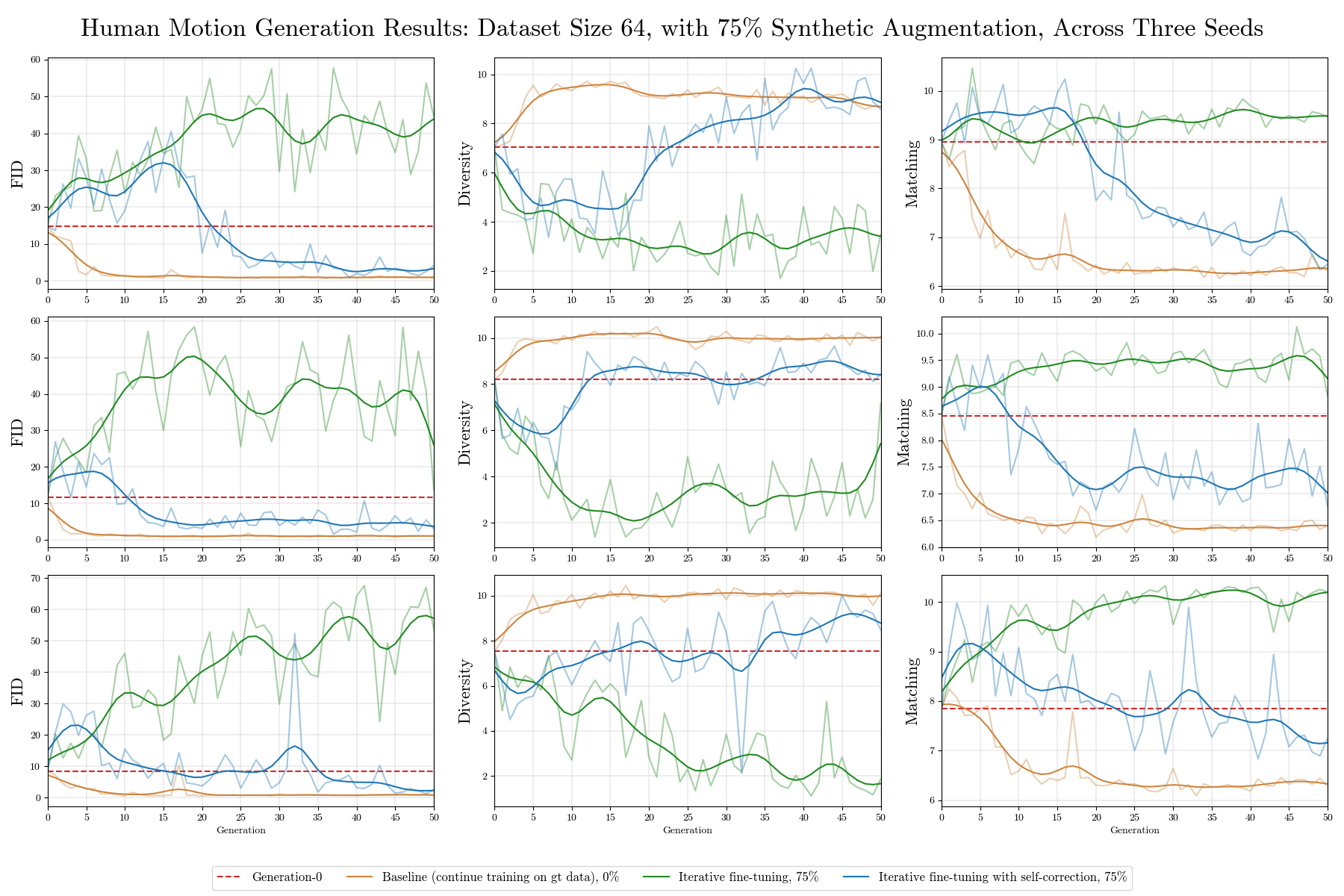}}
\caption{Results from our human motion experiments on iterative fine-tuning, with dataset size $n=64$ and $75\%$ augmentation percentage.
Each row corresponds to a different random seed.
We can see that iterative fine-tuning \textbf{\color{blue_better}with self-correction} consistently outperforms iterative fine-tuning \textbf{\color{green_better}with no self-correction}, and the FID score appears to converge near the \textbf{\color{orange}baseline} after many generations.}
\label{fig:0064_075_seeds}
\end{center}
\vskip -0.2in
\end{figure*}

\begin{figure*}
\vskip 0.2in
\begin{center}
\centerline{\includegraphics[width=0.7\columnwidth]{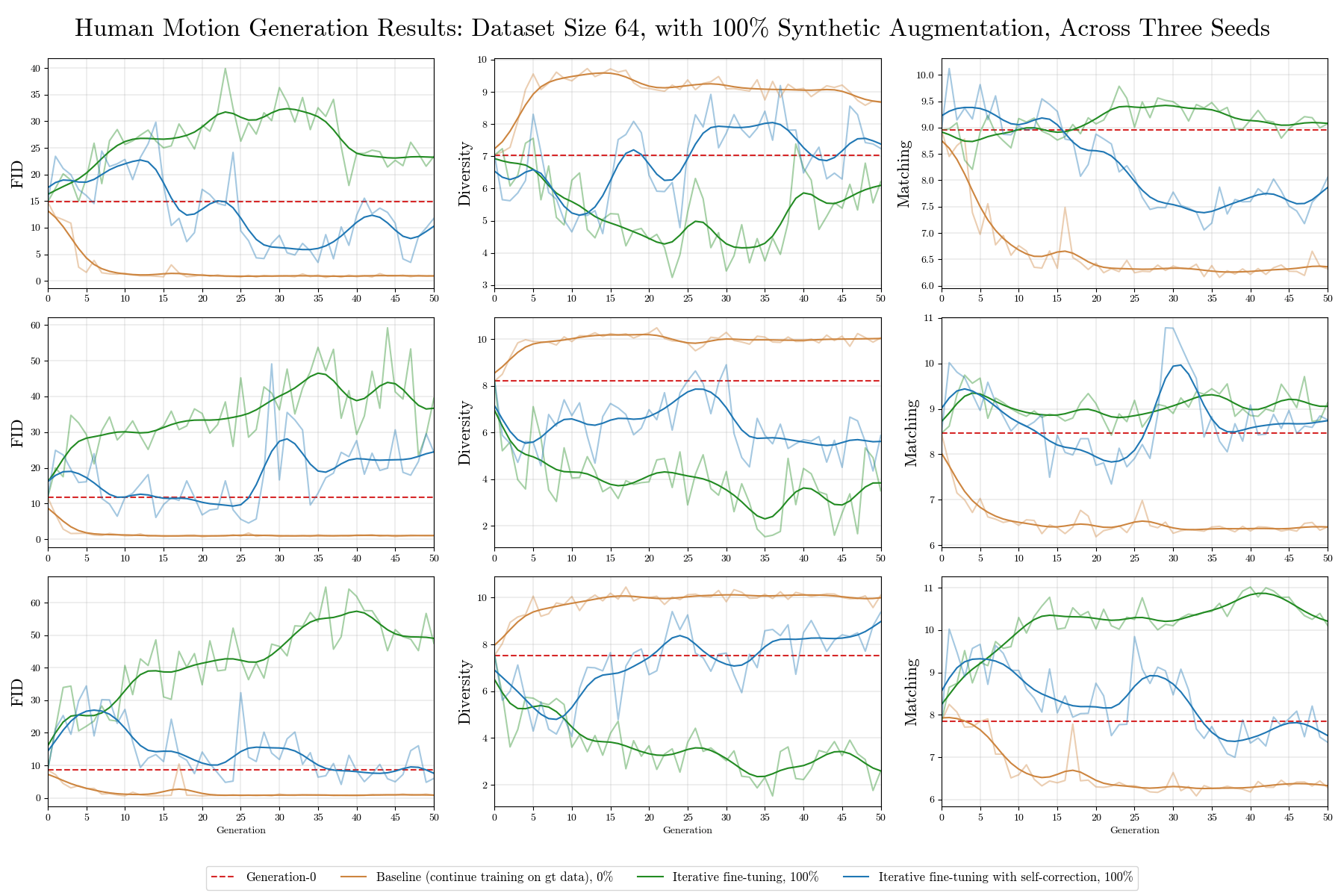}}
\caption{Results from our human motion experiments on iterative fine-tuning, with dataset size $n=64$ and $100\%$ augmentation percentage.
Each row corresponds to a different random seed.
We can see that iterative fine-tuning \textbf{\color{blue_better}with self-correction} consistently outperforms iterative fine-tuning \textbf{\color{green_better}with no self-correction}.
However, we see less stability than in the runs with a lower augmentation percentage. This is in accordance with Theorem~\ref{thm:2_shortened}.}
\label{fig:0064_100_seeds}
\end{center}
\vskip -0.2in
\end{figure*}

\end{document}